%% file: main.tex
\documentclass{article}

\usepackage{arxiv}

\usepackage[T1]{fontenc}    
\usepackage{hyperref}       
\usepackage{url}            
\usepackage{booktabs}       
\usepackage{amsfonts}       
\usepackage{nicefrac}       
\usepackage{microtype}      
\usepackage{lipsum}		
\usepackage{doi}

\usepackage[ansinew]{inputenc} 
\usepackage[intlimits]{amsmath} 
\usepackage{amssymb} 
\usepackage{amsfonts} 
\usepackage{amsthm} 
\usepackage{fixltx2e} 
\usepackage[final]{graphicx}    
\usepackage{parskip}            
\usepackage[toc,page]{appendix}

\usepackage{algorithm}
\usepackage{algpseudocode}

\usepackage{mathtools}
\usepackage{geometry}
\usepackage{pgfplots}
\usepackage{graphicx}
\usepackage{caption}
\usepackage{subcaption}
\pgfplotsset{compat = newest}

\usepackage{booktabs}       
\usepackage{tikz}
\usepackage{pgfplots}
\usetikzlibrary{shapes.geometric, arrows, trees}

\tikzstyle{nn_node} = [circle, 
minimum size=1cm, 
text centered, 
draw=black, 
fill=white!30]
\tikzstyle{arrow} = [thick,->,>=stealth, draw=gray]
\tikzstyle{annotation} = [rectangle, rounded corners, 
minimum width=1cm, 
minimum height=2cm,
text centered, 
draw=black, 
fill=lightgray]

\newcommand{\tablecentered}[1]{\begin{tabular}{l} #1 \end{tabular}}

\newsavebox{\Citname}

\newtheorem{definition}{Definition}[section]
\newtheorem{theorem}[definition]{Theorem}

\DeclareMathOperator*{\argmax}{arg\,max}
\DeclareMathOperator*{\argmin}{arg\,min}

\title{The Definitive Guide to Policy Gradients in Deep Reinforcement Learning: \\ Theory, Algorithms and Implementations}

\date{} 					

\author{ {Matthias Lehmann}
		\\
	University of Cologne\\
}

\renewcommand{\headeright}{}
\renewcommand{\undertitle}{}
\renewcommand{\shorttitle}{The Definitive Guide to Policy Gradients in Deep Reinforcement Learning}

\hypersetup{
pdftitle={The Definitive Guide to Policy Gradients in Deep Reinforcement Learning: Theory, Algorithms and Implementations},
pdfsubject={ML},
pdfauthor={Matthias Lehmann},
pdfkeywords={Reinforcement Learning, Policy Gradients, Deep Reinforcement Learning, PPO, TRPO, V-MPO},
}

\begin{document}
\maketitle

\begin{abstract}
	In recent years, various powerful policy gradient algorithms have been proposed in deep reinforcement learning. While all these algorithms build on the Policy Gradient Theorem, the specific design choices differ significantly across algorithms. We provide a holistic overview of on-policy policy gradient algorithms to facilitate the understanding of both their theoretical foundations and their practical implementations. In this overview, we include a detailed proof of the continuous version of the Policy Gradient Theorem, convergence results and a comprehensive discussion of practical algorithms. We compare the most prominent algorithms on continuous control environments and provide insights on the benefits of regularization. All code is available at \url{https://github.com/Matt00n/PolicyGradientsJax}.
\end{abstract}




\tableofcontents


\input{introduction.tex}

\input{preliminaries.tex}

\input{pg_theorem.tex}

\input{algorithms.tex}

\input{convergence.tex}

\input{experiments.tex}

\input{conclusion.tex}

\newpage


\bibliographystyle{plain}
\bibliography{refs}  

\newpage

\begin{appendices}

\input{hyperparameters.tex}

\input{extendend_experiments.tex}
\input{vmpo_derivation.tex}

\input{math_results.tex}

\end{appendices}

\end{document}

%% file: introduction.tex
\section{Introduction}\label{sec:Introduction}
\pagenumbering{arabic}






Reinforcement Learning (RL) is a powerful set of methods for an agent to learn how to act optimally in a given environment to maximize some reward signal. In contrast to other methods such as dynamic programming, RL achieves this task of learning an optimal policy, which dictates the optimal behavior, via a trial-and-error process of interacting with the environment \cite{Sutton1998}. 
Most early successful applications of RL use value-based methods (e.g., \cite{watkins1989learning, tesauro1995temporal, Mnih2015}), which estimate the expected future rewards to inform the agent's decisions. However, these methods only indirectly optimize the true objective of learning an optimal policy \cite{Hasselt2021RLlecture} and are non-trivial to apply in settings with continuous action spaces \cite{Sutton1998}.


In this work, we discuss policy gradient algorithms \cite{Sutton1998} as an alternative approach, which aims to directly learn an optimal policy. Policy gradient algorithms are by no means new \cite{barto1983neuronlike, sutton1984temporal, williams1987reinforcement, williams1992simple}, but this subfield only gained traction in recent years following the emergence of deep RL \cite{Mnih2015} with the development of various powerful algorithms (e.g., \cite{mnih2016asynchronous, schulman2017proximal, song2019v}). Deep RL is a subfield of RL, which uses neural networks and other deep learning methods. The increased interest in policy gradient algorithms is due to several appealing properties of this class of algorithms. They can be used natively in continuous action spaces without compromising the applicability to discrete spaces \cite{sutton2000comparing}. In contrast to value-based methods, policy gradient algorithms inherently learn stochastic policies, which results in smoother search spaces and partly remedies the exploration problem of having to acquire knowledge about the environment in order to optimize the policy \cite{Sutton1998, sutton2000comparing}. In some settings, the optimal policy may also be stochastic itself \cite{Sutton1998}. Lastly, policy gradient methods enable smoother changes in the policy during the learning process, which may result in better convergence properties \cite{sutton2000comparing}.


Our goal is to present a holistic overview of policy gradient algorithms. In doing so, we limit the scope to on-policy algorithms, which we will define in Section \ref{sec:Preliminaries}. Thus, we exclude some popular algorithms including DDPG \cite{lillicrap2015continuous}, TD3 \cite{fujimoto2018addressing} and SAC \cite{haarnoja2018soft}. See \autoref{fig:overview_rl_algos} for an overview of RL and the subfields we cover. Our contributions are as follows:
\begin{itemize}
	\item We give a comprehensive introduction to the theoretical foundations of policy gradient algorithms including a detailed proof of the continuous version of the Policy Gradient Theorem.
	\item We derive and compare the most prominent policy gradient algorithms and provide high quality pseudocode to facilitate understanding.
	\item We release competitive implementations of these algorithms, including the, to the best of our knowledge, first publicly available V-MPO implementation  displaying performance on par with the results in the original paper.
\end{itemize}

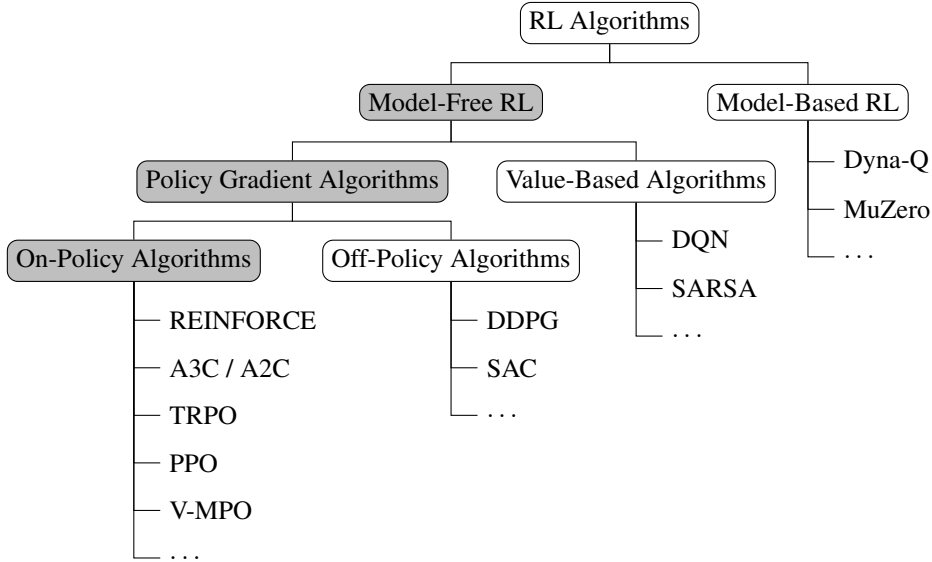
\begin{figure}
\centering
\begin{tikzpicture}[
  sibling distance=10em,
  level distance=30pt,
  root/.style = {shape=rectangle, rounded corners,
    draw, align=center,
    fill=white},
  main/.style = {shape=rectangle, rounded corners,
    draw, align=center,
    fill=lightgray,},
  grandchild/.style={grow=down,xshift=1em,anchor=west,
	edge from parent path={(\tikzparentnode.south) |- (\tikzchildnode.west)}},
  first/.style={level distance=5ex},
  second/.style={level distance=9ex},
  third/.style={level distance=13ex},
  fourth/.style={level distance=17ex},
  fifth/.style={level distance=21ex},
  sixth/.style={level distance=25ex},]]
  \node [root, align=center] {RL Algorithms}
  	[edge from parent fork down]
    child { node [main, xshift=-1em] {Model-Free RL}
      child { node [main, xshift=-1em] {Policy Gradient  Algorithms}
        child { node [main, xshift=-1em] {On-Policy Algorithms} 
        	child[grandchild,first] { node {REINFORCE} }
        	child[grandchild,second] { node {A3C / A2C} }
        	child[grandchild,third] { node {TRPO} }
        	child[grandchild,fourth] { node {PPO} }
        	child[grandchild,fifth] { node {V-MPO} }
        	child[grandchild,sixth] { node {\(\hdots\)} } }
        child { node [root, xshift=1em] {Off-Policy Algorithms} 
        	child[grandchild,first] { node {DDPG} }
        	child[grandchild,second] { node {SAC} }
        	child[grandchild,third] { node {\(\hdots\)} } } }
       child { node [root, xshift=2em] {Value-Based Algorithms} 
       		child[grandchild,first] { node {DQN} }
        	child[grandchild,second] { node {SARSA} }
        	child[grandchild,third] { node {\(\hdots\)} } } }
    child { node [root, xshift=2.5em] {Model-Based RL} 
    	child[grandchild,first] { node {Dyna-Q} }
        child[grandchild,second] { node {MuZero} }
        child[grandchild,third] { node {\(\hdots\)} }};
\end{tikzpicture}
\caption{Simplified taxonomy of RL algorithms. Subfields of RL we focus on are highlighted in gray.}
\label{fig:overview_rl_algos}
\end{figure}

The remainder of this paper is organized as follows. Section \ref{sec:Preliminaries} introduces fundamental definitions in RL as well as an overview of deep learning. Section \ref{sec:pg_theory} derives the theoretical foundations of policy gradient algorithms with a special focus on proving the Policy Gradient Theorem, based on which we will construct several existing practical algorithms in Section \ref{sec:algorithms}. In Section \ref{sec:convergence}, we discuss convergence results from literature. Section \ref{sec:experiments}, presents the results of our numerical experiments comparing the discussed algorithms. Section \ref{sec:conclusion} concludes.

%% file: preliminaries.tex
\section{Preliminaries}\label{sec:Preliminaries}

In this section, we present prerequisites for subsequent chapters. Specifically, we introduce our notation in Section \ref{sec:notation} and present overviews of RL in Section \ref{sec:rl_intro} and of deep learning in Section \ref{sec:deep_learning}. Furthermore, we list several well-known definitions and results from probability theory, measure theory and analysis, which we use in our paper, in Appendix \autoref{sec:aux}. 

\input{notation.tex}
\input{prelim_rl.tex}

\input{deep_learning.tex}

%% file: notation.tex
\subsection{Notation}\label{sec:notation}

We denote the set of natural numbers by \(\mathbb{N}\), natural numbers including zero by \(\mathbb{N}_0\), real numbers by \(\mathbb{R}\) and positive real numbers by \(\mathbb{R}_+\). We denote \(d\)-dimensional real-numbered vector spaces as \(\mathbb{R}^d\). By \(\mathcal{P}(\mathcal{A})\), we denote the power set of a set \(\mathcal{A}\). Where possible, we denote random variables with capital letters and their realizations with the corresponding lower case letters. For any probability measure \(\mathbb{P}\), we denote the probability of an event \(X=x\) as \(\mathbb{P}(X=x)\). Similarly, we write \(\mathbb{P}(X=x \mid Y=y)\) for conditional probabilities. When it is clear, which random variable is referred to, we regularly omit it to shorten notation, i.e. \(\mathbb{P}(X=x) = \mathbb{P}(x)\). We identify measurable spaces \((\mathcal{A}, \Sigma)\) just by the set \(\mathcal{A}\) as we always use the respective power set \(\mathcal{P}(\mathcal{A})\) for discrete sets and the Borel algebra for intervals in \(\mathbb{R}^d\) as the respective \(\sigma\)-algebra \(\Sigma\). We express most Lebesgue integrals w.r.t. the Lebesgue measure \(\lambda\) using Theorem \ref{th:measure_change}. To simplify notation, we write integrals for measurable functions \(f\) on \(\mathcal{A}\) as \(\int_{a \in \mathcal{A}} f(a) \: da \coloneqq \int_{a \in \mathcal{A}} f(a) \: d\lambda(a)\). We denote that a random variable \(X\) follows a probability distribution \(p\) by \(X \sim p\). For any random variable \(X \sim p\), we denote by \(\mathbb{E}_{X \sim p}[X]\) and \(\mathrm{Var}_{X \sim p}[X]\) its expectation and variance. We denote the set of probability distributions over some measurable space \(\mathcal{A}\) as  \(\Delta(\mathcal{A})\). We write \(\lvert \mathcal{A} \rvert\) for the cardinality of a finite set \(\mathcal{A}\) or area of a region \(\int_{a \in \mathcal{A}} da\). For any variable or function \(x\), we commonly denote approximations to it by \(\hat{x}\).

%% file: prelim_rl.tex
\subsection{Reinforcement Learning}\label{sec:rl_intro}

In the following, we formally describe the general problem setting encountered in RL, define fundamental functions and introduce the subfields of RL our work is further concerned with. Sections \ref{sec:rl_problem} and \ref{sec:rl_value} are based on \cite{Sutton1998}, Chapter 3. 

\subsubsection{Problem Setting}\label{sec:rl_problem}

Each problem instance in RL consists of an agent and an environment with which he interacts to achieve some specific goal. The environment comprises everything external to the agent and can be formalized as a Markov Decision Process (MDP). Let an action space \(\mathcal{A}\) be the set of all actions the agent can take and let a state space \(\mathcal{S}\) be the set of all possible states, i.e. snapshots of the environment at any given point in time. State and action spaces can be discrete or continuous\footnote{Here, we call a state/action space continuous if it is an interval in \(\mathbb{R}^d\) for \(d \in \mathbb{N}\).} and we assume both to be compact and measurable. We write an MDP as a tuple \(\mathcal{M} = (\mathcal{S}, \mathcal{A}, P, \gamma, p_0)\), where \(P \colon \mathcal{S} \times \mathcal{A} \rightarrow \Delta(\mathcal{S} \times \mathbb{R})\) is the environment's transition function, which defines the probability\footnote{Technically, this is the value of the probability density function for continuous distributions. However, we unify terminology by referring to the values of probability density functions as probabilities here and in the following.} \(P(s^\prime,r \mid s,a)\) of transitioning to a new environment state \(s^\prime\) and receiving reward \(r \in \mathbb{R}\) when the agent uses action \(a\) in state \(s\), \(\gamma \in [0, 1]\) is a discount rate and \(p_0 \in \Delta(\mathcal{S})\) is a probability distribution over potential starting states. We assume rewards \(r\) to be bounded. In the following, our notation assumes state and action spaces to be continuous.

We call sequences of states, actions and rewards \((s_t, a_t, r_{t+1}, s_{t+1}, a_{t+1}, r_{t+2}, \ldots, \\ s_{t+k-1}, a_{t+k-1}, r_{t+k}, s_{t+k})\) trajectories. A one-step trajectory, i.e. a tuple \((s_t, a_t, r_{t+1}, s_{t+1})\) is called a transition. In this work, we limit ourselves to episodic settings, where the agent only interacts with the environment for a finite number of at most \(T\) steps after which the environment is reset to a starting state. An episode may however be shorter than \(T\) if a terminal state is reached. Therefore each episode consists of a trajectory \((s_0, a_0, r_1, s_1, a_1, r_2, \ldots, s_{\tilde{T}-1}, a_{\tilde{T}-1}, r_{\tilde{T}}, s_{\tilde{T}})\), with \(\tilde{T} \leq T\). Rewards are occasionally omitted from the trajectory notation since they do not influence future states. Correspondingly, we also can compute the alternative transition probabilities \(P (s' \mid s,a) = \int_{r \in \mathbb{R}} P(s', r \mid s,a) \: dr\).

The main goal in reinforcement learning is to solve the control problem of learning a policy \(\pi \colon \mathcal{S} \rightarrow \Delta(\mathcal{A})\) to maximize the expected return. The return \(G_t \coloneqq \sum_{k=0}^{T} \gamma^k r_{t+k+1}\) is the discounted sum of rewards from timestep \(t\) onwards. Note that \(G_t\) is bounded since rewards are bounded. We denote the probability of taking action \(a\) in state \(s\) under policy \(\pi\) with \(\pi(a\mid s)\). For a policy \(\pi\), its stationary state distribution \(d^\pi\) determines the probability 
of being in a specific state \(s \in \mathcal{S}\) at any point in time when following \(\pi\). 

Let \(\Pi\) be the set of all possible policies. RL algorithms \(\mathfrak{A} \colon \Pi \rightarrow \Pi\) for the control problem now iteratively learn policies by interacting with the environment using the current policy to sample transitions, which are then used to update the policy. We will discuss how these updates can look like in Section \ref{sec:prelim_pg_intro}. A key characteristic of many RL problems is a necessary trade-off between exploration and exploitation in this learning process. The agent has no prior knowledge of the environment and thus needs to explore different transitions in order to learn which states and actions are desirable. As state and action spaces are typically large however, exploiting the already acquired knowledge about the environment is also crucial to guide the search process for an optimal policy to subspaces that hold most promise. A common approach to this exploration problem is to add noise to the policy.

\subsubsection{Value Functions}\label{sec:rl_value}


Based on the return, we define the value and action-value functions, which are fundamental in RL. The value function 
\begin{equation}
	V_{\pi}(s) \coloneqq \mathbb{E}_{\pi}\bigl[G_t \mid S_t = s\bigr] \label{eq:v_def}
\end{equation}
gives the expected return from state \(s\) onwards when following policy \(\pi\), which selects all subsequent actions. Thus, the value function states how good it is to be in a specific state \(s\) given a policy \(\pi\). Note that here we follow the general convention to write this just as an expectation over \(\pi\). However, it should be noted that this expectation integrates over all subsequent states and actions that are obtained by following policy \(\pi\), i.e. Equation \eqref{eq:v_def} computes the expected return given that all subsequent actions are sampled from \(\pi\) and all rewards and next states are sampled from \(P\). This is implicit in our notation here as well as in further expectations. 


Next, we define the action-value function 
\[Q_{\pi}(s,a) \coloneqq \mathbb{E}_{\pi}\bigl[G_t \mid S_t=s, A_t=a\bigr],\] 
which differs from the value function in that the very first action \(a\) is provided as an input to the function and not determined by the policy. We observe the following relation between \(V_\pi\) and \(Q_\pi\):
\begin{equation*}
	V_\pi (s) = \int_{a \in \mathcal{A}} \pi(a \mid s) \: Q_\pi (s, a) \: da. \label{eq:v_q_rel}
\end{equation*}
Further, we call 
\[ A_\pi(s,a) \coloneqq Q_\pi(s,a) - V_\pi(s) \]
the advantage function, which determines how good an action \(a\) is in state \(s\) in relation to other possible actions.

From the definitions of \(V_\pi\) and \(Q_\pi\) we can derive the so-called Bellman equations \cite{bellman1966dynamic}. Starting from Equation \eqref{eq:v_def}, we use the definition of the return \(G_t\), explicitly write out the expectation for the first transition and then apply the definition of \(V_\pi\) again:
\begin{align*}
	V_{\pi}(s) &= \mathbb{E}_{\pi}\bigl[G_t \mid S_t = s\bigr] \\
	&= \mathbb{E}_{\pi}\bigl[R_{t+1} + \gamma G_{t+1} \mid S_t = s\bigr] \\
	&= \int_{a \in \mathcal{A}} \pi(a \mid s) \int_{s^\prime \in \mathcal{S}} \int_{r \in \mathbb{R}} P(s^\prime, r \mid s,a)\biggl(r + \gamma \mathbb{E}_{\pi}\Bigl[G_{t+1} \mid S_{t+1} = s^\prime\Bigr]\biggr) \: dr \: ds' \: da \\
	&= \int_{a \in \mathcal{A}} \pi(a \mid s) \int_{s^\prime \in \mathcal{S}} \int_{r \in \mathbb{R}} P(s^\prime, r \mid s,a)\bigl(r + \gamma V_{\pi}(s^\prime)\bigr) \: dr \: ds' \: da
\end{align*}
Thus, we find a formulation of the value function, which depends on the value of subsequent states. Collapsing the expectation again yields the form known as the Bellman equation of the value function: 
\[V_{\pi}(s) = \mathbb{E}_{\pi}\bigl[R_{t+1} + \gamma V_{\pi}(S_{t+1})\bigr].\]
Similarly, we can find the Bellman equation for the action-value function:
\begin{equation}
	Q_{\pi}(s,a) = \mathbb{E}_{\pi}\bigl[R_{t+1} + \gamma Q_{\pi}(S_{t+1}, A_{t+1})\bigr]. \label{eq:bellman_q}
\end{equation}


Now, we can formally define what optimality means in RL. An optimal policy \(\pi^*\) is defined by \(V_{\pi^*}(s) \geq V_\pi(s)\) for all states \(s\) and policies \(\pi\), i.e. any optimal policy maximizes the expected return. It can be shown that in every finite MDP, a deterministic optimal policy exists \cite{MohriRostamizadehTalwalkar18}. All optimal policies share the same optimal value function \(V^*(s) \coloneqq \max_{\pi \in \Pi} V_\pi(s)\) and optimal action-value function \(Q^*(s,a) \coloneqq \max_{\pi \in \Pi} Q_\pi(s,a)\) and select actions \(a \in \argmax_{a^\prime} Q^*(s,a^\prime)\) for every state. Applying this to Equation \eqref{eq:bellman_q} yields the Bellman optimality equation
\begin{align*}
	Q^*(s,a) &= \mathbb{E}_{\pi^*}\bigl[R_{t+1} + \gamma Q^*(S_{t+1}, A_{t+1})\bigr] \\
	&= \mathbb{E}\bigl[R_{t+1} + \gamma \max_{a^\prime \in \mathcal{A}} Q^*(S_{t+1}, a^\prime)\bigr]
\end{align*}
We cite the following result without proof from \cite{Sutton1998} on how to obtain an optimal policy, which we will revisit in Section \ref{sec:convergence}.

\begin{theorem}\label{th:gpi}
	(Generalized Policy Iteration)
	Let \(\pi_\text{old}\) be the current policy. Then, Generalized Policy Iteration updates its policy by
	\begin{equation*}
		\pi_\text{new} \in \argmax_{\pi \in \Pi} \mathbb{E}_{A \sim \pi} \bigl[Q_{\pi_\text{old}}(s,A)\bigr]
	\end{equation*}
	for all \(s \in \mathcal{S}\).
	Let \(\bigl(\pi_n \bigr)^\infty_{n=0}\) be a sequence of policies obtained through Generalized Policy Iteration. Then, this sequence converges to an optimal policy, i.e. 
	\begin{equation*}
		\lim_{n \to \infty} \pi_n = \pi^*
	\end{equation*}
	and 
	\begin{equation*}
		\lim_{n \to \infty} Q_{\pi_n} = Q^*.
	\end{equation*}
\end{theorem}

\subsubsection{On-Policy Policy Gradient Methods}\label{sec:prelim_pg_intro}




Finally, we will delineate the subfields of RL on which our work focuses. In this context, we will successively introduce function approximation, policy gradient methods and the on-policy paradigm.


RL algorithms are mostly concerned with learning functions such as \(\pi\), \(V_\pi\) or \(Q_\pi\). Early reinforcement methods learn exact representations of these by maintaining lookup tables with entries for each possible function input \cite{Sutton1998}. While this approach yields theoretical convergence guarantees \cite{MohriRostamizadehTalwalkar18}, it is practically very limited. Similar states are treated independently such that learnings do not generalize from one state to others while specific states are only rarely visited in large state spaces \cite{Sutton1998}. Moreover, this approach is not applicable to continuous spaces. Function approximation remedies these shortcomings by parameterizing the function to be learned. Let \(f_\theta(x)\) be this learnable function, where \(\theta\) are the function's parameters, which are adjusted over the course of learning, and \(x\) are the functions inputs such as states and actions or representations thereof. By choosing \(f_\theta\) to be continuous in its inputs, we can ensure that \(f_\theta\) generalizes across its inputs when we fit it to sampled transitions \cite{Sutton1998}. \(f_\theta\) can be as simple as a linear mapping, i.e. \(f_\theta(x) = \theta^Tx\), however recent works mostly use neural networks as function approximators (e.g., \cite{Mnih2015, silver2016}). The field using neural networks as function approximators is coined deep RL \cite{Mnih2015}. For the remainder of this paper, you can consider any learned function to be a neural network unless explicitly stated otherwise, although all our statements apply to any differentiable function approximators. We will introduce deep learning and neural networks in detail in Section \ref{sec:deep_learning}. 



Policy gradient methods pose an alternative to value-based methods in RL. Most early successes in RL use value-based methods such as Q-Learning \cite{watkins1989learning} or SARSA \cite{rummery1994sarsa}, that aim at learning a sequence of value functions converging to the optimal value function, from which an optimal policy can then be inferred. In contrast, policy-based RL, which we focus on in this work, directly learns a parameterized policy \(\pi_\theta\). The main idea in this learning process is to increase the probability of those actions that lead to higher returns until we reach an (approximately) optimal policy \cite{Sutton1998}. While this optimization problem can be approached in several ways, gradient-based methods are most commonly used \cite{Hasselt2021RLlecture}. Following \cite{sutton2000comparing}, we define policy gradient methods as follows.

\begin{definition}\label{def:pg}
	(Policy Gradient Algorithm) Let \(\pi_\theta \colon \mathcal{S} \rightarrow \Delta(\mathcal{A})\) be a fully differentiable function with learnable parameters \(\theta \in \mathbb{R}^d\) mapping states to a probability distribution over actions. Let \(J \colon \mathbb{R}^d \rightarrow \mathbb{R}\) be some performance measure of the parameters. We call any learning algorithm a policy gradient algorithm if it learns its policy \(\pi_\theta\) by updating \(\theta\) via gradient ascent (or descent) on \(J\), i.e. its updates have the general form
	\begin{equation}
		\theta_\text{new} \gets \theta + \alpha \nabla_\theta J(\theta), \label{eq:pg_update}
	\end{equation}
	where \(\alpha \in \mathbb{R}\) is a step size parameter of the algorithm.
\end{definition}


In policy-based RL, two distinct ways exist to have the policy output a probability distribution over actions, from which actions can be sampled \cite{Sutton1998}. For discrete action spaces, we construct a discrete distribution over the action space by normalizing the policies' raw outputs via a softmax function \cite{goodfellow2016deep}. In this case, we have 
\[ \pi(a \mid s) = \frac{\exp(\pi_\theta(a \mid s))}{\sum_{a' \in \mathcal{A}}\exp(\pi_\theta(a' \mid s))}.\]
For continuous action spaces, we let \(\pi_\theta\) output the mean \(\mu\) and standard deviation \(\sigma\) of a Gaussian distribution, i.e. \(\pi_\theta(s) = \bigl(\mu_\theta(s), \sigma_\theta(s) \bigr)\) such that 
\[ \pi(a \mid s) = \frac{1}{\sigma_\theta(s)\sqrt{2\pi}} \exp \Biggl( -\frac{\bigl(a - \mu_\theta(s) \bigr)^2}{2\sigma_\theta(s)^2} \Biggr). \] 
This parameterization of a Gaussian distribution for the policy was first introduced by \cite{williams1987reinforcement, williams1992simple}. As action spaces are commonly bounded, the actions sampled from such a Gaussian are typically transformed to be within these bounds either by clipping or by applying a squashing distribution \cite{andrychowicz2020matters}. Further, we highlight that the policies learned by policy gradient methods in both the discrete and the continuous case are generally stochastic. This stands in contrast to value-based methods which generally learn deterministic policies \cite{sutton2000comparing}. Policy gradient methods are the core focus of this paper and will be discussed in-depth in subsequent sections.


Lastly, we delineate on-policy from off-policy algorithms. In RL, we distinguish between behavior and target policies \cite{Sutton1998}. A behavior policy is a policy which generates the data in form of trajectories from which we want to learn, i.e. this is the policy from which we sample actions when interacting with the environment. Conversely, the target policy is the policy which we want to learn about to evaluate how good it is in the given environment and improve it. Algorithms where behavior and target policy are not identical, e.g. Q-Learning \cite{watkins1989learning} or DQN \cite{Mnih2015}, are referred to as off-policy algorithms. In this work, we only discuss on on-policy algorithms, where behavior and target policy are identical. Hence, when speaking of policy gradient algorithms in the following, we always implicitly mean on-policy policy gradient algorithms if not mentioned otherwise.


%% file: deep_learning.tex
\subsection{Deep Learning}\label{sec:deep_learning}

In this section, we introduce deep learning as a subfield of machine learning since its methods are commonly used in policy gradient algorithms. In recent years, deep learning has emerged as the premier machine learning method in various fields, enabling state-of-the-art performance in domains such as computer vision (e.g., \cite{krizhevsky2012imagenet, he2016deep, dosovitskiy2020image}) and natural language processing (e.g., \cite{vaswani2017attention, brown2020language}. Following \cite{lecun2015deep} and \cite{goodfellow2016deep}, we define deep learning as a set of techniques to solve prediction tasks by learning multiple levels of representations from raw data using a composition of simple non-linear functions. This composition of functions, that we will describe in detail later, is referred to as (deep) neural network. Deep learning stands in contrast to conventional machine learning techniques like logistic regressions, which typically require hand-engineered representations as inputs to be effective \cite{lecun2015deep}. In the following, we introduce the general problem setting of deep learning using the notation of \cite{berner2021modern}, formalize neural networks and describe how they are trained.

Consider measurable spaces \(\mathcal{X}\) and \(\mathcal{Y}\). \(\mathcal{Z} \coloneqq \mathcal{X} \times \mathcal{Y}\) is the data space with each element \(z = (x, y) \in \mathcal{Z}\) being a tuple of input features \(x \in \mathcal{X}\) and a label \(y \in \mathcal{Y}\). Let \(\mathcal{M}(\mathcal{X}, \mathcal{Y})\) be the set of measurable functions from \(\mathcal{X}\) to \(\mathcal{Y}\). The problems we encounter in deep learning are prediction tasks. Thus, the goal is to learn a mapping \(f \in \mathcal{M}(\mathcal{X}, \mathcal{Y})\) from inputs to labels by minimizing some loss function \(\mathcal{L}\) over training data \(S = \{z^{(1)}, \hdots, z^{(m)} \} \) such that it generalizes to unseen data \(z \in \mathcal{Z}\). To learn the function \(f\), we first select a hypothesis set \(\mathcal{F} \subset \mathcal{M}(\mathcal{X}, \mathcal{Y})\). Deep learning then provides learning algorithms \(\mathfrak{A} \colon \mathcal{Z} \rightarrow \mathcal{F}\) that use training data \(S\) to learn the desired function \(f = \mathfrak{A}(S)\). Before we discuss this learning process, we will first further characterize the mapping to be learned.

In deep learning, functions in the hypothesis set \(\mathcal{F}\) represent instances of neural networks. Note that here we limit ourselves to feedforward networks, also called multilayer perceptrons (MLP), and will not discuss transformers \cite{vaswani2017attention}, recurrent (RNN) \cite{hochreiter1997long} or convolutional neural networks (CNN)\cite{krizhevsky2012imagenet}.

\begin{definition}
	(Feedforward Neural Network)
	A feedforward neural network \(f \colon \mathcal{X} \rightarrow \mathcal{Y}\) is a composition of functions 
	\[f = f^{(n+1)} \circ \cdots \circ f^{(1)} \] 
	that is differentiable almost everywhere.
	We refer to \(f^{(i)}, \ i = 1, \hdots, n\) as hidden layers, whereas \(f^{(n+1)}\) is the non-hidden output layer. Consequently, \(n\) denotes the number of of hidden layers in the network. Each hidden layer is characterized by a layer width \(N_i\). Let \(N_0\) and \(N_{n+1}\) further be the size of the input and output vectors respectively. Then, we can write each layer as 
	\[f^{(i)}(x) = g\left(W^{(i)}x + b^{(i)}\right),\]
	where \(x\) is the output of the previous layer or the network's inputs for \(i = 1\), \(W^{(i)} \in \mathbb{R}^{N_i \times N_{i-1}} \) and \(b^{(i)} \in \mathbb{R}^{N_i}\) are the layer's weight matrix and bias vector respectively and \(g \colon \mathbb{R} \rightarrow \mathbb{R}\) is a differentiable activation function introducing non-linearity.  \(g\) is applied element-wise.
\end{definition}

An MLP can be characterized by its architecture \(a = \bigl( (N_i)^{n+1}_{i=0}, \ g \bigr)\), consisting of layer sizes and the activation function to be used. The number of layers \(n+1\) is also referred to as the depth of the network. The Universal Approximation Theorem \cite{cybenko1989approximation, hornik1989multilayer} underpins the expressivity of neural networks: a two-layer network can already approximate any measurable function arbitrarily well under weak conditions on the activation function. Technically, each layer can feature a different activation function albeit this is uncommon. The activation function of the output layer is not defined by the architecture \(a\) but is derived from the prediction task. The standard choice for activation functions in the hidden layers is a rectified linear unit (ReLU)\footnote{Note that the ReLU function is not differentiable at 0. In practice, this is circumvented by using its sub-derivatives.} \cite{nair2010rectified}, due to typically fast learning \cite{glorot2011deep}. See \cite{lederer2021activation} for an overview of other commonly used activation functions. The output layer typically uses no activation function for regression tasks and sigmoid or softmax functions for classification tasks. Each element of a layer \(f^{(i)}\) is called a neuron. The outputs \(a^{(i)} = \bigl( f^{(i)} \circ \cdots \circ f^{(1)} \bigr)(x)\) of any layer are the learned representations of the inputs \(x\). We denote the outputs \(f(x)\), i.e. the predictions, of the neural network with \(\hat{y}\). \autoref{fig:dl_nn} depicts a neural network as an acyclic directed graph.

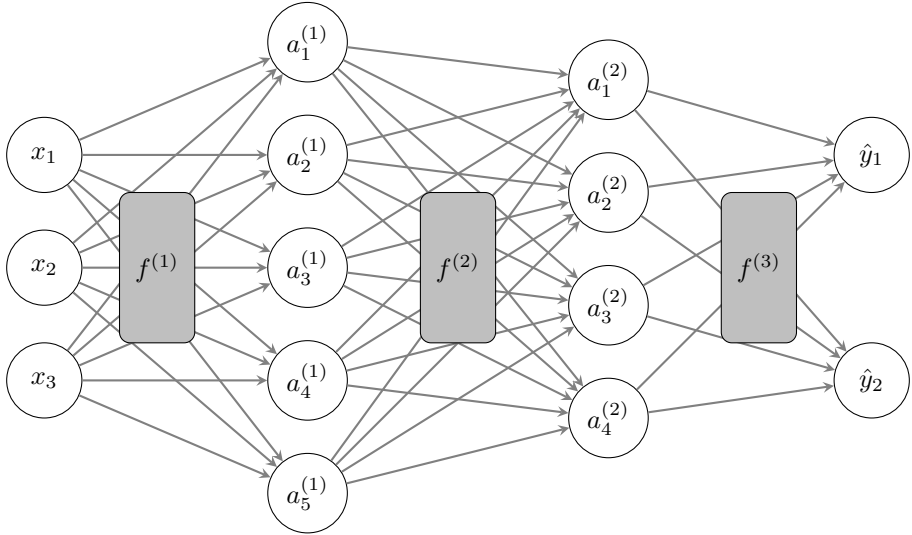
\begin{figure}
\centering
\begin{tikzpicture}[node distance=1.5cm]

\node (inp1) [nn_node] {\(x_1\)};
\node (inp2) [nn_node, below of=inp1] {\(x_2\)};
\node (inp3) [nn_node, below of=inp2] {\(x_3\)};

\node (h13) [nn_node, right of=inp2, xshift=2cm] {\(a^{(1)}_3\)};
\node (h12) [nn_node, above of=h13] {\(a^{(1)}_2\)};
\node (h11) [nn_node, above of=h12] {\(a^{(1)}_1\)};
\node (h14) [nn_node, below of=h13] {\(a^{(1)}_4\)};
\node (h15) [nn_node, below of=h14] {\(a^{(1)}_5\)};

\node (h21) [nn_node, right of=inp1, xshift=6cm, yshift=1cm] {\(a^{(2)}_1\)};
\node (h22) [nn_node, below of=h21] {\(a^{(2)}_2\)};
\node (h23) [nn_node, below of=h22] {\(a^{(2)}_3\)};
\node (h24) [nn_node, below of=h23] {\(a^{(2)}_4\)};

\node (out1) [nn_node, right of=h12, xshift=6cm] {\(\hat{y}_1\)};
\node (out2) [nn_node, right of=h14, xshift=6cm] {\(\hat{y}_2\)};

\draw [arrow] (inp1) -- (h11);
\draw [arrow] (inp1) -- (h12);
\draw [arrow] (inp1) -- (h13);
\draw [arrow] (inp1) -- (h14);
\draw [arrow] (inp1) -- (h15);
\draw [arrow] (inp2) -- (h11);
\draw [arrow] (inp2) -- (h12);
\draw [arrow] (inp2) -- (h13);
\draw [arrow] (inp2) -- (h14);
\draw [arrow] (inp2) -- (h15);
\draw [arrow] (inp3) -- (h11);
\draw [arrow] (inp3) -- (h12);
\draw [arrow] (inp3) -- (h13);
\draw [arrow] (inp3) -- (h14);
\draw [arrow] (inp3) -- (h15);

\draw [arrow] (h11) -- (h21);
\draw [arrow] (h11) -- (h22);
\draw [arrow] (h11) -- (h23);
\draw [arrow] (h11) -- (h24);
\draw [arrow] (h15) -- (h21);
\draw [arrow] (h12) -- (h21);
\draw [arrow] (h12) -- (h22);
\draw [arrow] (h12) -- (h23);
\draw [arrow] (h12) -- (h24);
\draw [arrow] (h15) -- (h22);
\draw [arrow] (h13) -- (h21);
\draw [arrow] (h13) -- (h22);
\draw [arrow] (h13) -- (h23);
\draw [arrow] (h13) -- (h24);
\draw [arrow] (h15) -- (h23);
\draw [arrow] (h14) -- (h21);
\draw [arrow] (h14) -- (h22);
\draw [arrow] (h14) -- (h23);
\draw [arrow] (h14) -- (h24);
\draw [arrow] (h15) -- (h24);

\draw [arrow] (h21) -- (out1);
\draw [arrow] (h22) -- (out1);
\draw [arrow] (h23) -- (out1);
\draw [arrow] (h24) -- (out1);
\draw [arrow] (h21) -- (out2);
\draw [arrow] (h22) -- (out2);
\draw [arrow] (h23) -- (out2);
\draw [arrow] (h24) -- (out2);

\node (f1) [annotation, right of=inp2, xshift=0cm] {\(f^{(1)}\)};
\node (f2) [annotation, right of=inp2, xshift=4cm] {\(f^{(2)}\)};
\node (f3) [annotation, right of=inp2, xshift=8cm] {\(f^{(3)}\)};

\end{tikzpicture}
\caption{A neural network with hidden layers of sizes 5 and 4 as a directed graph.}
\label{fig:dl_nn}
\end{figure}

Selecting a hypothesis set \(\mathcal{F}\) is done implicitly by specifying an architecture \(a\). Hence, we denote the hypothesis set for architecture \(a\), whose elements are all MLPs with that architecture, by \(\mathcal{F}_a\). The MLPs in \(\mathcal{F}_a\) therefore differ only in their weights and biases. We call these the (learnable) parameters of the network and typically collect them in a flattened parameter vector \(\theta \in \mathbb{R}^d\). We denote an MLP with parameters \(\theta\) as \(f_\theta\).

Given a hypothesis set \(\mathcal{F}_a\), we now aim to learn a neural network \(f_\theta \in \mathcal{F}_a\), i.e. to learn parameters \(\theta\), such that we reduce the expected loss or risk, 
\[\mathcal{R}(f) \coloneqq \mathbb{E}_{Z \sim \mathbb{P}_Z} \bigl[ \mathcal{L}(f, Z) \bigr] = \int_{z \in \mathcal{Z}} \mathcal{L}(f, z) \: d \mathbb{P}_Z(z),\]
over the data distribution \(\mathbb{P}_Z\) for some appropriately chosen differentiable loss function \(\mathcal{L} \colon \mathcal{F} \times \mathcal{Z} \rightarrow \mathbb{R}\) \cite{berner2021modern, goodfellow2016deep}. Here \(\mathbb{P}_Z\) is the image measure of \(Z\) on \(\mathcal{Z}\), from which the training data \(S = \{z^{(1)}, \hdots, z^{(m)} \} \) and unknown out-of-sample data \(z\) is drawn. We generally assume that training \(z^{(1)}, \hdots, z^{(m)}\) and out-of-sample data \(z\) are realizations of i.i.d. random variables \(Z^{(1)}, \hdots, Z^{(m)}, Z \sim \mathbb{P}_Z\) \cite{berner2021modern}. For a given MLP \(f_\theta = \mathfrak{A}(S)\) trained on \(S\), the risk becomes \(\mathcal{R}(f_\theta) = \mathbb{E}_{Z \sim \mathbb{P}_Z} \bigl[ \mathcal{L}(f_\theta, Z) \mid S \bigr]\). In practice, noisy data results in a positive lower bound on risk, i.e. an irreducible error \cite{berner2021modern}. Common loss functions are binary cross-entropy loss, 
\[\mathcal{L}(f, (x,y)) = - \bigl( y \cdot \ln(f(x)) + (1 - y) \cdot \ln (1 - f(x)) \bigr),\] 
for (binary) classification and mean squared error (MSE), 
\[\mathcal{L}(f, (x,y)) =  \bigl( y - f(x) \bigr)^2,\] 
for regression tasks. Sometimes, loss functions are augmented by regularization terms \(\Omega(\theta)\) such as an L2-penalty of the parameters, i.e. \(\beta \lVert \theta \rVert^2_2\) with \(\beta \in \mathbb{R}\) \cite{goodfellow2016deep}.

The data distribution \(\mathbb{P}_Z\) is generally unknown. Hence, we replace it by an empirical distribution based on the sampled training data \(S\) and use empirical risk minimization (ERM) as the learning algorithm to minimize it \cite{goodfellow2016deep, berner2021modern}. 

\begin{definition}
	(Empirical Risk). Given training data \(S = \{z^{(1)}, \hdots, z^{(m)} \}\) and a function \(f_\theta \in \mathcal{M}(\mathcal{X}, \mathcal{Y})\), the empirical risk is defined by
	\begin{equation}
		\hat{\mathcal{R}}_S(f_\theta) \coloneqq \frac{1}{m} \sum^m_{i=1} \mathcal{L}(f_\theta, z^{(i)}). 
		\label{eq:empiricalrisk}
	\end{equation}
\end{definition}

\begin{definition}
	(ERM learning algorithm). Given hypothesis set \(\mathcal{F}_a\) and training data \(S\), an empirical risk minimization algorithm \(\mathfrak{A}^\mathrm{erm}\) terminates with an (approximate\footnote{In practice, the empirical risk is generally highly non-convex prohibiting guaranteed convergence to a global minimum \cite{berner2021modern}.}) minimizer \(\hat{f}_S \in \mathcal{F}_a\) of empirical risk:
	\[\mathfrak{A}^\mathrm{erm}(S) = \hat{f}_S \in \argmin_{f \in \mathcal{F}_a} \hat{\mathcal{R}}_S(f). \]
\end{definition}

We approximately minimize empirical risk typically via gradient-based methods due to efficient computation of point-wise derivatives via the backpropagation algorithm \cite{rumelhart1985learning, kelley1960gradient}. Backpropagation means the practical application of the chain rule to neural networks. The gradient of the objective function \(\mathcal{L}\) with respects to the \(i\)-th layer's inputs \(a^{(i-1)}\) can be computed by working backwards from the gradient with respects to the layer's outputs \(a^{(i)}\) as \(\nabla_{a^{(i-1)}} \mathcal{L} = \sum_j (\nabla_{a^{(i-1)}} a^{(i)}_j) \frac{\partial \mathcal{L}}{\partial a^{(i)}_j}\). From these gradients, the gradients with respects to the weights and biases in each layer can be calculated similarly. Due to this flow of information from the objective function to each of the layers, the optimization of MLPs is also referred to as backwards pass, in contrast to the forward pass of calculating \(\hat{y} = f(x)\). The full backpropagation algorithm for MLPs is formulated in Algorithm \ref{alg:backprop}.

\begin{algorithm}
	\caption{Backpropagation, pseudocode taken from \cite{goodfellow2016deep}}\label{alg:backprop}
	\begin{algorithmic}
	\Require labels $y$, regularizer $\Omega(\theta)$, network outputs $\hat{y}$, activated and unactivated layer outputs $a^{(k)}$ and $h^{(k)}$ for $k = 1, \hdots, n$, activation function $g$, loss $\mathcal{L}$
	\State $\delta \gets \nabla_{\hat{y}}\mathcal{L}$
	\For{$k = n, \hdots, 1$}
	\State $\delta \gets \nabla_{h^{(k)}}\mathcal{L} = \delta \odot g'(h^{(k)}) $ \Comment{hadamard product if $g$ is element-wise}
	\State $\nabla_{b^{(k)}}\mathcal{L} \gets \delta + \nabla_{b^{(k)}} \Omega(\theta)$
	\State $\nabla_{W^{(k)}}\mathcal{L} \gets \delta h^{(k-1)\mathsf{T}} + \nabla_{W^{(k)}} \Omega(\theta)$
	\State $\delta \gets \nabla_{a^{(k-1)}}\mathcal{L} = W^{(k)\mathsf{T}} \delta$ 
	\EndFor
	\end{algorithmic}
\end{algorithm}

The gradients computed via backpropagation are used to update the parameters in each layer using gradient descent. However, due to the prohibitive computational costs of evaluating the expectation in Equation \eqref{eq:empiricalrisk}, computing gradients only on a subset of the training data is generally preferred and typically also results in faster convergence \cite{goodfellow2016deep}. At each iteration, a batch \(S'\) of data with size \(m' \leq m\) (typically \(m' \ll m\)) is randomly sampled from the training data to conduct the update \cite{berner2021modern}
\begin{equation}
	\Theta^{(k)} \coloneqq \Theta^{(k-1)} - \alpha_k \frac{1}{m'} \sum_{z \in S'} \nabla_\theta \mathcal{L}\Bigl(f_{\Theta^{(k-1)}}, z\Bigr).
	\label{eq:sgd_update}
\end{equation}
 Here, \(\Theta\) is a random variable whose realizations are neural network parameters \(\theta\). \(\alpha_k\) is the step size or learning rate on the k-th optimization step. The learning rate is commonly decayed over the training process to help convergence \cite{goodfellow2016deep}. The procedure using updates as in Equation \eqref{eq:sgd_update} is known as stochastic (minibatch) gradient descent (SGD)\footnote{Sometimes SGD refers to updates which only involve a single data point. We however follow the nowadays common terminology of calling any sample-based gradient descent stochastic.} \cite{goodfellow2016deep} and dates back to \cite{robbins1951stochastic, kiefer1952stochastic}. Using SGD has the additional benefit of introducing random fluctuations which enable escaping saddle points \cite{berner2021modern}. SGD in its general form is depicted in Algorithm \ref{alg:sgd}, where in our context \(r(\theta) = \hat{\mathcal{R}}_S(f_\theta)\). The neural network parameters \(\theta\) are set to be the realization of the final \(\Theta^{(K)}\) or a convex combination of \(\bigl( \Theta^{(k)} \bigr)^K_{k=1}\) \cite{berner2021modern}.

\begin{algorithm}
	\caption{Stochastic Gradient Descent, pseudocode from \cite{berner2021modern}}\label{alg:sgd}
	\begin{algorithmic}
	\Require Differentiable function $r \colon \mathbb{R}^d \rightarrow \mathbb{R}$, step sizes $\alpha_k \in (0, \infty)$, $k = 1, \hdots, K$, $\mathbb{R}^d$-valued random variable $\Theta^{(0)}$
	\For{$k = 1, \hdots, K$}
	\State Let $D^{k}$ be a random variable such that $\mathbb{E}\bigl[ D^{(k)} \mid \Theta^{(k-1)} \bigr] = \nabla r (\Theta^{(k-1)})$
	\State $\Theta^{(k)} \gets \Theta^{(k-1)} - \alpha_k D^{(k)}$
	\EndFor
	\end{algorithmic}
\end{algorithm}

Despite the stochasticity of SGD and highly non-convex loss landscapes, SGD's convergence can be guaranteed in some regimes \cite{berner2021modern}, and it exhibits strong performance in practice \cite{goodfellow2016deep}. Hence, SGD and its variants are the default choice to optimize neural networks. The most prominently used variant is Adam \cite{kingma2014adam}, which uses momentum \cite{polyak1964some} and an adaptive scaling of gradients to stabilize learning. Nonetheless, the initialization of \(\theta\) is also important for convergence. Biases are commonly initialized to 0 whereas weights are randomly initialized close to 0 using various strategies \cite{goodfellow2016deep}. Finally, note that regardless of the non-convexity of the loss landscapes, local minima are not considered problematic if the neural networks are large enough \cite{dauphin2014identifying, choromanska2015loss}

In practice, the training of neural networks is an iterative process.  We alternate between choosing the network architecture \(a\) as well as further hyperparameters of the learning algorithm such as the learning rates \(\alpha\), and approximately minimizing the empirical risk for this set of hyperparameters. This is generally a trial-and-error process to find a suitable set of hyperparameters to maximize generalization performance, i.e. to minimize risk. To approximate risk, the trained models are typically evaluated by the empirical risk on a held-out test data set, which was not seen during training \cite{goodfellow2016deep}. The achieved empirical risk can be decomposed into a generalization error, an optimization error, an approximation error and the irreducible error \cite{berner2021modern}. The generalization error is the difference between empirical and actual risk stemming from the random sampling of training data, which may not be representative of the actual data distribution \(\mathbb{P}_Z\). The optimization error is the result of potentially not finding a global mimimum during the learning process. The approximation error is the difference between the minimum achievable risk over functions in \(\mathcal{F}_a\) and over all \(f \in \mathcal{M}(\mathcal{X}, \mathcal{Y})\).

%% file: pg_theorem.tex
\section{Theoretical Foundations of Policy Gradients}\label{sec:pg_theory}


Having introduced the fundamentals of deep RL, we can now discuss policy gradient algorithms in detail. In this section, we derive their theoretical foundations. Our main focus is going to be the Policy Gradient Theorem, on which all policy gradient algorithms build. This theorem will be discussed in Section \ref{sec:pg_theorem}. Furthermore, Sections \ref{sec:baseline} and \ref{sec:importance} introduce the theoretical justifications for additional methods that are frequently used in policy gradient algorithms.

\input{policy_gradient_theorem.tex}
\input{baseline.tex}

\input{importance_sampling.tex}

%% file: policy_gradient_theorem.tex
\subsection{Policy Gradient Theorem}\label{sec:pg_theorem}

Given an MDP \(\mathcal{M} = (\mathcal{S}, \mathcal{A}, P, \gamma, p_0)\), consider a parameterized policy \(\pi_\theta\), which is differentiable almost everywhere, and the following objective function \(J\) for maximizing the expected episodic return: 
\begin{align*}
	J(\theta) &= \mathbb{E}_{S_0 \sim p_0, \pi_\theta} \bigl[ G_0 \bigr] \\
	&= \mathbb{E}_{S_0 \sim p_0} \Bigl[ \mathbb{E}_{\pi_\theta} \bigl[ G_t \mid S_t = S_0 \bigr] \Bigr] \\
	&= \mathbb{E}_{S_0 \sim p_0} \Bigl[ V_{\pi_\theta}(S_0) \Bigr]
\end{align*}
The idea of policy gradient algorithms is to maximize \(J(\theta)\) over the parameters \(\theta\) by performing gradient ascent \cite{Sutton1998}. Hence, we require the gradients \(\nabla_\theta J(\theta)\), however it is a priori not obvious how the right-hand side \(\mathbb{E}_{S_0 \sim p_0, \pi_\theta} \bigl[ G_0 \bigr]\) depends on \(\theta\) as changes in the policy \(\pi\) also affect the state distribution \(d^\pi\). The Policy Gradient Theorem \cite{sutton1999policy, marbach2001simulation} yields an analytic form of \(\nabla_\theta J(\theta)\) from which we can sample gradients that does not involve the derivative of \(d^\pi\). Here, we focus on the undiscounted case, i.e. \(\gamma = 1\). Note that any discounted problem instance can be reduced to the undiscounted case by letting the reward function absorb the discount factor \cite{schulman2015high}.

\begin{theorem} \label{pg_theorem}
	(Policy Gradient Theorem)
	For a given MDP, let \(\pi_\theta\) be differentiable w.r.t. \(\theta\) and \(\nabla_\theta \pi_\theta\) be bounded, let \(Q_{\pi_\theta}\) be differentiable w.r.t. \(\theta\) and \(\nabla_\theta Q_{\pi_\theta}\) be bounded for all \(s \in \mathcal{S}\) and \(a \in \mathcal{A}\). Then, there exists a constant \(\eta\) such that
	\begin{equation}
		\nabla_\theta J(\theta) = \eta \: \mathbb{E}_{S \sim d^{\pi_\theta}, A \sim \pi_\theta} \Bigl[ Q_{\pi_\theta}(S,A) \: \nabla_\theta \ln \pi_\theta(A \mid S) \Bigr]. \label{eq:pg_theorem}
	\end{equation}
\end{theorem}

\begin{proof}
	We largely follow the proof by \cite{Sutton1998} albeit in a more detailed form  and extended to continuous state and action spaces.
	To enhance readability, we omit subscripts \(\theta\) for the policy \(\pi\) and all gradients \(\nabla\) but both always depend on the parameters \(\theta\).
	
	Starting from the definition of the objective function, we explicitly write out the expectation over starting states, use the relationship between value and action-value function, 
	\( V_\pi (s) = \int_{a \in \mathcal{A}} \pi(a \mid s) \: Q_\pi (s, a) \: da, \) 
	and differentiate by parts.
	\begin{align}
		\nabla J(\theta) &= \nabla\mathbb{E}_{S \sim p_0} \bigl[ V_\pi(S) \bigr] \nonumber \\ 
		&= \nabla \int_{s \in \mathcal{S}} p_0(s) \: V_\pi(s) \:ds \nonumber \\
		&= \nabla \int_{s \in \mathcal{S}} p_0(s) \int_{a \in \mathcal{A}} \pi(a \mid s) \: Q_\pi(s,a) \: da \: ds \nonumber \\ 
		&= \int_{s \in \mathcal{S}} p_0(s) \biggl( \int_{a \in \mathcal{A}} \bigl(\nabla \pi(a \mid s)\bigr) \: Q_\pi(s,a) \: da  +  \int_{a \in \mathcal{A}} \pi(a \mid s) \: \nabla Q_\pi(s,a) \: da \biggr) \: ds. \label{eq:pgproof_1}
	\end{align}
	Note that in the last step via used the Leibniz integral rule (Theorem \ref{th:leibniz}) to swap the order of integration and differentiation prior to applying the product rule. 
	The conditions for Leibniz are satisfied since \(\pi(\cdot \mid s) Q_\pi(s, \cdot)\) is integrable for any \(s \in \mathcal{S}\) and its partial derivatives exist and are bounded for all \(s \in \mathcal{S}\) and \(a \in \mathcal{A}\) since \(\pi\) and \(Q_\pi\) are bounded and \(\nabla Q_\pi\) and \(\nabla \pi\) exist and are bounded by assumption.
	
	Now, consider the recursive formulation of the action-value function 
	\[Q_\pi(s,a) = \int_{s' \in \mathcal{S}} \int_{r \in \mathbb{R}} P(s', r \mid s,a) \: \bigl(r + V_\pi(s')\bigr) \: dr \: ds'.\]
	Due to the identity \(\int_{r \in \mathbb{R}} P(s', r \mid s,a) \: dr = P(s' \mid s,a)\) and since realized rewards \(r\) and environment transitions for a given action no longer depend on the policy, we can reformulate the gradients of \(Q_\pi\) w.r.t. \(\theta\), again using the Leibniz integral rule.
	\begin{align}
		\nabla Q_\pi(s,a) &= \nabla \int_{s' \in \mathcal{S}} \int_{r \in \mathbb{R}} P(s', r \mid s,a) \: \bigl(r + V_\pi(s')\bigr) \: dr \: ds' \nonumber \\
		&= \int_{s' \in \mathcal{S}} \int_{r \in \mathbb{R}} P(s', r \mid s,a) \: \nabla \bigl(r + V_\pi(s')\bigr) \: dr \: ds' \nonumber \\ 
		&= \int_{s' \in \mathcal{S}} \int_{r \in \mathbb{R}} P(s', r \mid s,a) \: \nabla V_\pi(s') \: dr \: ds' \nonumber \\ 
		&= \int_{s' \in \mathcal{S}} \biggl( \int_{r \in \mathbb{R}} P(s', r \mid s,a) \: dr \biggr) \: \nabla V_\pi(s') \: ds' \nonumber \\ 
		&= \int_{s' \in \mathcal{S}} P(s' \mid s,a) \: \nabla V_\pi(s') \: ds' \label{eq:pgproof_2}.
	\end{align}
	Further, note that for all \(s \in \mathcal{S}\)
	\begin{align}
		\nabla V_\pi(s) &= \nabla \int_{a \in \mathcal{A}} \pi(a \mid s) \: Q_\pi(s,a) \: da \nonumber \\
		&= \int_{a \in \mathcal{A}} \bigl(\nabla \pi(a \mid s)\bigr) \: Q_\pi(s,a) \: da + \int_{a \in \mathcal{A}} \pi(a \mid s) \: \nabla Q_\pi(s,a) \: da, \label{eq:pgproof_3}
	\end{align}
	which is equivalent to the inner expression in Equation \eqref{eq:pgproof_1}. By using \eqref{eq:pgproof_2} and \eqref{eq:pgproof_3}, we can transform \eqref{eq:pgproof_1} into a recursive form, which we are then going to unroll subsequently to yield an explicit form. In the following, we simply notation by defining 
	\begin{equation}
		\label{eq:pgproof_4}
		\phi(s) \coloneqq \int_{a \in \mathcal{A}} \bigl(\nabla \pi(a \mid s)\bigr) \: Q_\pi(s,a) \: da.
	\end{equation}
	Applying \eqref{eq:pgproof_4} and \eqref{eq:pgproof_2} to \eqref{eq:pgproof_1} in order and rearranging the integrals gives
	\begin{align}
		\nabla J(\theta) &= \int_{s \in \mathcal{S}} p_0(s) \: \biggl( \int_{a \in \mathcal{A}} \bigl(\nabla \pi(a \mid s)\bigr) \: Q_\pi(s,a) da + \int_{a \in \mathcal{A}} \pi(a \mid s) \: \nabla Q_\pi(s,a) \: da \biggr) \: ds \nonumber \\
		&= \int_{s \in \mathcal{S}} p_0(s) \: \biggl( \phi(s) +  \int_{a \in \mathcal{A}} \pi(a \mid s) \: \nabla Q_\pi(s,a) \: da \biggr) \: ds \nonumber \\
		&= \int_{s \in \mathcal{S}} p_0(s) \: \biggl( \phi(s) + \int_{a \in \mathcal{A}} \pi(a \mid s) \: \int_{s' \in \mathcal{S}} P(s' \mid s,a) \: \nabla V_\pi(s') \: ds' \: da \biggr) \: ds \nonumber \\
		&= \int_{s \in \mathcal{S}} p_0(s) \: \biggl( \phi(s) + \int_{s' \in \mathcal{S}} \int_{a \in \mathcal{A}} \pi(a \mid s) \: P(s' \mid s,a) \: da \: \nabla V_\pi(s') \: ds' \biggr) \: ds \label{eq:pgproof_5}
	\end{align}
	In the final step, we switched the order of integration using Fubini's Theorem (Theorem \ref{th:fubini}), which is applicable since \(\nabla V_\pi\) is bounded and \(\pi(\cdot \mid s) P(\cdot \mid s, \cdot)\) is a probability measure on \(\mathcal{S} \times \mathcal{A}\) such that \(\lvert \pi(\cdot \mid s) P(\cdot \mid s, \cdot) \nabla V_\pi \rvert\) is integrable over the product space \(\mathcal{S} \times \mathcal{A}\). To unroll Equation \eqref{eq:pgproof_5} across time, we introduce notation for multi-step transition probabilities. Let \(\rho_\pi (s \to s', k)\) be the probability of transitioning from state \(s\) to \(s'\) after \(k\) steps under policy \(\pi\). We have that 
	\[\rho_\pi (s \to s', 0) \coloneqq 
	\begin{cases}
		1 & \text{if } s = s',\\
		0 & \text{else}
	\end{cases}
	\] 	
	and \(\rho_\pi (s \to s', 1) \coloneqq \int_{a \in \mathcal{A}} \pi(a \mid s) \: P(s' \mid s,a) \: da\). 
	Now, we can recursively write
	\[\rho_\pi (s \to s'', k+1) = \int_{s' \in \mathcal{S}} \rho_\pi (s \to s', k) \: \rho_\pi (s' \to s'', 1) \: ds'.\]
	Using this notation, iteratively substituting in \eqref{eq:pgproof_2} and \eqref{eq:pgproof_3} and applying Fubini, we can unroll \eqref{eq:pgproof_5}:
	\begin{align*}
		\nabla J(\theta) &= \int_{s \in \mathcal{S}} p_0(s) \: \biggl( \phi(s) + \int_{s' \in \mathcal{S}} \int_{a \in \mathcal{A}} \pi(a \mid s) \: P(s' \mid s,a) \: da \: \nabla V_\pi(s') \: ds' \biggr) \: ds \\
		&= \int_{s \in \mathcal{S}} p_0(s) \: \biggl( \phi(s) + \int_{s' \in \mathcal{S}} \rho_\pi (s \to s', 1) \: \nabla V_\pi(s') \: ds' \biggr) \: ds \\
		&= \int_{s \in \mathcal{S}} p_0(s) \: \Biggl( \phi(s) + \int_{s' \in \mathcal{S}} \rho_\pi (s \to s', 1) \: \biggl( \phi(s') + \int_{a \in \mathcal{A}} \pi(a \mid s') \: \nabla Q_\pi(s',a) \: da \biggr) \: ds' \Biggr) \: ds \\
		&= \int_{s \in \mathcal{S}} p_0(s) \: \Biggl( \phi(s) + \int_{s' \in \mathcal{S}} \rho_\pi (s \to s', 1) \: \biggl( \phi(s') + \int_{s'' \in \mathcal{S}} \rho_\pi (s' \to s'', 1) \: \nabla V_\pi(s'') \: ds'' \biggr) \: ds' \Biggr) \: ds \\ 
		&= \int_{s \in \mathcal{S}} p_0(s) \: \Biggl( \phi(s) + \int_{s' \in \mathcal{S}} \rho_\pi (s \to s', 1) \: \phi(s') \: ds' \\
		&\qquad + \int_{s'' \in \mathcal{S}} \biggl(\int_{s' \in \mathcal{S}} \rho_\pi (s \to s', 1) \: \rho_\pi (s' \to s'', 1) \: ds'\biggr) \: \nabla V_\pi(s'') \: ds'' \: \Biggr) \: ds \\ 
		&= \int_{s \in \mathcal{S}} p_0(s) \: \biggl( \phi(s) + \int_{s' \in \mathcal{S}} \rho_\pi (s \to s', 1) \: \phi(s') \: ds' + \int_{s'' \in \mathcal{S}} \rho_\pi (s \to s'', 2) \: \nabla V_\pi(s'') \: ds'' \biggr) \: ds \\
		&= \int_{s \in \mathcal{S}} p_0(s) \: \biggl( \rho_\pi (s \to s, 0) \phi(s) + \int_{s' \in \mathcal{S}} \rho_\pi (s \to s', 1) \: \phi(s') \: ds' \\ 
		&\qquad + \int_{s'' \in \mathcal{S}} \rho_\pi (s \to s'', 2) \: \phi(s'') \: ds'' + \int_{s''' \in \mathcal{S}} \rho_\pi (s \to s''', 3) \: \nabla V_\pi(s''') \: ds''' \biggr) \: ds \\
		&\;\;\vdots \\
		&= \int_{s \in \mathcal{S}} p_0(s) \int_{s' \in \mathcal{S}} \sum^T_{t=0} \rho_\pi (s \to s', t) \: \phi(s') \: ds' \: ds
	\end{align*}
	We set \(\eta_s(s') \coloneqq \sum^T_{t=0} \rho_\pi (s \to s', t)\), rearrange the integrals and multiply by 1 to obtain 
	\begin{align}
		\nabla_\theta J(\theta) &= \int_{s \in \mathcal{S}} p_0(s) \int_{s' \in \mathcal{S}} \sum^T_{t=0} \rho_\pi (s \to s', t) \: \phi(s') \: ds' \: ds \nonumber \\
		&= \int_{s' \in \mathcal{S}} \int_{s \in \mathcal{S}} p_0(s) \: \eta_s(s') \: \phi(s') \: ds \: ds' \nonumber \\
		&= \frac{\int_{s'' \in \mathcal{S}} \int_{s \in \mathcal{S}} p_0(s) \eta_s(s'') \: ds \: ds''}{\int_{s'' \in \mathcal{S}} \int_{s \in \mathcal{S}} p_0(s) \eta_s(s'') \: ds \: ds''} \int_{s' \in \mathcal{S}} \int_{s \in \mathcal{S}} p_0(s) \: \eta_s(s') \: ds \: \phi(s') \: ds' \nonumber \\
		&= \int_{s'' \in \mathcal{S}} \int_{s \in \mathcal{S}} p_0(s) \eta_s(s'') \: ds \: ds'' \int_{s' \in \mathcal{S}} \frac{\int_{s \in \mathcal{S}} p_0(s) \eta_s(s') \: ds}{\int_{s'' \in \mathcal{S}} \int_{s \in \mathcal{S}} p_0(s) \eta_s(s'') \: ds \: ds''} \:  \: \phi(s') \: ds' \nonumber \\
		&= \int_{s \in \mathcal{S}} p_0(s) \int_{s'' \in \mathcal{S}} \eta_s(s'') \: ds'' \: ds \int_{s' \in \mathcal{S}} d^\pi(s') \: \phi(s') \: ds'. \label{eq:pgproof_6}
	\end{align}
	In the final step, we used the identity \[d^\pi(s') = \frac{\int_{s \in \mathcal{S}} p_0(s) \eta_s(s') \: ds}{\int_{s'' \in \mathcal{S}} \int_{s \in \mathcal{S}} p_0(s) \eta_s(s'') \: ds \: ds''},\]
	which can be seen as \(\eta_s(s')\) is the accumulate sum over probabilities of reaching \(s'\) after any number of steps for a given starting state. Integrating over the starting state distribution and normalizing hence yields the probability of visiting state \(s'\) and thereby the stationary distribution \(d^\pi\) over states under the current policy.
	
	Finally, we can derive the canonical form of the Policy Gradient Theorem from \eqref{eq:pgproof_6} by using the definition of \(\phi(s)\), setting 
	\begin{equation*}
		\eta \coloneqq \int_{s \in \mathcal{S}} p_0(s) \int_{s'' \in \mathcal{S}} \eta_s(s'') \: ds'' \: ds 
	\end{equation*}
	 and multiplying with 1:
	\begin{align}
		\nabla J(\theta) &= \int_{s \in \mathcal{S}} p_0(s) \int_{s'' \in \mathcal{S}} \eta_s(s'') \: ds'' \: ds \int_{s' \in \mathcal{S}} d^\pi(s') \: \phi(s') \: ds' \nonumber \\
		&= \eta \int_{s' \in \mathcal{S}} d^\pi(s') \int_{a \in \mathcal{A}} \bigl(\nabla \pi(a \mid s')\bigr) \: Q_\pi(s',a) \: da \: ds' \nonumber \\
		&= \eta \int_{s' \in \mathcal{S}} d^\pi(s') \int_{a \in \mathcal{A}} \pi(a \mid s') \: \frac{\nabla \pi(a \mid s')}{\pi(a \mid s')} \: Q_\pi(s',a) \: da \: ds' \nonumber \\
		&= \eta \int_{s' \in \mathcal{S}} d^\pi(s') \int_{a \in \mathcal{A}} \pi(a \mid s') \: \bigl(\nabla \ln \pi(a \mid s')\bigr) \: Q_\pi(s',a) \: da \: ds' \nonumber \\
		&= \eta \: \mathbb{E}_{S \sim d^\pi} \biggl[ \mathbb{E}_{A \sim \pi} \Bigl[Q_\pi(S,A) \: \nabla \ln \pi(A \mid S) \Bigr] \biggr]. \nonumber
	\end{align}
\end{proof}

The Policy Gradient Theorem provides us with an explicit form of the policy gradients from which we can sample gradients. This allows the use of gradient-based optimization to directly optimize the policy using the methods presented in Section \ref{sec:deep_learning}. Thus, the theorem serves as the foundation for the policy gradient algorithms which we will discuss in Section \ref{sec:algorithms}.

We conclude this section with some further remarks on Equation \eqref{eq:pg_theorem}. First, we note that for any starting state \(s \in \mathcal{S}\), we have that 
\begin{align*}
	\eta &= \int_{s \in \mathcal{S}} p_0(s) \int_{s' \in \mathcal{S}} \eta_s(s') \: ds' \: ds = \int_{s \in \mathcal{S}} p_0(s) \int_{s' \in \mathcal{S}} \sum^T_{t=0} \rho_\pi (s \to s', t) \: ds' \: ds \\
	&= \mathbb{E}_{S \sim p_0} \Biggl[ \sum^T_{t=0} \int_{s' \in \mathcal{S}} \rho_\pi (S \to s', t) \: ds' \Biggr],
\end{align*}
which is the average episode length\footnote{Note that \(\rho_\pi (s_t \to s_{t+1}, 1) = 0\) if the episode already terminated due to reaching a terminal state on any previous step.} under policy \(\pi\) \cite{Sutton1998}. Second, the use of gradient-based methods makes it sufficient to sample gradients which are only proportional to the actual gradients since any constant of proportionality can be absorbed by the learning rate parameter of the optimization algorithms. Hence, \(\eta\) is commonly omitted \cite{Sutton1998}, i.e. 
\begin{equation}
	\nabla_\theta J(\theta) \propto \mathbb{E}_{S \sim d^{\pi_\theta}, A \sim \pi_\theta} \Bigl[ Q_{\pi_\theta}(S,A) \: \nabla_\theta \ln \pi_\theta(A \mid S) \Bigr]. \label{eq:pg_proportional}
\end{equation}
	We observe that all terms on the right hand side are known or can be estimated via sampling.
		

%% file: baseline.tex
\subsection{Value Function Estimation with Baselines}\label{sec:baseline}

In practice, the resulting estimates of the policy gradients can become very noisy when sampling from Equation \eqref{eq:pg_proportional}. Therefore, a main practical challenge of policy gradient algorithms is to introduce measures to reduce the variance of the gradients while keeping the bias low \cite{sutton2000comparing}. In this context, a well-known and widely used technique is to use a baseline \cite{williams1992simple} when sampling an estimate of the action-value function \(Q_\pi\) \cite{greensmith2004variance}. In this section, we show that using an appropriately chosen baseline does not bias the estimate but can greatly reduce the variance of the sampled gradients.

Let \(\hat{Q}(s,a)\) be a sampled estimate of \(Q_\pi(s,a)\), assuming \(\mathbb{E}\bigl[ \hat{Q}(s,a) \bigr]= Q_\pi(s,a)\). Then, we can construct a new estimator \(\hat{Q}_b(s,a)\) by subtracting some baseline \(b \colon \mathcal{S} \rightarrow \mathbb{R}\), i.e. \(\hat{Q}_b(s,a) = \hat{Q}(s,a) - b(s)\). Our only condition towards \(b\) is that it does not depend on the action \(a\), though it can depend on the state \(s\) and even be a random variable \cite{Sutton1998}. Our sampled estimate of the gradient \(\nabla_\theta J(\theta)\) becomes
\begin{equation*}
	\hat{\nabla}_\theta J(\theta) = \nabla_\theta \ln \pi_\theta(a \mid s) \bigl( \hat{Q}(s,a) - b(s) \bigr).
\end{equation*}
In expectation over the policy \(\pi\), this yields
\begin{align*}
	\mathbb{E}_\pi \Bigl[ \hat{\nabla}_\theta J(\theta) \Bigr] &= \mathbb{E}_\pi \Bigl[\nabla_\theta \ln \pi_\theta(A \mid S) \: \bigl( \hat{Q}(S,A) - b(S) \bigr) \Bigr] \\
	&= \mathbb{E}_\pi \Bigl[\nabla_\theta \ln \pi_\theta(A \mid S) \: \hat{Q}(S,A) \Bigr] - \mathbb{E}_\pi \Bigl[\nabla_\theta \ln \pi_\theta(A \mid S) \: b(S) \Bigr]
\end{align*}
using the linearity of the expectation. Now, we show that the second part is 0. Using the Leibniz integral rule, we have that 
\begin{align*}
	\mathbb{E}_{S \sim d^{\pi_\theta}, A \sim \pi_\theta} \Bigl[\nabla_\theta \ln \pi_\theta(A \mid S) \: b(S) \Bigr] &= \int_{s \in \mathcal{S}} d^\pi(s) \int_{a \in \mathcal{A}} \pi_\theta(a \mid s) \: \nabla_\theta \ln \pi_\theta(a \mid s) \: b(s) \: da \: ds \\
	&= \int_{s \in \mathcal{S}} d^\pi(s) \: b(s) \int_{a \in \mathcal{A}} \pi_\theta(a \mid s) \: \nabla_\theta \ln \pi_\theta(a \mid s) \: da \: ds \\
	&= \int_{s \in \mathcal{S}} d^\pi(s) \: b(s) \int_{a \in \mathcal{A}} \pi_\theta(a \mid s) \frac{\nabla_\theta \pi_\theta(a \mid s)}{\pi_\theta(a \mid s)} \: da \: ds \\
	&= \int_{s \in \mathcal{S}} d^\pi(s) \: b(s) \: \nabla_\theta \int_{a \in \mathcal{A}} \pi_\theta(a \mid s) \: da \: ds \\
	&= \int_{s \in \mathcal{S}} d^\pi(s) \: b(s) \: \nabla_\theta \: 1 \: ds \\
	&= 0
\end{align*}
since \(\pi(\cdot \mid s)\) is a probability distribution over actions.
Thus, subtracting an action-independent baseline \(b\) from an action-value function estimator \(\hat{Q}\) does indeed not add any bias to the gradient estimate. While here we have shown this for a baseline which only depends on the current state, this result can be extended to baselines which depend on the current and all subsequent states \cite{schulman2015high}.

Next, we analyze the effect on the variance of the gradient estimates. Here, we only provide an approximate explanation,
see \cite{greensmith2004variance} for a more thorough analysis which derives bounds of the true variance. We can compute the variance using \(\mathrm{Var}[X] = \mathbb{E}[X^2] - \mathbb{E}[X]^2\). Due to the above, \(\mathbb{E}[X]^2\) is independent of the baseline in our case. This yields
\begin{align*}
	\argmin_b \mathrm{Var}_\pi \Bigl[\nabla_\theta \ln \pi_\theta(A \mid S) \bigl( \hat{Q}(S,A) - b(S) \bigr) \Bigr] 
	&= \argmin_b \mathbb{E}_\pi\Bigl[ \Bigl( \nabla_\theta \ln \pi_\theta(A \mid S) \bigl( \hat{Q}(S,A) - b(S) \bigr) \Bigr)^2 \Bigr] \\
	&\approx \argmin_b \biggl( \mathbb{E}_\pi\bigl[ \bigl( \nabla_\theta \ln \pi_\theta(A \mid S) \bigr)^2 \bigr] \cdot \mathbb{E}_\pi \bigl[ \bigl( \hat{Q}(S,A) - b(S) \bigr)^2 \bigr] \biggr),
\end{align*}
where we approximated the variance by assuming independence of the two terms in the second step. Under this approximation, the variance of sampled gradients can be minimized by minimizing \(\mathbb{E}_\pi \bigl[ \bigl( \hat{Q}(S,A) - b(S) \bigr)^2 \bigr]\). This is a common least squares problem resulting in the optimal choice of \(b(s) = \mathbb{E}_\pi[\hat{Q}(s,A)]\) (see Theorem \ref{th:least_squares}). This result indicates that an appropriately chosen baseline can potentially significantly reduce variance of the gradients. Using this choice for the baseline, we would like to compute gradients for sampled states and actions as
\begin{align*}
	\nabla_\theta \ln \pi_\theta(a \mid s) \bigl( Q_\pi(s,a) - \mathbb{E}_{A \sim \pi_\theta}[Q_\pi(s,A)] \bigr) &= \nabla_\theta \ln \pi_\theta(a \mid s) \bigl( Q_\pi(s,a) - V_\pi(s) \bigr) \\
	&= \nabla_\theta \ln \pi_\theta(a \mid s) A_\pi(s,a).
\end{align*}
Here, we used the relation of the value function \(V_\pi\) to \(Q_\pi\) and the definition of the advantage function \(A_\pi\). Despite our approximations, this choice of a baseline turns out to yield almost the lowest possible variance of the gradients \cite{schulman2015high}. However, note that in practice the advantage function must also be estimated. Learning this estimate typically introduces bias \cite{konda2003onactor, sutton1999policy}.


%% file: importance_sampling.tex
\subsection{Importance Sampling}\label{sec:importance}

Importance sampling is a technique to calculate expectations under one distribution given samples from another \cite{rubinstein1981simulation, hesterberg1988advances, Sutton1998}. Traditionally, this is only needed in off-policy RL, where we sample transitions using a behavior policy \(\beta\) but want to calculate expectations over the target policy \(\pi\). However, in some implementations of on-policy algorithms the policy may be updated before all data generated by it is processed. This makes these implementations slightly off-policy and thus importance sampling becomes relevant even for theoretically on-policy algorithms \cite{weng2018PG}. We build our presentation of importance sampling on \cite{Sutton1998}, Section 5.5. 

 Given a behavior policy \(\beta\), we want to estimate the value function \(V_\pi\) of our target policy \(\pi\).
Generally, we have
\[ V_\beta(s) = \mathbb{E}_{\beta} \bigl[G_t \mid S_t = s \bigr] \neq V_\pi(s). \]
We can calculate the probability of a trajectory \((a_t, s_{t+1}, a_{t+1}, \hdots, a_{T-1}, s_T)\) under any policy \(\pi\) as 
\begin{equation*}
	\prod^{T-1}_{k=t} \pi(a_k \mid s_k) P(s_{k+1} \mid s_k, a_k). 
\end{equation*}
Now, we can define the importance sampling ratio.

\begin{definition}
	(Importance Sampling Ratio) Given a target policy \(\pi\), a behavior policy \(\beta\) and a trajectory \(\tau = (a_t, s_{t+1}, a_{t+1}, \hdots, s_T)\) generated by \(\beta\), the importance sampling ratio is defined as 
	\begin{equation*}
		\rho_{t:T-1} \coloneqq \frac{\prod^{T-1}_{k=t} \pi(a_k \mid s_k) P(s_{k+1} \mid s_k, a_k)}{\prod^{T-1}_{k=t} \beta(a_k \mid s_k) P(s_{k+1} \mid s_k, a_k)} = \frac{\prod^{T-1}_{k=t} \pi(a_k \mid s_k) }{\prod^{T-1}_{k=t} \beta(a_k \mid s_k)}.
		\label{eq:importance_sampling}
	\end{equation*}
\end{definition}

Let \(\mathcal{T}\) be the set of possible trajectories. By multiplying returns of trajectories \(\tau \in \mathcal{T}\) generated by the behavior policy \(\beta\) with the importance sampling ratio \(\rho\) we get
\begin{align*}
	\mathbb{E}_{\beta} \bigl[\rho_{t:T-1} G_t \mid S_t = s \bigr] &= 
	\mathbb{E}_{\beta} \bigl[\rho_{t:T-1} G(\tau) \mid S_t = s \bigr] \\
	&= \sum_{\tau \in \mathcal{T}} \rho_{t:T-1} G(\tau) \prod^{T-1}_{k=t} \beta(a_k \mid s_k) \: P(s_{k+1} \mid s_k, a_k)  \\
	&= \sum_{\tau \in \mathcal{T}} \frac{\prod^{T-1}_{k=t} \pi(a_k \mid s_k) }{\prod^{T-1}_{k=t} \beta(a_k \mid s_k)} G(\tau)\prod^{T-1}_{k=t} \beta(a_k \mid s_k) \: P(s_{k+1} \mid s_k, a_k)  \\
	&= \sum_{\tau \in \mathcal{T}} G(\tau) \prod^{T-1}_{k=t} \frac{\pi(a_k \mid s_k) }{\beta(a_k \mid s_k)} \beta(a_k \mid s_k) \: P(s_{k+1} \mid s_k, a_k)  \\
	&= \sum_{\tau \in \mathcal{T}} G(\tau)\prod^{T-1}_{k=t} \pi(a_k \mid s_k) P(s_{k+1} \mid s_k, a_k) \\
	&= \mathbb{E}_{\pi} \bigl[G_t \mid S_t = s \bigr] \\
	&= V_\pi(s).
\end{align*}
The intuition behind this importance sampling correction is that, to evaluate \(\pi\), we want to weigh returns more heavily that are more likely under \(\pi\) than under \(\beta\) and vice versa. As an extension of the derivation above, we also get the per-decision importance sampling ratio \(\rho \coloneqq \frac{\pi(a \mid s)}{\beta(a \mid s)}\) \cite{Sutton1998}.

Using importance sampling, we can derive the following approximate policy gradients of the target policy \(\pi_\theta\) in an off-policy setting with behavior policy \(\beta\):
\begin{equation*}
	\nabla_\theta J(\theta) \approx \eta \: \mathbb{E}_{S \sim d^{\beta}, A \sim \beta} \biggl[ \frac{\pi_\theta(A \mid S)}{\beta(A \mid S)} \: Q_{\pi_\theta}(S,A) \: \nabla_\theta \ln \pi_\theta(A \mid S)  \biggr].
\end{equation*}
See \cite{degris2012off} for a proof. Note that \(\eta\) now is the average episode length under \(\beta\).

%% file: algorithms.tex
\section{Policy Gradient Algorithms}\label{sec:algorithms}

Building on Theorem \autoref{pg_theorem}, several policy gradient algorithms have been proposed, which compute sample-based estimates \(\hat{\nabla}_\theta J(\theta)\) of the actual policy gradients \(\nabla_\theta J (\theta)\). This is done by constructing surrogate objectives \(J_*\) such that \(\hat{\nabla}_\theta J(\theta) = \nabla_\theta J_*(\theta)\). Additionally, most algorithms focus on stabilizing learning by regularizing the policy \cite{andrychowicz2020matters} and reducing the variance of \(\hat{\nabla}_\theta J(\theta)\) \cite{sutton2000comparing}. In this section, we derive the most prominent\footnote{As determined by their impact on subsequent research and the adoption rate by users.} algorithms before than comparing them in the final subsection. 

\input{reinforce.tex}
\input{a3c.tex}

\input{trpo.tex}

\input{ppo.tex}

\input{mpo.tex}

\input{comparison.tex}

%% file: reinforce.tex
\subsection{REINFORCE}


REINFORCE (\textbf{RE}ward Increment = \textbf{N}on-negative \textbf{F}actor \(\times\) \textbf{O}ffset \textbf{R}einforcement \(\times\) \textbf{C}haracteristic \textbf{E}ligibility) \cite{williams1992simple} is the earliest policy gradient algorithm. While this algorithm precedes the formulation of the Policy Gradient Theorem, REINFORCE can be seen as a straightforward application of it. By using Monte Carlo methods \cite{Sutton1998} to estimate \(Q_\pi\) in Equation \eqref{eq:pg_proportional}, i.e. by sampling entire episodes to compute the sample returns \(G_t = \sum_{k=0}^{T} \gamma^k r_{t+k+1}\), REINFORCE samples policy gradients
\begin{equation*}
	\hat{\nabla}_\theta J(\theta) = G_t \nabla_\theta \ln \pi_\theta (a_t \mid s_t).
\end{equation*}
Using the generic policy gradient update from Equation \eqref{eq:pg_update} results in the gradient ascend updates
\[\theta_\text{new} = \theta + \alpha G_t \nabla_\theta \ln \pi_\theta (a_t \mid s_t)\]
 where \(\alpha \in (0,1]\) is the learning rate determining the step size of the gradient steps and is set as a hyperparameter. At times, REINFORCE is extended by subtracting some baseline value from \(G_t\) to reduce variance \cite{williams1992simple}. The pseudocode for REINFORCE is presented in Algorithm \ref{alg:reinforce}. 

\begin{algorithm}
	\caption{REINFORCE}\label{alg:reinforce}
	\begin{algorithmic}
	\Require $\alpha \in (0,1], \gamma \in [0,1]$
	\State Initialize $\theta$ at random 
	\ForAll{episodes}
	\State Generate trajectory $s_0, a_0, r_1, s_1 \hdots, s_T$ under policy $\pi_\theta$
	\For{$t = 1, \hdots, T$}
	\State $G_t \gets \sum_{k=t}^T \gamma^{k-t} r_k$ \Comment{estimate expected return $Q_\pi$}
	\State $\theta \gets \theta + \alpha G_t \nabla_\theta \ln \pi_\theta(a_t \mid s_t)$ \Comment{update policy parameters}
	\EndFor
	\EndFor
	\end{algorithmic}
	\end{algorithm}

%% file: a3c.tex
\subsection{A3C}


Instead of estimating \(Q_\pi\) directly via sampling as in REINFORCE, we can alternatively learn such an estimate via function approximation. Algorithms that use this approach to learn a parameterized action-value function \(\hat{Q}_{\phi}\) or value function \(\hat{V}_{\phi}\) (called critic) with parameters \(\phi\) in addition to learning the parameterized policy \(\pi_{\theta}\) (called actor) are referred to as actor-critic algorithms \cite{Sutton1998}. Note that in practice the actor and the critic may also share parameters.

The most archetypical representative of this class of algorithms is Asynchronous Advantage Actor-Critic (A3C) \cite{mnih2016asynchronous}. A3C builds on two main ideas from which the algorithm's name originates. First, as suggested by the results from Section \ref{sec:baseline}, A3C learns an estimate \(\hat{A}_{\phi}\) of the advantage function indirectly by learning an estimate \(\hat{V}_{\phi}\) of the value function. Second, A3C introduces the concept of using multiple parallel actors to interact with the environment to stabilize training. We will discuss both ideas in detail below. The algorithm samples policy gradients 
\begin{equation*}
	\hat{\nabla}_{\theta}J(\theta) = \frac{1}{\lvert \mathcal{D} \rvert} \sum_{s,a \in \mathcal{D}} \hat{A}_\phi(s,a) \nabla_{\theta} \ln \pi_{\theta} (a \mid s),
\end{equation*}
where \(\mathcal{D}\) is a batch of transitions collected by the actors. The pseudocode for A3C is presented in Algorithm \ref{alg:a3c}.

In the original work \cite{Mnih2015}, the advantage function is estimated via 
\begin{equation}
	\hat{A}_{\phi}(s_t, a_t) = \Bigl(\sum^{k-1}_{i=0} \gamma^i r_{t+i} + \gamma^k\hat{V}_{\phi}(s_{t+k})\Bigr) - \hat{V}_{\phi}(s_{t}). \label{eq:a3c_adv}
\end{equation}
To understand this estimate, observe that
\begin{align*}
	A_\pi(s_t,a_t) &= Q_\pi(s_t, a_t) - V_\pi(s_t) \\
	&= \mathbb{E}_\pi\Bigl[ R_{t+1} + \gamma V_\pi(S_{t+1}) \mid S_t = s_t, A_t = a_t \Bigl] - V_\pi(s_t) \\
	&= \mathbb{E}_\pi\Bigl[ R_{t+1} + \gamma R_{t+2} + \gamma^2 V_\pi(S_{t+2}) \mid S_t = s_t, A_t = a_t \Bigl] - V_\pi(s_t) \\
	&\;\;\vdots \\
	&= \mathbb{E}_\pi\Bigl[ \sum^{k-1}_{i=0} \gamma^i R_{t+i} + \gamma^k V_\pi(S_{t+k}) \mid S_t = s_t, A_t = a_t \Bigl] - V_\pi(s_t),
\end{align*}
for any \( k \in \mathbb{N}\), which follows from the definition of the value and action-value functions as well as their relationship. Sampling this n-step temporal difference \cite{Sutton1998} expression and replacing \(V_\pi\) with our learned \(\hat{V}_{\phi}\) yields Equation \eqref{eq:a3c_adv}. Simultaneously to updating \(\pi_{\theta}\), we learn \(\hat{V}_{\phi}\) by minimizing the mean squared error loss
\begin{equation*}
	\frac{1}{\lvert \mathcal{D} \rvert} \sum_{\mathcal{D}} \biggl( \Bigl( \sum^{k-1}_{i=0} \gamma^i r_{t+i} + \gamma^k\hat{V}_{\phi}(s_{t+k})\Bigr) - \hat{V}_{\phi}(s_{t}) \biggr)^2 \label{eq:a3c_value_loss}
\end{equation*}
over \(\phi\) via SGD. Note that the inner expression is identical to the right hand side in Equation \eqref{eq:a3c_adv}. In Equation \eqref{eq:a3c_adv}, we compute the difference between the estimated return when choosing action \(a_t\) in state \(s_t\) and the estimated return when in state \(s_t\), under policy \(\pi\) respectively. However, \(a_t\) is sampled from \(\pi\) such that in expectation this difference should be 0 for the true value function \(V_\pi\). Hence, we minimize this squared difference to optimize \(\phi\) by treating the first term, \(\sum^{k-1}_{i=0} \gamma^i r_{t+i} + \gamma^k\hat{V}_{\phi}(s_{t+k})\), as independent of \(\phi\).

The use of multiple parallel actors is justified as follows. Deep RL is notoriously unstable, which was first resolved by off-policy algorithms using replay buffers that store and reuse sampled transitions for multiple updates \cite{Mnih2015}. As an alternative, \cite{mnih2016asynchronous} propose using several actors \(\pi^{(1)}_\theta, \hdots, \pi^{(k)}_\theta\) to decrease noise by accumulating the gradients over multiple trajectories. These accumulated gradients are applied to a centrally maintained copy of \(\theta\), which is then redistributed to each actor. By doing this asynchronously, each actor has a potentially unique set of parameters at any point in time compared to the other actors. This decreases the correlation of the sampled trajectories across actors, which can further stabilize learning.

\begin{algorithm}
	\caption{A3C}\label{alg:a3c}
	\begin{algorithmic}
	\Require $n \in \mathbb{N}$, $\alpha \in (0,1], \gamma \in [0,1]$, $t_\mathrm{MAX} \in \mathbb{N}$, $T_\mathrm{MAX} \in \mathbb{N}$
	\State Initialize $\theta$ and $\phi$ at random 
	\For{$i = 1, \hdots, n$} \Comment{in parallel}
	\While{$T \leq T_\mathrm{MAX}$}
	\State $d\theta \gets 0, \ d\phi \gets 0$ \Comment{reset gradients}
	
	\State $\theta^{(i)} \gets \theta, \ \phi^{(i)} \gets \phi$ \Comment{synchronize parameters on actors}
	
	\State $s_t \sim p_0$ \Comment{sample start state}
	\State $t_\mathrm{start} \gets t$ 
	\While{$s_t$ not terminal and $t - t_\mathrm{start} \leq t_\mathrm{MAX}$}
	\State $a_t \sim \pi_{\theta^{(i)}}$ \Comment{sample action}
	\State $s_{t+1}, r_{t+1} \sim P(s_t, a_t)$ \Comment{sample next state and reward}
	\State $t \gets t + 1, \ T \gets T + 1$
	\EndWhile
	\State $ R \gets \begin{cases}
 	0 & \text{if }s_t \text{ is terminal} \\ V_{\phi^{(i)}}(s_t) & \text{else}
 \end{cases}$ \Comment{bootstrap if necessary}
	\For{$j = t - 1, \hdots, t_\mathrm{start}$}
	\State $R \gets r_j + \gamma R$
	\State $A = R - V_{\phi^{(i)}}(s_j)$
	\State $d\theta \gets d\theta + \nabla_{\theta^{(i)}} \ln \pi_{\theta^{(i)}} (a_j \mid s_j)A$ \Comment{accumulate gradients}
	\State $d\phi \gets d\phi + \nabla_{\phi^{(i)}} (R - V_{\phi^{(i)}}(s_j))^2$ \Comment{accumulate gradients}
	\EndFor
	\State update $\theta$ and $\phi$ using $d\theta$ and $d\phi$ via gradient ascent / descent
	\EndWhile
	\EndFor
	\end{algorithmic}
\end{algorithm}
	
As a final implementation detail, the policy loss function of A3C, from which the policy gradients are obtained, is typically augmented with an entropy bonus for the policy. Thus, the policy gradients become 
\begin{equation*}
	\hat{\nabla}_{\theta} J(\theta) = \frac{1}{\lvert \mathcal{D} \rvert} \sum_{\mathcal{D}} \Biggl( \biggl( \sum^{k-1}_{i=0} \gamma^i r_{t+i} + \gamma^k\hat{V}_{\phi}(s_{t+k}) - \hat{V}_{\phi}(s_{t}) \biggr) \nabla_{\theta} \ln \pi_{\theta} (a_t \mid s_t) + \beta \nabla_{\theta} H(\pi_{\theta}(\cdot \mid s_t)) \Biggr),
\end{equation*}
where \(H\) is the entropy (see Definition \ref{def:entropy}) and the entropy coefficient \(\beta\) is a hyperparameter. This entropy bonus, first proposed by \cite{williams1991function}, regularizes the policy such that it does not prematurely converges to a suboptimal policy. By rewarding entropy, the policy is encouraged to spread probability mass over actions which improves exploration \cite{mnih2016asynchronous}.

%% file: trpo.tex
\subsection{TRPO}


Excessively large changes in the policy can result in instabilities during the training of RL algorithms. Even small changes in policy parameters \(\theta\) can lead to significant changes in the resulting policy and its performance. Hence, small step sizes during gradient ascent cannot fully remedy this problem and would impair the sample efficiency of the algorithm \cite{SpinningUp2018}. Trust Region Policy Optimization (TRPO) \cite{schulman2015trust} mitigates these issues by imposing a trust region constraint on the Kullback-Leibler (KL) divergence (see Definition \ref{def:kl}) between consecutive policies. In addition, TRPO uses an off-policy correction through importance sampling as discussed in Section \ref{sec:importance} to account for the interleaved optimization and collection of transitions. 

TRPO samples gradients
\begin{equation*}
	\hat{\nabla}_{\theta}J(\theta) = \frac{1}{\lvert \mathcal{D} \rvert} \sum_{s,a \in \mathcal{D}} \hat{A}_\phi(s,a) \nabla_{\theta} \frac{\pi_\theta (a \mid s)}{\pi_\text{old} (a \mid s)} \label{eq:trpo_grad}
\end{equation*}
and postprocesses them as detailed below to solve the approximate trust region optimization problem
\begin{align*}
	\begin{split}
	\max_\theta &\quad \biggl(J_\text{TRPO}(\theta) = \mathbb{E}_{S \sim d^{\pi_\text{old}}, A \sim \pi_{\text{old}}} \biggl[ \hat{A}_\phi(S, A) \frac{ \pi_{\theta} (A \mid S)}{\pi_{\text{old}} (A \mid S)} \biggr]\biggr) \\
	\text{subject to} &\quad \mathbb{E}_{S \sim d^{\pi_\text{old}}} \bigl[ D_{KL} (\pi_{\text{old}}(\cdot \mid S) \: \Vert \: \pi_\theta( \cdot \mid S)) \bigr] \leq \delta \label{eq:trpo_objective}
	\end{split}
\end{align*}
where \(\pi_\text{old} = \pi_{\theta_\text{old}}\) is the previous policy and \(\theta_\text{old}\) the corresponding parameters. This optimization problem is an approximation to an objective with convergence guarantees, which we will show in the following. We start by presenting \cite{schulman2015trust}'s main theoretical result. Consider the objective of maximizing the expected return \(\mathbb{E}_{S_0 \sim p_0, \pi} \bigl[ G_0 \bigr]\) under policy \(\pi\), which we denote as \(\eta(\pi)\) here. Let \(L_\pi\) be the following local approximation of \(\eta\):
\[L_\pi(\tilde{\pi}) = \eta(\pi) + \int_{s \in \mathcal{S}} d^{\pi}(s) \int_{a \in \mathcal{A}} \tilde{\pi} (a \mid s) A_{\pi} (s,a) \: da \: ds,\]
with \(L_{\pi_\theta}(\pi_\theta) = \eta(\pi_\theta)\) and \(\nabla_\theta L_{\pi_{\theta_0}}(\pi_\theta)|_{\theta = \theta_0} = \nabla_\theta \eta(\pi_\theta)|_{\theta = \theta_0}\) \cite{kakade2002approximately}. Based on the total variation divergence \(D_{TV}\) (see Definition \ref{def:tv_div}), we define 
\[D^{max}_{TV}(\pi, \tilde{\pi}) \coloneqq \max_{s \in \mathcal{S}} D_{TV} (\pi(\cdot \mid s) \Vert \tilde{\pi}( \cdot \mid s)).\] Then, we have \cite{schulman2015trust}:

\begin{theorem}\label{th:trpo_theorem}
	Let \(\alpha = D^{max}_{TV}(\pi_\text{old}, \pi_\text{new})\), then
	\begin{equation}
		\eta(\pi_\text{new}) \geq L_{\pi_\text{old}}(\pi_\text{new}) - \frac{4\varepsilon\gamma}{(1 - \gamma)^2} \alpha^2, \label{eq:trpo_theorem}
	\end{equation}
	where \(\varepsilon = \max_{s \in \mathcal{S},a \in \mathcal{A}} \lvert A_\pi(s,a)\rvert\).
\end{theorem}
See the appendix in \cite{schulman2015trust} for a proof. By using the relationship between total variation divergence and KL divergence \(D_{TV}(\pi \Vert \tilde{\pi})^2 \leq D_{KL}(\pi \Vert \tilde{\pi})\) \cite{pollard2000asymptopia} and setting \(D^{max}_{KL}(\pi, \tilde{\pi}) \coloneqq \max_{s \in \mathcal{S}} D_{KL} (\pi(\cdot \mid s) \Vert \tilde{\pi}( \cdot \mid s))\) and \(C = \frac{4\varepsilon\gamma}{(1-\gamma)^2}\), we derive the following lower bound for the objective from Equation \eqref{eq:trpo_theorem}:
\begin{equation}
	\eta(\pi_\text{new}) \geq L_{\pi_\text{old}}(\pi_\text{new}) - C D^{max}_{KL}(\pi_\text{old}, \pi_\text{new}) \label{eq:trpo_theory}
\end{equation}
Iteratively maximizing the right-hand side yields a sequence of policies \(\pi_i, \pi_{i+1}, \pi_{i+2}, \hdots\) with the monotonic improvement guarantee \(\eta(\pi_i) \leq \eta(\pi_{i+1}) \leq \eta(\pi_{i+2}) \leq \hdots\). This is  because we have equality in \eqref{eq:trpo_theory} for \(\pi_\text{new} = \pi_\text{old}\) and hence 
\[\eta(\pi_{i+1}) - \eta(\pi_i) \geq \Bigl( L_{\pi_i}(\pi_{i+1}) - C D^{max}_{KL}(\pi_i, \pi_{i+1}) \Bigr) - \Bigl( L_{\pi_i}(\pi_{i}) - C D^{max}_{KL}(\pi_i, \pi_{i}) \Bigr), \]
which is non-negative as we maximize over \(\pi\) each iteration. Thus, we could construct a Minorization-Maximization-type algorithm \cite{hunter2004tutorial} which maximizes the right-hand side of Inequality \eqref{eq:trpo_theory} at each iteration and is thereby guaranteed to converge to an optimum as the objective is bounded. 

Such an algorithm would be impractical as it requires evaluating the advantage function at every point in the state-action product space \(\mathcal{S} \times \mathcal{A}\) and the KL penalty at every point in the state space \(\mathcal{S}\). Hence, \cite{schulman2015trust} apply several approximations to the objective stemming from Inequality \eqref{eq:trpo_theory}. We replace the KL penalty, which would yield restrictively small step sizes given by \(C\), by a trust region constraint:
\begin{align*}
	\begin{split}
	\max_\theta &\quad L_{\pi_\text{old}}(\pi_\theta) \\
	\text{subject to} &\quad D^\text{max}_{KL} (\pi_{\text{old}}, \pi_\theta) \leq \delta.
	\end{split}
\end{align*}
To avoid computing \(D^\text{max}_{KL}\), we use the average KL divergence
\[\bar{D}^{\pi_\text{old}}_{KL}(\pi \Vert \tilde{\pi}) \coloneqq \mathbb{E}_{S \sim d^{\pi_\text{old}}} \bigl[ D_{KL} (\pi(\cdot \mid S) \Vert \tilde{\pi}( \cdot \mid S)) \bigr]\]
between policies as heuristic constraint, which we can sample. Further, we rewrite the surrogate objective \(\max_\theta L_{\pi_\text{old}}(\pi_\theta)\) as an expectation over the old policy \(\pi_{\text{old}}\) via importance sampling. Note that \(\eta(\pi_\text{old})\) is a constant w.r.t \(\theta\):
\begin{align*}
	\argmax_\theta L_{\pi_\text{old}}(\pi_\theta) &= \argmax_\theta \biggl( \eta(\pi_\text{old}) + \int_{s \in \mathcal{S}} d^{\pi_{\text{old}}}(s) \int_{a \in \mathcal{A}} \pi_\theta (a \mid s) A_{\pi_{\text{old}}} (s,a) \: da \: ds \biggr) \\
	&= \argmax_\theta \int_{s \in \mathcal{S}} d^{\pi_{\text{old}}}(s) \int_{a \in \mathcal{A}} \pi_\theta (a \mid s) A_{\pi_{\text{old}}} (s,a) \: da \: ds \\
	&= \argmax_\theta \int_{s \in \mathcal{S}} d^{\pi_{\text{old}}}(s) \int_{a \in \mathcal{A}} \frac{\pi_{\text{old}} (a \mid s)}{\pi_{\text{old}} (a \mid s)} \pi_\theta (a \mid s) A_{\pi_{\text{old}}} (s,a) \: da \: ds \\
	&= \argmax_\theta \mathbb{E}_{S \sim d^{\pi_\text{old}}, A \sim \pi_{\text{old}}} \biggl[ \frac{\pi_{\theta} (A \mid S)}{\pi_{\text{old}} (A \mid S)} A_{\pi_{\text{old}}} (S,A) \biggr].
\end{align*}
Using these modifications, we are now left with solving the trust region problem
\begin{align}
	\begin{split}
	\max_\theta &\quad \mathbb{E}_{S \sim d^{\pi_\text{old}}, A \sim \pi_{\text{old}}} \biggl[ \frac{\pi_{\theta} (A \mid S)}{\pi_{\text{old}} (A \mid S)} A_{\pi_{\text{old}}} (S,A) \biggr] \\
	\text{subject to} &\quad \mathbb{E}_{S \sim d^{\pi_\text{old}}} \Bigl[ D_{KL} (\pi_{\text{old}}(\cdot \mid S) \Vert \pi_\theta( \cdot \mid S)) \Bigr] \leq \delta. \label{eq:trpo_approx_problem}
	\end{split}
\end{align}
To approximately solve this constrained problem, \cite{schulman2015trust} use backtracking line search, where the search direction is computed by Taylor-expanding (see Theorem \ref{th:taylor}) the objective function and the constraint. Let \(g = \nabla_\theta \mathbb{E}_{S \sim d^{\pi_\text{old}}, A \sim \pi_{\text{old}}} \Bigl[ \frac{\pi_{\theta} (A \mid S)}{\pi_{\text{old}} (A \mid S)} A_{\pi_{\text{old}}} (S,A) \Bigr] \). Approximating \(L_{\pi_\text{old}}(\pi_\theta)\) to first order around \(\theta_\text{old}\) yields 
\[ L_{\pi_\text{old}}(\pi_\theta) \approx g^\mathsf{T}(\theta - \theta_\text{old}), \]
where we again ignored the constant \(\eta(\pi_{\text{old}})\). The quadratic approximation of the constraint at \(\theta_\text{old}\) is 
\[\bar{D}^{\pi_\text{old}}_{KL}(\pi \Vert \tilde{\pi})  \approx \frac{1}{2} (\theta - \theta_\text{old})^\mathsf{T}H(\theta - \theta_\text{old}),\]
where \(H\) is the Fisher information matrix, which is estimated via 
\[\hat{H}_{i,j} = \frac{1}{\lvert \mathcal{D} \rvert} \sum_{s \in \mathcal{D}} \frac{\partial^2}{\partial \theta_i \partial \theta_j} D_{KL} (\pi_\text{old} (\cdot \mid s) \Vert \pi (\cdot \Vert s)), \]
albeit the full matrix is not required.
We solve the resulting approximate optimization problem analytically using Lagrangian duality methods \cite{boyd2004convex} leading to 
\[\theta_\text{new} = \theta_\text{old} + \sqrt{\frac{2\delta}{g^\mathsf{T}\hat{H}^{-1}g}}\hat{H}^{-1}g. \]
However, due to the Taylor approximations, this solution may not satisfy the original trust region constraint or may not improve the surrogate objective of Problem \eqref{eq:trpo_approx_problem}. Therefore, TRPO employs backtracking line search along the search direction \(H^{-1}g\) with search parameter \(\beta \in (0,1)\):
\[\theta_\text{new} = \theta_\text{old} + \beta^m \sqrt{\frac{2\delta}{g^\mathsf{T}H^{-1}g}}H^{-1}g. \]
We choose the exponent \(m\) as the smallest non-negative integer such that the trust region constraint is satisfied and the surrogate objective improves. We circumvent the computationally expensive matrix inversion of \(H\) for the search direction \(d \approx H^{-1}g\) by computing \(d\) via the conjugate gradient algorithm \cite{hestenes1952methods}. To further reduce the computational costs, the Fisher-vector products in this process can also be only calculated on a subset of the dataset \(\mathcal{D}\) of sampled transitions. 

\begin{algorithm}
	\caption{TRPO}\label{alg:trpo}
	\begin{algorithmic}
	\Require $\delta \in \mathbb{R}$, $b \in (0,1)$, $K \in \mathbb{N}$, $\alpha \in (0,1]$, $U \in \mathbb{N}$, $T \in \mathbb{N}$
	\State Initialize $\theta$ and $\phi$ at random 
	\State $t \gets 0$
	\While{$t \leq T$}
	\For{$i = 1, \hdots, U$}
	\State $a \sim \pi_\theta$ \Comment{sample action}
	\State $\beta(a \mid s) \gets \pi_\theta(a \mid s)$
	\State $s, r \sim P(s, a)$ \Comment{sample next state and reward}
	\State $t \gets t + 1$
	\State Store $(a, s, r, \beta(a \mid s)) $ in $\mathcal{D}$
	\EndFor
	\ForAll{epochs}
	\State Compute returns $R$ and advantages $A$
	\State $g \gets \frac{1}{\lvert \mathcal{D} \rvert} \sum_{\mathcal{D}} \nabla_{\theta} \frac{\pi_\theta (a \mid s)}{\beta (a \mid s)}  A$
	\State Compute $\hat{H}$ as the Hessian of the sample average KL-divergence
	\State Compute $d \approx \hat{H}^{-1}g$ via conjugate gradient algorithm
	\State $m \gets 0$
	\Repeat
	\State $\theta \gets \theta_\text{old} + b^m \sqrt{\frac{2\delta}{d^\mathsf{T}\hat{H} d}} d$ 
	\State $m \gets m + 1$ 
	\Until (sample loss improves and KL constraint satisfied) or $m > K$
	\State $d\phi \gets \frac{1}{\lvert \mathcal{D} \rvert} \sum_{\mathcal{D}} \nabla_{\phi} (R - V_{\phi}(s))^2$
	\State Update $\phi$ using $d\phi$ via gradient descent
	\EndFor
	\EndWhile
\end{algorithmic}
\end{algorithm}

\cite{schulman2015trust} do not specify an advantage estimator to be used in TRPO. The algorithm is commonly used with either the estimator used by A3C or the one which we present in the next subsection. TRPO is typically used with multiple parallel actors as A3C. The pseudocode for TRPO is presented in Algorithm \ref{alg:trpo}. We remark that while being a policy-based algorithm, TRPO does not strictly adhere to Definition \ref{def:pg} as it solves a constrained optimization problem via line search. Yet, it does compute gradients of its objective function w.r.t. the policy parameters and therefore we treat it as a policy gradient algorithm.

%% file: ppo.tex
\subsection{PPO}


Given the complexity of TRPO, Proximal Policy Optimization (PPO) \cite{schulman2017proximal} is designed to enforce comparable constraints on the divergence between consecutive policies during the learning process while simplifying the algorithm to not require second-order methods. This is achieved by  heuristically flattening the gradients outside of an approximate trust region around the old policy. In addition, PPO uses a novel method to learn an estimate of the advantage function.

Let \(r_{\theta}(a \mid s) = \frac{ \pi_{\theta} (a \mid s)}{\pi_{\text{old}} (a \mid s)}\). Then, PPO uses the following estimate of the policy gradients:
\begin{equation}
	\hat{\nabla}_{\theta}J(\theta) = \frac{1}{\lvert \mathcal{D} \rvert} \sum_{s,a \in \mathcal{D}} \hat{A}_{\phi}(s, a) \nabla_{\theta} \min \biggl\{ r_{\theta}(a \mid s), \text{clip}\Bigl( r_{\theta}(a \mid s), 1-\varepsilon, 1+\varepsilon \Bigr)  \biggr\}. \label{eq:ppo_grad}
\end{equation}
Here, the clip-function \(\text{clip} \colon \mathbb{R} \times \mathbb{R} \times \mathbb{R} \rightarrow \mathbb{R}\) is defined by
\[\text{clip}(x, a, b) = \begin{cases}
	a & \text{if } x < a, \\
	x & \text{if } a \leq x \leq b, \\
	b & \text{if } b < x.
\end{cases}\]
and is applied element-wise to \(r_{\theta}\). \(\varepsilon\) is a hyperparameter.

This clipped objective conservatively removes the incentive for moving the new policy to far away from the old one. Intuitively, this can be seen as follows. We distinguish two cases: positive and negative estimated advantages \(\hat{A}(s, a)\), i.e. whether action \(a\) is good or bad. If \(\hat{A}(s, a) > 0\), the surrogate objective \(J_\text{PPO}(\theta)\) increases when \(a\) becomes more likely. Similarly, if \(\hat{A}(s, a) < 0\), \(J_\text{PPO}(\theta)\) increases when \(a\) becomes less likely. Hence, we want to adjust the policy parameters \(\theta\) accordingly. However, by clipping the policy ratio \(r_{\theta}\), this positive effect on the objective function disappears once we move outside the clip range. This clipping process is conservative as we only clip if the objective function would improve. If the policy is changed in the opposite direction such that \(J_\text{PPO}(\theta)\) decreases, \(r_{\theta}\) is not clipped due to taking the minimum in Equation \eqref{eq:ppo_grad}. \autoref{fig:ppo_expl} illustrates this explanation. The pseudocode for PPO is presented in Algorithm \ref{alg:ppo}.

\begin{figure}
     \centering
     \begin{subfigure}[b]{0.4\textwidth}
         \centering
         \begin{tikzpicture}[scale=0.85]
			\begin{axis}[
    			axis lines = left,
    			xlabel = \(r_\theta\),
    			ylabel = {\(J_\text{PPO}(\theta)\)},
    			ymax = 2,
    			xtick={0,1, 1.3},
    			xticklabels={\(0\), \(1\), \(1 + \varepsilon\),},
    			ymajorticks=false,
			]
			\addplot [
    			domain=0:2, 
    			samples=100, 
    			color=darkgray,
				]
				{min(x, 1.3};
			\draw [dashed] (1.3,-2) -- (1.3,2);
			\end{axis}
		\end{tikzpicture}
         \caption{$A > 0$}
         \label{fig:y equals x}
     \end{subfigure}
     \hfill
     \begin{subfigure}[b]{0.4\textwidth}
         \centering
         \begin{tikzpicture}[scale=0.85]
			\begin{axis}[
    			axis lines = left,
    			xlabel = \(r_\theta\),
    			ylabel = {\(J_\text{PPO}(\theta)\)},
    			ymax = 0,
    			xtick={0,0.7, 1},
    			xticklabels={\(0\), \(1 + \varepsilon\), \(1\),},
    			ymajorticks=false,
			]
			\addplot [
    			domain=0:2, 
    			samples=100, 
    			color=darkgray,
				]
				{min(-x, -0.7};
			\draw [dashed] (0.7,-2) -- (0.7,2);
			\end{axis}
		\end{tikzpicture}
         \caption{$A < 0$}
         \label{fig:three sin x}
     \end{subfigure}
        \caption{Illustration of the conservative clipping of PPO's objective function, which is shown as a function of the ratio \(r_\theta\) for a single transition depending on whether the advantages are positive (a) or negative (b). Replicated from \cite{schulman2017proximal}.}
        \label{fig:ppo_expl}
\end{figure}
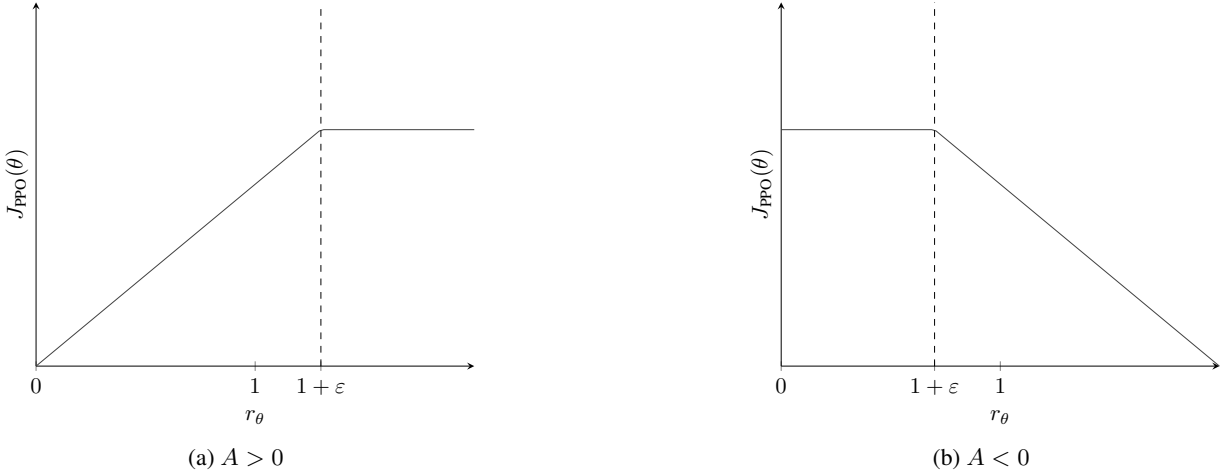

\begin{algorithm}
	\caption{PPO}\label{alg:ppo}
	\begin{algorithmic}
	\Require $\varepsilon \in \mathbb{R}$, $\alpha \in (0,1], \gamma \in [0,1]$, $\lambda \in [0,1]$, $U \in \mathbb{N}$, $T \in \mathbb{N}$
	\State Initialize $\theta$ and $\phi$ at random 
	\State $t \gets 0$
	\While{$t \leq T$}
	\For{$i = 1, \hdots, U$}
	\State $a \sim \pi_\theta$ \Comment{sample action}
	\State $\beta(a \mid s) \gets \pi_\theta(a \mid s)$
	\State $s, r \sim P(s, a)$ \Comment{sample next state and reward}
	\State $t \gets t + 1$
	\State Store $(a, s, r, \beta(a \mid s)) $ in $\mathcal{D}$
	\EndFor
	\ForAll{epochs}
	\State $R, A \gets \text{computeGAE}(v, r, \lambda, \gamma)$ \Comment{Compute returns and advantages}
	\State $d\theta \gets \nabla_{\theta} \frac{1}{\lvert \mathcal{D} \rvert} \sum_\mathcal{D} \min \bigl( \frac{\pi(a \mid s)}{\beta(a \mid s)} , \text{clip}(\frac{\pi(a \mid s)}{\beta(a \mid s)}, 1-\varepsilon, 1+\varepsilon) \bigr)A $
	\State $d\phi \gets \nabla_{\phi} \frac{1}{\lvert \mathcal{D} \rvert} \sum_\mathcal{D} (R - V_{\phi}(s))^2 $ 
	\State update $\theta$ and $\phi$ using $d\theta$ and $d\phi$ via gradient ascent / descent 
	\EndFor
	\EndWhile
\end{algorithmic}
\end{algorithm}

To compute the estimate \(\hat{A}_\phi\) of the advantage function, PPO uses generalized advantage estimation (GAE) \cite{schulman2015high} to further reduce the variance of gradients. GAE computes the estimated advantage as 
\begin{equation}
	\hat{A}_\phi(s_t, a_t) = \sum^{T-1}_{i=t} (\gamma \lambda)^{i-t} \delta_i, \label{eq:ppo_gae}
\end{equation}
where \(\delta_i = r_i + \gamma \hat{V}_\phi(s_{i+1}) - \hat{V}_\phi(s_i)\).
The value function estimate \(\hat{V}_\phi\) is learned by minimizing 
\begin{equation*}
	\frac{1}{\lvert \mathcal{D} \rvert} \sum_{\mathcal{D}} \Bigl( \bigl(\hat{A}_\phi(s, a) + \hat{V}_{\phi}(s)\bigr) - \hat{V}_{\phi}(s) \Bigr)^2,
\end{equation*}
where the first term is treated as independent of \(\phi\). GAE relates to the idea of eligibility traces \cite{sutton1981toward} to use both the sampled rewards and the current value function estimate on every time step. By computing such an exponentially weighted estimator, GAE reduces the variance of the policy gradients at the cost of introducing a slight bias to the value function estimate \cite{schulman2015high}. The hyperparameters \(\gamma\) and \(\lambda\) both adjust this bias-variance tradeoff. \(\gamma\) does so by scaling the value function estimate \(\hat{V}\) whereas \(\lambda\) controls the dependence on delayed rewards. Note that GAE is a strict generalization of A3C's advantage estimate as Equation \eqref{eq:ppo_gae} reduces to Equation \eqref{eq:a3c_adv} when \(\lambda = 1\). The pseudocode for GAE is presented in Algorithm \ref{alg:gae}.

\begin{algorithm}
	\caption{GAE}\label{alg:gae}
	\begin{algorithmic}
	\Require $\gamma \in [0,1]$, $\lambda \in [0,1]$
	\Require rewards $(r_k)^{t+n}_{k=t}$, values $(v_k)^{t+n+1}_{t=k}$
	\State $A_t, \hdots, A_{t+n} \gets 0$
	\State $x \gets 0$
	\For{$i = t+n, \hdots, t$} 
	\If{transition was terminal}
	\State $\omega \gets 1$
	\Else
	\State $\omega \gets 0$
	\EndIf
	\State $\delta \gets r_i + \gamma \cdot v_{i+1} \cdot (1-\omega) - v_i$
	\State $x \gets \delta + \gamma \cdot \lambda \cdot (1-\omega) \cdot x$
	\State $A_i \gets x$
	\EndFor
	\For{$i = t, \hdots, t+n$} 
	\State $R_i \gets A_i + v_i $
	\EndFor
\end{algorithmic}
\end{algorithm}

Beyond these main innovations, PPO uses several implementational details to improve learning. PPO conducts multiple update epochs for each batch of data such that several gradient descent steps are based on the same transitions to increase sample efficiency and speed up learning. Moreover, PPO commonly augments its surrogate objective with an entropy bonus \(H(\pi_{\theta}(\cdot \mid s))\) and uses multiple actors similarly to A3C. Lastly, we note that further algorithms have been proposed as modifications of PPO, e.g. Phasic Policy Gradients \cite{cobbe2021phasic} and Robust Policy Optimization \cite{rahman2022robust}, which we will not discuss further as they only modify minor details.

%% file: mpo.tex
\subsection{V-MPO}



In the previous algorithms, we learn a policy from the control perspective by selecting actions to maximize expected rewards. In this subsection, we consider an alternative formulation of RL problems, which casts them as probabilistic inference problems of estimating posterior policies that are consistent with a desired outcome \cite{abdolmaleki2018relative}. This problem is then solved via Expectation Maximization (EM) \cite{dempster1977maximum}. This procedure was first proposed in the off-policy algorithm Maximum a-posteriori Policy Optimization (MPO) \cite{abdolmaleki2018maximum, abdolmaleki2018relative}. Here, we discuss its on-policy variant V-MPO \cite{song2019v}, where the "V" in the name refers to learning the value function \(V_\pi\) instead of \(Q_\pi\) as in MPO. 

The main idea of V-MPO is to find a maximum a posteriori estimate of the policy by sequentially finding a tight lower bound on the posterior and then maximizing this lower bound. This problem can be transformed into the objective function
\begin{equation*}
	J_\text{V-MPO}(\theta, \eta, \nu) = \mathcal{L}_\pi(\theta) + \mathcal{L}_\eta(\eta) + \mathcal{L}_\nu(\theta, \nu),
\end{equation*}
where \(\mathcal{L}_\pi\) is the policy loss
\begin{equation}
	\mathcal{L}_\pi(\theta) = - \sum_{a,s \in \tilde{\mathcal{D}}} \frac{\exp\Bigl( \frac{\hat{A}_{\phi}(s,a)}{\eta} \Bigr)}{\sum_{a',s' \in \tilde{\mathcal{D}}} \exp\Bigl( \frac{\hat{A}_{\phi}(s',a')}{\eta} \Bigr)} \ln \pi_{\theta}(a \mid s), \label{eq:mpo_policy_loss} 
\end{equation}
\(\mathcal{L}_\eta\) is the temperature loss
\begin{equation}
	\mathcal{L}_\eta(\eta) = \eta \varepsilon_\eta + \eta \ln \Biggl[ \frac{1}{\lvert \tilde{\mathcal{D}} \rvert} \sum_{a,s \in \tilde{\mathcal{D}}} \exp\biggl( \frac{\hat{A}_{\phi}(s,a)}{\eta} \biggr) \Biggr] \label{eq:mpo_temp_loss}
\end{equation}
and \(\mathcal{L}_\nu\) is the trust-region loss
\begin{align}
\begin{split}
	\mathcal{L}_\nu(\theta, \nu) &= \frac{1}{\lvert \mathcal{D} \rvert} \sum_{s \in \mathcal{D}} \biggl( \nu \biggl( \varepsilon_\nu - \mathrm{sg}\Bigl[\Bigl[ D_{KL}(\pi_{\text{old}}(\cdot \mid s) \: \Vert \: \pi_{\theta}(\cdot \mid s))  \Bigr]\Bigr] \biggr)  + \mathrm{sg}\bigl[\bigl[ \nu \bigr]\bigr] D_{KL}\bigl( \pi_{\text{old}}(\cdot \mid s) \: \Vert \: \pi_{\theta}(\cdot \mid s) \bigr) \biggr). \label{eq:mpo_kl_loss}
\end{split}
\end{align}
Here, \(\mathrm{sg}[[\cdot]]\) is a stop-gradient operator, meaning its arguments are treated as constants when computing gradients, \(\eta\) is a learnable temperature parameter, \(\nu\) is a learnable KL-penalty parameter, \(\varepsilon_\nu\) and \(\varepsilon_\eta\) are hyperparameters, \(\mathcal{D}\) is a batch of transitions and \(\tilde{\mathcal{D}} \subset \mathcal{D}\) is the half of these transitions with the largest advantages. We will provide a sketch of how to derive this objective function in the following. We refer the interested reader to Appendix \autoref{sec:mpo_details} for a more detailed derivation. 

Let \(p_\theta(s,a) = \pi_\theta(a \mid s) d^{\pi_\theta}(s)\) denote the joint state-action distribution under policy \(\pi_\theta\) conditional on the parameters \(\theta\). Let \(\mathcal{I}\) be a binary random variable whether the updated policy \(\pi_\theta\) is an improvement over the old policy \(\pi_{\text{old}}\), i.e. \(\mathcal{I}=1\) if it is an improvement. We assume the conditional probability of \(\pi_\theta\) being an improvement given a state \(s\) and an action \(a\) is proportional to the following expression 
\begin{equation}
	p_\theta(\mathcal{I} = 1 \mid s,a) \propto \exp\Bigl( \frac{A_{\pi_\text{old}}(s,a)}{\eta} \Bigr). \label{eq:mpo_0}
\end{equation}
Given the desired outcome \(\mathcal{I}=1\), we seek the posterior distribution conditioned on this event. Specifically, we seek the maximum a posteriori estimate 
\begin{align}
\begin{split}
	\theta^* &= \argmax_\theta \bigl[ p_\theta(\mathcal{I}=1) \rho(\theta) \bigr] \\
	&= \argmax_\theta \bigl[ \ln p_\theta(\mathcal{I}=1) + \ln \rho(\theta) \bigr], \label{eq:mpo_2}
\end{split}
\end{align}  
where \(\rho\) is some prior distribution. \(\ln p_\theta(\mathcal{I}=1)\) can be rewritten as
\begin{equation}
	\ln p_\theta(\mathcal{I}=1) = \mathbb{E}_{S,A \sim \psi} \biggl[ \ln \frac{p_\theta(\mathcal{I}=1,S,A)}{\psi(S,A)} \biggr] + D_{KL} \bigl(\psi \:\Vert \: p_\theta(\cdot,\cdot \mid \mathcal{I}=1)\bigr) \label{eq:mpo_1}
\end{equation}
for some distribution \(\psi\) over \(\mathcal{S} \times \mathcal{A}\). Observe that, since the KL divergence is non-negative, the first term is a lower bound for \(\ln p_\theta(\mathcal{I}=1)\). Akin to EM algorithms, V-MPO now iterates by choosing the variational distribution \(\psi\) in the expectation (E) step to minimize the KL divergence in Equation \eqref{eq:mpo_1} to make the lower bound as tight as possible. In the maximization (M) step, we maximize this lower bound and the prior \(\ln \rho(\theta)\) to obtain a new estimate of \(\theta^*\) via Equation \eqref{eq:mpo_2}.

First, we consider the E-step. Under the proportionality assumption \eqref{eq:mpo_0}, we turn the problem of finding a variational distribution \(\psi\) to minimize \(D_{KL} (\psi \: \Vert \: p_{\theta_\text{old}}(\cdot,\cdot \mid \mathcal{I}=1))\) into an optimization problem over the temperature \(\eta\). This is formulated as a constrained problem subject to a bound on the KL divergence between \(\psi\) and the previous state-action distribution \(p_{\theta_\text{old}}\) while ensuring that \(\psi\) is a state-action distribution. To enable optimizing \(\eta\) via gradient descent, we transform this constrained problem into an unconstrained problem via Lagrangian relaxation, which emits both the form of the variational distribution
\begin{equation*}
	\psi(s,a) = \frac{p_{\theta_\text{old}}(s,a) \: p_{\theta_\text{old}}(\mathcal{I}=1 \mid s,a)}{\int_{s \in \mathcal{S}} \int_{a \in \mathcal{A}} p_{\theta_\text{old}}(s,a) \: p_{\theta_\text{old}}(\mathcal{I}=1 \mid s,a) \: da \: ds }
\end{equation*}
and the temperature loss \eqref{eq:mpo_temp_loss}
\begin{equation*}
	\mathcal{L}_\eta(\eta) = \eta \varepsilon_\eta + \eta \ln \biggl( \int_{s \in \mathcal{S}} \int_{a \in \mathcal{A}} \exp\Bigl( \frac{A_{\pi_\text{old}}(s,a)}{\eta} \Bigr) \: da \: ds \biggr).
\end{equation*} 
\cite{song2019v} find that using only the highest 50 \% of advantages per batch when sampling these expressions, i.e. replacing \(\mathcal{D}\) with \(\tilde{\mathcal{D}}\), substantially improves the algorithm. The advantage function \(A_{\pi}\) is estimated by \(\hat{A}_\phi\), which is learned as in A3C.

Then, in the M-Step we solve the maximum a posterior estimation problem \eqref{eq:mpo_2} over the policy parameters \(\theta\) for the constructed variational distribution \(\psi(s,a)\) and the thereby implied lower bound.
This lower bound, i.e. the first term in Equation \eqref{eq:mpo_1}, becomes the weighted maximum likelihood policy loss \eqref{eq:mpo_policy_loss}
\begin{equation*}
	\mathcal{L}_\pi(\theta) = - \int_{s \in \mathcal{S}} \int_{a \in \mathcal{A}} \psi(s,a) 
	\ln \pi_\theta(a \mid s) \: da \: ds
\end{equation*}
after dropping terms independent of \(\theta\). This loss is computed over the same reduced batch \(\tilde{\mathcal{D}}\) as the temperature loss, effectively assigning out-of-sample transitions a weight of zero. Simultaneously, we want to maximize the prior \(\rho(\theta)\) according to the maximization problem \eqref{eq:mpo_2}. V-MPO follows TRPO and PPO to choose a prior such that the new policy is kept close to the previous one, i.e.
\[\rho(\theta) = -\nu \mathbb{E}_{S \sim d^{\pi_\text{old}}} \bigl[D_{KL}\bigl( \pi_{\text{old}} (\cdot \mid S) \Vert \pi_\theta (\cdot \mid S) \bigr) \bigr], \]
with learnable parameter \(\nu\). However, optimizing the resulting sample-based maximum likelihood objective directly tends to result in overfitting. Hence, a sequence of transformations is applied. First, the prior is transformed into a hard constraint on the KL divergence when optimizing the policy loss. To employ gradient-based optimization, we use Lagrangian relaxation to transform this constrained optimization problem back into an unconstrained problem and use a coordinate-descent strategy to simultaneously optimize for \(\theta\) and \(\nu\). This can equivalently be written via the stop-gradient operator yielding the trust-region loss \eqref{eq:mpo_kl_loss}. 

\begin{algorithm}
	\caption{V-MPO}\label{alg:vmpo}
	\begin{algorithmic}
	\Require $\eta \in \mathbb{R}$, $\nu \in \mathbb{R}$, $\varepsilon_\eta \in \mathbb{R}$, $\varepsilon_\nu \in \mathbb{R}$, $U \in \mathbb{N}$, $T \in \mathbb{N}$
	\State Initialize $\theta$ and $\phi$ at random 
	\State $t \gets 0$
	\While{$t \leq T$}
	\For{$i = 1, \hdots, U$}
	\State $a \sim \pi_\theta$ \Comment{sample action}
	\State $\beta(a \mid s) \gets \pi_\theta(a \mid s)$
	\State $s, r \sim P(s, a)$ \Comment{sample next state and reward}
	\State $t \gets t + 1$
	\State Store $(a, s, r, \beta(a \mid s)) $ in $\mathcal{D}$
	\EndFor
	\ForAll{epochs}
	\State Compute returns $R$ and advantages $A$
	\State Compute $\tilde{\mathcal{D}}$
	\State $L_\nu \gets \frac{1}{\lvert \mathcal{D} \rvert} \sum_\mathcal{D} \nu (\varepsilon_\nu - \mathrm{sg}(D_{KL}(\pi_\text{old} \Vert \pi_\theta))) + \mathrm{sg}(\nu)D_{KL}(\pi_\text{old} \Vert \pi_\theta)  $ \Comment{KL loss}
	\State $L_\pi \gets - \frac{1}{\lvert \tilde{\mathcal{D}} \rvert} \sum_{\tilde{\mathcal{D}}} \ln \pi_\theta(a \mid s) \psi(s,a)$ \Comment{Policy loss}
	\State $L_\eta \gets \eta \varepsilon_\eta + \eta \ln (\frac{1}{\lvert \tilde{\mathcal{D}} \rvert} \sum_{\tilde{\mathcal{D}}} \exp \frac{A}{\eta})$ \Comment{Temperature loss}
 	
	\State $d\theta \gets \nabla_{\theta} (L_\pi + L_\nu) $, $d\eta \gets \frac{\partial}{\partial \eta} L_\eta$, $d\nu \gets \frac{\partial}{\partial \nu} L_\nu$ \Comment{Compute gradients}
	\State $d\phi \gets \frac{1}{\lvert \mathcal{D} \rvert} \sum_{\mathcal{D}} \nabla_{\phi} (R - V_{\phi}(s))^2$
	\State update $\theta$, $\eta$, $\nu$ and $\phi$ using $d\theta$, $d\eta$, $d\nu$ and $d\phi$ via gradient ascent / descent 
	\EndFor
	\EndWhile
\end{algorithmic}
\end{algorithm}

The learnable parameters \(\eta\) and \(\nu\) are Lagrangian multipliers and hence must be positive. We enforce this by projecting the computed values to small positive values \(\eta_\text{min}\) and \(\nu_\text{min}\) respectively if necessary. The pseudocode for V-MPO is depicted in Algorithm \ref{alg:vmpo}. As implementational details, V-MPO typically uses decoupled KL constraints for the mean and covariance of the policy in continuous action spaces following \cite{abdolmaleki2018maximum}. This enables better exploration without moving the policy mean as well as fast learning by rapidly changing the mean without resulting in a collapse of the policy due to vanishing standard deviations. In addition, V-MPO can be used with an off-policy correction via an importance sampling ratio similarly to TRPO and PPO and uses multiple actors following A3C.

%% file: comparison.tex
\subsection{Comparing Design Choices in Policy Gradient Algorithms}\label{sec:comparison}

Having outlined the main on-policy policy gradient algorithms, we want to shortly compare them to characterize the main design choices in constructing such algorithms. 

The predominant differences across policy gradient algorithms lie in the estimators \(\hat{\nabla}_\theta J(\theta)\) of the policy gradients. We summarize these estimates in \autoref{table:gradient_estimators}\footnote{For V-MPO, we focus on the policy loss, thus ignoring the gradient of the KL loss \(\mathcal{L}_\nu\) w.r.t. the policy parameters \(\theta\) here.}. The algorithms can be distinguished along several dimensions with respects to the gradients. First, they use different variance reduction techniques, which are especially reflected in how \(Q_\pi\) in the policy gradient formula \eqref{eq:pg_proportional} is estimated. Second, various policy regularization strategies are used. Third, the algorithms employ further lower-level details to stabilize learning. We will discuss each of these dimensions in the following.

\begin{table}
\begin{center}
\begin{tabular}{c | c} 
\toprule
\\[-1em]
 \tablecentered{Algorithm} & \tablecentered{Gradient estimator} \\ 
 \\[-1em] 
 \toprule
 \\[-1em]
 REINFORCE  & \(G \nabla_\theta \ln \pi_\theta (a \mid s)\) \\ 
 \\[-1em]
 \midrule
 \\[-1em]
 A3C & \( \frac{1}{\lvert \mathcal{D} \rvert} \sum_{\mathcal{D}} \hat{A} \nabla_{\theta} \ln \pi_{\theta} (a \mid s)\) \\
 \\[-1em]
 \midrule
 \\[-1em]
 TRPO  & \( \frac{1}{\lvert \mathcal{D} \rvert} \sum_{\mathcal{D}} \hat{A} \nabla_{\theta} \frac{\pi_\theta (a \mid s)}{\pi_\text{old} (a \mid s)}\) \\
 \\[-1em]
 \midrule
 \\[-1em]
 PPO  & \(\frac{1}{\lvert \mathcal{D} \rvert} \sum_{\mathcal{D}} \hat{A} \nabla_{\theta} \min \biggl( \frac{\pi_\theta (a \mid s)}{\pi_\text{old} (a \mid s)}, \text{clip}\Bigl( \frac{\pi_\theta (a \mid s)}{\pi_\text{old} (a \mid s)}, 1-\varepsilon, 1+\varepsilon \Bigr) \biggr)\) \\
 \\[-1em]
 \midrule
 \\[-1em]
 V-MPO & \(\frac{1}{\sum_{\tilde{\mathcal{D}}} \exp\bigl( \frac{\hat{A}}{\eta} \bigr)} \sum_{\tilde{\mathcal{D}}} \exp\bigl( \frac{\hat{A}}{\eta} \bigr) \nabla_\theta \ln \pi_{\theta}(a \mid s)\) \\
 \bottomrule
\end{tabular}
\caption{Policy gradient estimates used by various policy gradient algorithms.}\label{table:gradient_estimators}
\end{center}
\end{table}

Reducing variance is important to stabilize learning and speed up convergence \cite{sutton2000comparing}. However, while high variance means algorithms require more samples to converge, bias in the estimates is not resolvable even with infinite samples \cite{schulman2015high}. All contemporary policy gradient algorithms, i.e. all presented algorithms except REINFORCE, make use of baselines to reduce variance as discussed in Section \ref{sec:baseline} when approximating the unknown \(Q_\pi\). While REINFORCE samples returns as an unbiased but high-variance estimate \cite{Sutton1998}, the other algorithms learn a value function \(\hat{V}\) to estimate the advantage function \(\hat{A}\). Notably, this reduces variance at the cost of introducing bias \cite{konda2003onactor, sutton1999policy, schulman2015high}. Further differences arise from how advantages are estimated, albeit these strategies can be easily transferred between algorithms. PPO uses GAE to estimate advantages, which generalizes the n-step temporal difference estimates used in A3C, TRPO and V-MPO. In addition, V-MPO scales the advantages via the learned temperature \(\eta\) and only uses the top 50 \% of advantages per batch.

To stabilize learning beyond variance reduction, several regularization techniques are proposed by TRPO, PPO and V-MPO to limit the change in policies across iterations. TRPO imposes a constraint on the KL divergence between the newly learned and the previous policy. Thus, the policy gradients are not directly applied to update the policy parameters \(\theta\) but instead they are postprocessed to yield an approximate solution to this constrained optimization problem. This comes at the cost of algorithmic complexity however. Whereas the other algorithms directly compute the estimated policy gradients via automatic differentiation \cite{wengert1964simple, margossian2019review}, TRPO requires the estimation of a hessian and the application of the conjugate gradient algorithm followed by a line search. PPO avoids such complexity by introducing a heuristic, which bounds the probability ratio \(\frac{\pi_\theta (a \mid s)}{\pi_\text{old} (a \mid s)}\), into its objective function. By conservatively clipping this ratio, large policy changes induced by overfitting the advantage function are prevented. This also enables PPO to conduct multiple updates on the same data to accelerate learning \cite{schulman2017proximal}. V-MPO too limits the KL divergence across policies. The prior distribution is selected such that V-MPO arrives at a similar optimization problem with a penalty on the KL divergence as TRPO. Following TRPO, V-MPO transforms this into a constrained optimization problem, albeit now with the goal of automatically tuning the penalty parameter by applying coordinate-descent to the Lagrangian relaxation of this constrained optimization problem.

Lastly, we point out that different lower-level details are employed by the discussed algorithm. Except for REINFORCE, all algorithms use several actors, which are potentially updated asynchronously, and average gradients over batches of transitions to further reduce their variance. Here, V-MPO slightly diverges from the other algorithms as it computes a weighted average based on the advantages of the transitions, i.e. with weights \[\frac{\exp\Bigl( \frac{A_{\pi_{\theta_\text{old}}}(s,a)}{\eta} \Bigr)}{\sum_{a,s \in \tilde{\mathcal{D}}} \exp\Bigl( \frac{A_{\pi_{\theta_\text{old}}}(s,a)}{\eta} \Bigr)}.\]
A3C and PPO commonly use an entropy bonus to prevent premature convergence to a suboptimal policy by incentivizing a higher standard deviation of the Gaussian output by the policy. We observe that V-MPO does not use an entropy bonus but achieves a comparable effect in continuous action spaces by constraining the policy mean and standard deviation separately. Lastly, TRPO and PPO include an importance sampling correction to compensate for the slight off-policy nature of the algorithms induced by using multiple asynchronous workers. This is also mentioned as an option for V-MPO \cite{song2019v}.

%% file: convergence.tex
\section{Convergence Results}\label{sec:convergence}


In this section, we discuss convergence results for policy gradient algorithms from literature. First, we present an overview of different convergence proofs in Section \ref{sec:convergence_lit}. Then, we thoroughly present one selected result in Section \ref{sec:mirror_learning}. 

\subsection{Literature Overview}\label{sec:convergence_lit}


Several convergence results have been proposed for policy gradient algorithms. They differ along various dimensions: the specific algorithms covered, the shown strength of convergence, the problem settings and the employed proving techniques. The following overview of convergence results is not intended to be complete but rather shall showcase these differences. 

As previously discussed, REINFORCE uses an unbiased estimator of the policy gradients, which in expectation therefore point in the direction of the true gradients. Hence, under common stochastic approximation assumptions towards the step sizes, REINFORCE can be shown to converge to a locally optimal policy \cite{Sutton1998}. In the general form of policy gradient algorithms, agnostic to the specific estimator of \(Q_\pi\), showing convergence is more complex as the estimated gradients are typically biased when using a learned value function \cite{konda2003onactor, sutton1999policy}. \cite{agarwal20optimality} and \cite{bhandari2021linear} consider the simplest case where state and action spaces \(\mathcal{S}\) and \(\mathcal{A}\) are finite and no function approximation is used, i.e. the policy uses a tabular parameterization with one parameter for each state-action combination. Using  that the exact policy gradients can be calculated in this case, both studies show the global convergence to an optimal policy with a linear convergence rate. 

\cite{sutton1999policy} and \cite{zhang2020global} generalize these results to settings with function approximation, albeit under impractical conditions on the approximators, which are required to be linear in their inputs. The extension to function approximation comes at the cost of only being able to proof local convergence using stochastic approximation and the Supermartingale Convergence Theorem \cite{robbins1971convergence}.

Finally, some convergence results exist for the specific algorithms such as TRPO and PPO. TRPO is based on Theorem \ref{th:trpo_theorem}, which comes with monotonic improvement guarantees. However, TRPO is only an approximation to the algorithm stemming from Theorem \ref{th:trpo_theorem}, so that no such guarantee holds for TRPO in practice. We further remark that PPO is similarly intended as an heuristic of this theoretical algorithm \cite{schulman2017proximal}. Nonetheless, efforts have been made to proof the convergence of these practical algorithms. \cite{liu2019neural} show that a slightly modified version of PPO converges to a globally optimal policy at a sublinear rate under specific assumptions. In particular, they require an overparameterized neural network as the function approximator such that they can use infinite-dimensional mirror-descent \cite{beck2003mirror} to proof the convergence. \cite{holzleitner2021convergence} provide a proof using two time-scale stochastic approximation \cite{karmakar2018two} that PPO converges to a locally optimal policy under more realistic assumptions akin to typical learning scenarios. 

\subsection{Mirror Learning}\label{sec:mirror_learning}

In this subsection, we focus on the convergence proof provided by \cite{kuba2022mirror}. While primarily of theoretical interest, we choose to discuss this particular result as it is agnostic to the selected algorithm and parameterization and can hence by applied to a range of policy gradient algorithms. \cite{kuba2022mirror} introduce a framework called \emph{mirror learning}, which comes with global convergence guarantees for all policy gradient algorithms that adhere to a specific form. In the following, we follow \cite{kuba2022mirror} in deriving their results. We start by giving some definitions, based on which we then present the general form of mirror learning updates. We show that the discussed algorithms largely adhere to this form. Finally, we proof that this implies the convergence to an optimal policy.

\subsubsection{Fundamentals of Mirror Learning}

From here onwards, we do not explicitly write down the policy parameters, i.e. we omit the subscript \(\theta\) when describing a policy \(\pi\). \cite{kuba2022mirror} define the drift \(\mathfrak{D}\) and the neighborhood operator \(\mathcal{N}\) as follows.

\begin{definition}\label{def:drift}
	(Drift)
	A drift functional 
	\[\mathfrak{D} \colon \Pi \times \mathcal{S} \rightarrow \{ \mathfrak{D}_\pi(\cdot \mid s) \colon \Delta(\mathcal{A}) \rightarrow \mathbb{R} \}\] 
	is a map which satisfies the following conditions for all \(s \in \mathcal{S}\) and \(\pi, \bar{\pi} \in \Pi\):
	\begin{enumerate}
		\item \(\mathfrak{D}_\pi(\bar{\pi} \mid s) \geq \mathfrak{D}_\pi(\pi \mid s) = 0\), \quad (non-negativity)
		\item \(\mathfrak{D}_\pi(\bar{\pi} \mid s)\) has zero gradient with respects to \(\bar{\pi}(\cdot \mid s)\) at \(\bar{\pi}(\cdot \mid s) = \pi(\cdot \mid s)\), more precisely all its Gâteaux derivatives\footnote{See Definition \ref{def:gateaux}.} are zero,  \quad (zero gradient)
	\end{enumerate}
	
	where we used \(\mathfrak{D}_\pi\bigl(\bar{\pi}(\cdot \mid s) \mid s\bigr) \coloneqq \mathfrak{D}_\pi(\bar{\pi} \mid s)\).
	
	For any state distribution \(\nu^{\bar{\pi}}_\pi \in \Delta(\mathcal{S})\), that can depend on both \(\bar{\pi}\) and \(\pi\), the drift from \(\bar{\pi}\) to \(\pi\) is given by
	\[\mathfrak{D}^\nu_\pi(\bar{\pi}) \coloneqq \mathbb{E}_{s \sim \nu^{\bar{\pi}}_\pi} \bigl[ \mathfrak{D}_\pi(\bar{\pi} \mid s) \bigr]. \]
	We require \(\nu^{\bar{\pi}}_\pi\) to be such that this expectation is continuous in \(\bar{\pi}\) and \(\pi\). We call a drift trivial if \(\mathfrak{D}^\nu_\pi(\bar{\pi}) = 0\) for all \(\pi, \bar{\pi} \in \Pi\).
\end{definition}

\begin{definition}
	(Neighborhood Operator)
	A mapping
	\[ \mathcal{N} \colon \Pi \rightarrow \mathcal{P}(\Pi) \]
	is a neighborhood operator if it satisfies the following conditions:
	\begin{enumerate}
		\item \(\mathcal{N}\) is continuous, \quad (continuity)
		\item \(\mathcal{N}(\pi)\) is compact for all \(\pi \in \Pi\), \quad (compactness)
		\item There exists a metric \(\chi \colon \Pi \times \Pi \rightarrow \mathbb{R}\) such that \(\chi(\pi, \bar{\pi}) \leq \zeta\) implies \(\bar{\pi} \in \mathcal{N}(\pi)\) for all \(\pi, \bar{\pi} \in \Pi\) given some \(\zeta \in \mathbb{R}_+\). \quad (closed ball)
	\end{enumerate} 
	We call \(\mathcal{N}(\cdot) = \Pi\) the trivial neighborhood operator.
\end{definition}

With these definitions, we can define the mirror learning update rule.

\begin{definition}
	(Mirror Learning Update)
	Let \(\pi_\text{old}\) be the previous policy and \(d^{\pi_\text{old}}\) the state distribution under \(\pi_\text{old}\). Further, let 
	\[ \Bigl[\mathcal{M}^{\bar{\pi}}_\mathfrak{D} V_{\pi} \Bigr](s) \coloneqq \mathbb{E}_{A \sim \bar{\pi}} \bigl[Q_{\pi}(s,A) \bigr] - \frac{\nu^{\bar{\pi}}_{\pi}}{d^{\pi}} \mathfrak{D}_\pi(\bar{\pi} \mid s) \]
	be the mirror learning operator.
	Then, the mirror learning update chooses the new policy \(\pi_\text{new}\) as
	\begin{equation}
		\pi_\text{new} \in \argmax_{\bar{\pi} \in \mathcal{N}(\pi_\text{old})} \mathbb{E}_{S \sim d^{\pi_\text{old}}} \biggl[ \Bigl[\mathcal{M}^{\bar{\pi}}_\mathfrak{D} V_{\pi_\text{old}} \Bigr](S) \biggr].\label{eq:mirror_update}
	\end{equation} 
\end{definition} 

Under the light of this mirror learning update, the drift \(\mathfrak{D}\) from one policy to the next induces some penalty on the objective while the neighborhood operator puts a hard constraint on the divergence of subsequent policies. 

\subsubsection{Policy Gradient Algorithms as Instances of Mirror Learning}

Before proving the convergence of mirror learning to an optimal policy, we first show that the discussed policy gradient algorithms can partly be seen as instances of mirror learning, i.e. use updates of the form
\[\pi_\text{new} \in \argmax_{\pi \in \mathcal{N}(\pi_\text{old})} \mathbb{E}_{S \sim d^{\pi_\text{old}}} \biggl[\mathbb{E}_{A \sim \pi} \bigl[Q_{\pi_\text{old}}(S,A) \bigr] - \frac{\nu^{\pi}_{\pi_\text{old}}}{d^{\pi_\text{old}}} \mathfrak{D}_{\pi_\text{old}}(\pi \mid S) \biggr]. \]

\textbf{A3C}

A3C is a direct application of the Policy Gradient Theorem, albeit with a learned advantage function. Thus, at each iteration it approximately solves the optimization problem
\[\pi_\text{new} \in \argmax_{\pi \in \Pi} \mathbb{E}_{S \sim d^{\pi_\text{old}}, A \sim \pi} \Bigl[A_{\pi_\text{old}}(S,A) \Bigr] = \argmax_{\pi \in \Pi} \mathbb{E}_{S \sim d^{\pi_\text{old}}} \Bigl[ \mathbb{E}_{A \sim \pi} \bigl[Q_{\pi_\text{old}}(S,A) \bigr] \Bigr]. \]
This is the most trivial instantiation of mirror learning by using the trivial drift \(\mathfrak{D}(\cdot \mid \cdot) = 0 \) and the trivial neighborhood operator \(\mathcal{N}(\cdot) = \Pi\). The same argumentation also applies to REINFORCE. Note that in practice however, we maximize the expectation over \(\pi_\text{old}\) rather than \(\pi\). For this reason, these are not exact instances of mirror learning.

\textbf{TRPO}

TRPO's constrained optimization problems
\begin{align*}
	&\pi_\text{new} \in \argmax_{\pi \in \Pi} \mathbb{E}_{S \sim d^{\pi_\text{old}}, A \sim \pi_\text{old}} \biggl[ \frac{ \pi (A \mid S)}{\pi_\text{old} (A \mid S)} A_{\pi_\text{old}}(S, A) \biggr] \\
	&\text{subject to} \quad \mathbb{E}_{S \sim d^{\pi_\text{old}}} \bigl[ D_{KL} (\pi_\text{old}(\cdot \mid S) \Vert \pi( \cdot \mid S)) \bigr] \leq \delta
\end{align*}
can be rewritten as 
\[\pi_\text{new} \in \argmax_{\pi \in \mathcal{N}_\text{TRPO}(\pi_\text{old})} \mathbb{E}_{S \sim d^{\pi_\text{old}}} \Bigl[ \mathbb{E}_{A \sim \pi} \bigl[Q_{\pi_\text{old}}(S,A) \bigr] \Bigr] \]
with the average-KL ball as the neighborhood operator, i.e.
\[ \mathcal{N}_\text{TRPO}(\pi_\text{old}) = \bigl\{ \pi \mid \mathbb{E}_{S \sim d^{\pi_\text{old}}} \bigl[ D_{KL} (\pi_\text{old}(\cdot \mid S) \Vert \pi( \cdot \mid S)) \bigr] \leq \delta \bigr\}. \]
Here, we used that 
\begin{align*}
	\mathbb{E}_{A \sim \pi_\text{old}} \biggl[ \frac{ \pi (A \mid s)}{\pi_\text{old} (A \mid s)} A_{\pi_\text{old}}(s, A) \biggr] &= \int_{a \in \mathcal{A}} \pi_\text{old}(a \mid s) \frac{ \pi (a \mid s)}{\pi_\text{old} (a \mid s)} A_{\pi_\text{old}}(s, a) da \\
	&= \mathbb{E}_{A \sim \pi} \Bigl[ A_{\pi_\text{old}}(s, A) \Bigr]
\end{align*}
and that maximizing over the action-value function is identical to maximizing over the advantage function following the discussion in Section \ref{sec:baseline}.
Thus, TRPO is a mirror learning instance with the trivial drift \(\mathfrak{D}(\cdot \mid \cdot) = 0 \). 

\textbf{PPO}

Each iteration, PPO searches for
\[\pi_\text{new} \in \argmax_{\pi \in \Pi} \mathbb{E}_{\pi_\text{old}} \biggl[ \min \biggl\{ r_{\pi}(A \mid S) , \text{clip}\bigl( r_{\pi}(A \mid S), 1-\varepsilon, 1+\varepsilon \bigr)  \biggr\}A_{\pi_\text{old}}(S, A) \biggr], \]
where we write \(r_\pi(a \mid s)\) for \(\frac{\pi(a \mid s)}{\pi_\text{old} (a \mid s)}\). We can rewrite the expectation over actions by adding zero as 
\begin{align*}
	\mathbb{E}_{A \sim \pi_\text{old}} &\biggl[ \min \biggl\{ r_{\pi}(A \mid s) A_{\pi_\text{old}}(s, A), \text{clip}\bigl( r_{\pi}(A \mid s), 1-\varepsilon, 1+\varepsilon \bigr) A_{\pi_\text{old}}(s, A) \biggr\} \biggr] \\
	&= \mathbb{E}_{A \sim \pi_\text{old}} \biggl[ r_\pi(A \mid s) A_{\pi_\text{old}}(s,A)  \biggr] - \mathbb{E}_{A \sim \pi_\text{old}} \biggl[ r_\pi(A \mid s) A_{\pi_\text{old}}(s,A) \\
	&\quad - \min \biggl\{ r_{\pi}(A \mid s) A_{\pi_\text{old}}(s, A), \text{clip}\bigl( r_{\pi}(A \mid s), 1-\varepsilon, 1+\varepsilon \bigr) A_{\pi_\text{old}}(s, A) \biggr\} \biggr].
\end{align*}
Using the same technique as before, we can write the first expectation equivalently as \(\mathbb{E}_{A \sim \pi} \bigl[ A_{\pi_\text{old}}(s, A) \bigr]\). We now focus on the second expectation. We replace the \(\min\) operator with a \(\max\) and push the first term inside the \(\max\) to obtain
\begin{align*}
	&\mathbb{E}_{A \sim \pi_\text{old}} \biggl[ r_\pi(A \mid s) A_{\pi_\text{old}}(s,A)  - \min \biggl\{ r_{\pi}(A \mid s) A_{\pi_\text{old}}(s, A), \text{clip}\bigl( r_{\pi}(A \mid s), 1-\varepsilon, 1+\varepsilon \bigr) A_{\pi_\text{old}}(s, A) \biggr\} \biggr] \\
	&\qquad = \mathbb{E}_{A \sim \pi_\text{old}} \biggl[ r_\pi(A \mid s) A_{\pi_\text{old}}(s,A)  + \max \biggl\{ - r_{\pi}(A \mid s) A_{\pi_\text{old}}(s, A), -\text{clip}\bigl( r_{\pi}(A \mid s), 1-\varepsilon, 1+\varepsilon \bigr) A_{\pi_\text{old}}(s, A) \biggr\} \biggr] \\
	&\qquad = \mathbb{E}_{A \sim \pi_\text{old}} \biggl[ \max \biggl\{ r_\pi(A \mid s) A_{\pi_\text{old}}(s,A) - r_{\pi}(A \mid s) A_{\pi_\text{old}}(s, A), \\
	&\qquad \qquad r_\pi(A \mid s) A_{\pi_\text{old}}(s,A)-\text{clip}\bigl( r_{\pi}(A \mid s), 1-\varepsilon, 1+\varepsilon \bigr) A_{\pi_\text{old}}(s, A) \biggr\} \biggr] \\
	&\qquad = \mathbb{E}_{A \sim \pi_\text{old}} \biggl[ \max \biggl\{0, \Bigl( r_\pi(A \mid s) -\text{clip}\bigl( r_{\pi}(A \mid s), 1-\varepsilon, 1+\varepsilon \bigr) \Bigr) A_{\pi_\text{old}}(s, A) \biggr\} \biggr]. \\
\end{align*}
This final expression is non-negative. Moreover, it is zero for \(\pi\) sufficiently close to \(\pi_\text{old}\), i.e. such that for all actions \(a \in \mathcal{A}\) we have \(r_\pi(a \mid s) = \frac{\pi(a \mid s)}{\pi_\text{old} (a \mid s)} \in [1- \varepsilon, 1 + \varepsilon]\), because then the \(\mathrm{clip}\)-function reduces to the identity function w.r.t. its first argument. Thus, the derivatives of this expression must also be zero at \(\pi(\cdot \mid s) = \pi_\text{old}( \cdot \mid s)\). These properties are the exact conditions for a mapping to be considered a drift in the sense of Definition \ref{def:drift}. With this preparation, we can now write the PPO update as 
\[\pi_\text{new} \in \argmax_{\pi \in \Pi} \mathbb{E}_{S \sim d^{\pi_\text{old}}} \biggl[ \mathbb{E}_{A \sim \pi} \Bigl[ A_{\pi_\text{old}}(S, A) \Bigr] - \mathfrak{D}_{\pi_\text{old}}(\pi \mid S) \biggr], \]
where \(\mathfrak{D}_{\pi_\text{old}}\) is a drift given by
\[\mathfrak{D}_{\pi_\text{old}}(\pi \mid s) = \mathbb{E}_{A \sim \pi_\text{old}} \biggl[ \max \biggl\{0, \Bigl( r_{\pi}(A \mid s) - \mathrm{clip}\bigl(r_{\pi}(A \mid s), 1 - \varepsilon, 1 + \varepsilon \bigr) \Bigr) A_{\pi_\text{old}}(s,A) \biggr\} \biggr]. \]
This is an instance of the mirror learning update with the trivial neighborhood operator \(\mathcal{N}(\cdot) = \Pi\) and \(\nu^\pi_{\pi_\text{old}} = d^{\pi_\text{old}} \).

\subsubsection{Convergence Proof}

Now, we present the main theoretical result of \cite{kuba2022mirror}. 

\begin{theorem}\label{th:mirror_learning}
	Let \(\mathfrak{D}^\nu\) be a drift, \(\mathcal{N}\) a neighborhood operator and \(d^\pi\) the sampling distribution, all continuous in \(\pi\). Let the objective, i.e. the expected returns under a policy \(\pi\), be written as \(J(\pi) = \mathbb{E}_{S_0 \sim p_0, \pi} \bigl[ G_0 \bigr]\). Let \(\pi_0 \in \Pi\) be the initial policy and the sequence of policies \(\bigl( \pi_n \bigr)^\infty_{n=0}\) be obtained through the mirror learning update rule \eqref{eq:mirror_update} under \(\mathfrak{D}^\nu\), \(\mathcal{N}\) and \(d^\pi\). Then,
	\begin{enumerate}
		\item \quad (Strict monotonic improvement)
		\[ J(\pi_{n+1}) \geq J(\pi_n) + \mathbb{E}_{S \sim p_0} \biggl[\frac{\nu^{\pi_{n+1}}_{\pi_n}(S)}{d^{\pi_n}(S)} \mathfrak{D}_{\pi_n} (\pi_{n+1} \mid S) \biggr] \quad \forall n \in \mathbb{N}_0. \]
		\item (Value function optimality)
		\[\lim_{n \to \infty} V_{\pi_n} = V^*.\] \quad 
		\item (Maximum attainable return)
		\[ \lim_{n \to \infty} J(\pi_n) = \max_{\pi \in \Pi} J(\pi). \]
		\item (Policy optimality) 
		\[\lim_{n \to \infty} \pi_n = \pi^*.\]
	\end{enumerate}
\end{theorem}

\begin{proof}
	We structure the proof by \cite{kuba2022mirror} in five steps. In step 1, we start by showing that mirror learning updates lead to improvements under the mirror learning operator \(\mathcal{M}^{\pi_n}_\mathfrak{D} V_{\pi_{n-1}}\), which implies improvements in the value function \(V_{\pi_n}\). In step 2, we prove that the sequence of value functions \(\bigl(V_{\pi_n}\bigr)^\infty_{n=0}\) converges to some limit. In step 3, we show the existence of limit points of the sequence of policies \(\bigl(\pi_n\bigr)^\infty_{n=0}\), which are fixed points of the mirror learning update \eqref{eq:mirror_update}. In step 4, we prove that these limit points are also fixed points of Generalized Policy Iteration (GPI) \cite{Sutton1998}, from which we conclude that these limit points are optimal policies in step 5. For simplicity, we proof Theorem \ref{th:mirror_learning} for discrete state and actions spaces. However, the results are straightforward to extended to the continuous cases (see the appendix in \cite{kuba2022mirror} for details).
	
	\textbf{Step 1} 
	
	We start by showing by contradiction that for all \(n \in \mathbb{N}_0\) and for all \(s \in \mathcal{S}\):
	\begin{equation}
		\bigl[\mathcal{M}^{\pi_{n+1}}_\mathfrak{D} V_{\pi_n}\bigr] (s) \geq \bigl[\mathcal{M}^{\pi_n}_\mathfrak{D} V_{\pi_n}\bigr] (s). \label{eq:mirror_operator_impr}
	\end{equation}
	
	Suppose there exists \(s_0 \in \mathcal{S}\), which violates \eqref{eq:mirror_operator_impr}. We define a policy \(\hat{\pi}\) with
	\[ \hat{\pi}(\cdot \mid s) = \begin{cases}
		\pi_{n+1}(\cdot \mid s) & \text{if } s \neq s_0, \\
		\pi_{n}(\cdot \mid s) & \text{if } s = s_0.
	\end{cases} \]
	This way, we guarantee \(\hat{\pi} \in \mathcal{N}(\pi_n)\) because \(\pi_{n+1} \in \mathcal{N}(\pi_n)\) is forced by the mirror learning update \eqref{eq:mirror_update} and the distance between \(\hat{\pi}\) and \(\pi_n\) is similar to the distance between \(\pi_{n+1}\) and \(\pi_n\) at every \(s \neq s_0\) but smaller at \(s = s_0\). 
	
	By assumption, we have at \(s_0\) that 
	\begin{align*}
		\bigl[\mathcal{M}^{\hat{\pi}}_\mathfrak{D} V_{\pi_n} \bigr](s_0) &= \mathbb{E}_{A \sim \hat{\pi}} \Bigl[Q_{\pi_n}(s_0,A) \Bigr] - \frac{\nu^{\hat{\pi}}_{\pi_n}}{d^{\pi_n}} \mathfrak{D}_{\pi_n}(\hat{\pi} \mid s_0) \\
		&= \mathbb{E}_{A \sim \pi_{n}} \Bigl[Q_{\pi_n}(s_0,A) \Bigr] - \frac{\nu^{\pi_{n}}_{\pi_n}}{d^{\pi_n}} \mathfrak{D}_{\pi_n}(\pi_{n} \mid s_0) \\
		&= \bigl[\mathcal{M}^{\pi_{n}}_\mathfrak{D} V_{\pi_n} \bigr](s_0) \\
		&> \bigl[\mathcal{M}^{\pi_{n+1}}_\mathfrak{D} V_{\pi_n} \bigr](s_0).
	\end{align*} 
	
	It follows that 
	\begin{align*}
		\mathbb{E}_{S \sim d^{\pi_n}} \biggl[ \Bigl[\mathcal{M}^{\hat{\pi}}_\mathfrak{D} V_{\pi_n} \Bigr](S) \biggr] - \mathbb{E}_{S \sim d^{\pi_n}} \biggl[ \Bigl[\mathcal{M}^{\pi_{n+1}}_\mathfrak{D} V_{\pi_n} \Bigr](S) \biggr] = d^{\pi_n}(s_0) \biggl( \Bigl[\mathcal{M}^{\hat{\pi}}_\mathfrak{D} V_{\pi_n} \Bigr](s_0) - \Bigl[\mathcal{M}^{\pi_{n+1}}_\mathfrak{D} V_{\pi_n} \Bigr](s_0) \biggr) > 0,
	\end{align*}
	where we used that \(\Bigl[\mathcal{M}^{\hat{\pi}}_\mathfrak{D} V_{\pi_n} \Bigr](s) = \Bigl[\mathcal{M}^{\pi_{n+1}}_\mathfrak{D} V_{\pi_n} \Bigr](s)\) for \(s \neq s_0\). Thus, 
	\[\mathbb{E}_{S \sim d^{\pi_n}} \biggl[ \Bigl[\mathcal{M}^{\hat{\pi}}_\mathfrak{D} V_{\pi_n} \Bigr](S) \biggr] > \mathbb{E}_{S \sim d^{\pi_n}} \biggl[ \Bigl[\mathcal{M}^{\pi_{n+1}}_\mathfrak{D} V_{\pi_n} \Bigr](S) \biggr],\]
	which contradicts the mirror learning update rule, i.e. that 
	\[ \mathbb{E}_{S \sim d^{\pi_n}} \biggl[ \Bigl[\mathcal{M}^{\pi_{n+1}}_\mathfrak{D} V_{\pi_n} \Bigr](S) \biggr] = \max_{\bar{\pi} \in \mathcal{N}(\pi_n)} \mathbb{E}_{S \sim d^{\pi_n}} \biggl[ \Bigl[\mathcal{M}^{\bar{\pi}}_\mathfrak{D} V_{\pi_n} \Bigr](S) \biggl]. \]
	Hence, we have shown that the sequence of policies \(\bigl( \pi_n \bigr)^\infty_{n=0}\) created by the mirror learning updates monotonically increases the mirror learning operator at every state. Next, we show that this property, i.e. \(\bigl[\mathcal{M}^{\pi_{n+1}}_\mathfrak{D} V_{\pi_n}\bigr] (s) \geq \bigl[\mathcal{M}^{\pi_n}_\mathfrak{D} V_{\pi_n}\bigr] (s)\),
	implies the monotonic improvement in the value function
	\begin{equation}
		V_{\pi_{n+1}} (s) \geq V_{\pi_n} (s) \label{eq:mirror_value_improv}
	\end{equation}
	for all \(s \in \mathcal{S}\) and \(n \in \mathbb{N}_0\).
	
	By using the definitions of the value function \(V_\pi\), the action-value function \(Q_\pi\), the mirror learning operator \(\mathcal{M}^{\bar{\pi}}_\mathfrak{D} V_{\pi}\) and the identity \(\mathfrak{D}_{\pi}(\pi \mid s) = 0\), adding zeros and rearranging, we obtain
	\begin{align}
		V_{\pi_{n+1}}(s) - V_{\pi_n}(s)
		&= \mathbb{E}_{\pi_{n+1}} \Bigl[R + \gamma V_{\pi_{n+1}}(S') \Bigr] - \mathbb{E}_{\pi_{n}} \Bigl[R + \gamma V_{\pi_{n}}(S') \Bigr] \nonumber \\
		&= \mathbb{E}_{\pi_{n+1}} \Bigl[R + \gamma V_{\pi_{n+1}}(S') \Bigr] - \mathbb{E}_{\pi_{n}} \Bigl[R + \gamma V_{\pi_{n}}(S') \Bigr] \nonumber \\
		&\qquad + \frac{\nu^{\pi_{n+1}}_{\pi_n}(s)}{d^{\pi_n}(s)} \mathfrak{D}_{\pi_n}(\pi_{n+1} \mid s) - \frac{\nu^{\pi_{n+1}}_{\pi_n}(s)}{d^{\pi_n}(s)} \mathfrak{D}_{\pi_n}(\pi_{n+1} \mid s) \nonumber \\
		&= \mathbb{E}_{\pi_{n+1}} \Bigl[R + \gamma V_{\pi_{n+1}}(S') + \gamma V_{\pi_{n}}(S') - \gamma V_{\pi_{n}}(S') \Bigr] - \mathbb{E}_{\pi_{n}} \Bigl[R + \gamma V_{\pi_{n}}(S') \Bigr] \nonumber \\
		&\qquad + \frac{\nu^{\pi_{n+1}}_{\pi_n}(s)}{d^{\pi_n}(s)} \mathfrak{D}_{\pi_n}(\pi_{n+1} \mid s) - \frac{\nu^{\pi_{n+1}}_{\pi_n}(s)}{d^{\pi_n}(s)} \mathfrak{D}_{\pi_n}(\pi_{n+1} \mid s) \nonumber \\
		&= \biggl( \mathbb{E}_{\pi_{n+1}} \Bigl[R + \gamma V_{\pi_{n}}(S') \Bigr] - \frac{\nu^{\pi_{n+1}}_{\pi_n}(s)}{d^{\pi_n}(s)} \mathfrak{D}_{\pi_n}(\pi_{n+1} \mid s) \biggr) \nonumber \\
		&\qquad - \biggl( \mathbb{E}_{\pi_{n}} \Bigl[R + \gamma V_{\pi_{n}}(S') \Bigr] - \frac{\nu^{\pi_{n}}_{\pi_n}(s)}{d^{\pi_n}(s)} \mathfrak{D}_{\pi_n}(\pi_{n} \mid s) \biggr) \nonumber \\
		&\qquad + \gamma \mathbb{E}_{\pi_{n+1}} \Bigl[V_{\pi_{n+1}}(S') - V_{\pi_{n}}(S') \Bigr] + \frac{\nu^{\pi_{n+1}}_{\pi_n}(s)}{d^{\pi_n}(s)} \mathfrak{D}_{\pi_n}(\pi_{n+1} \mid s) \nonumber \\
		&= \biggl( \mathbb{E}_{\pi_{n+1}} \Bigl[Q_{\pi_{n}}(s, A) \Bigr] - \frac{\nu^{\pi_{n+1}}_{\pi_n}(s)}{d^{\pi_n}(s)} \mathfrak{D}_{\pi_n}(\pi_{n+1} \mid s) \biggr) \nonumber \\
		&\qquad - \biggl( \mathbb{E}_{\pi_{n}} \Bigl[Q_{\pi_{n}}(s, A) \Bigr] - \frac{\nu^{\pi_{n}}_{\pi_n}(s)}{d^{\pi_n}(s)} \mathfrak{D}_{\pi_n}(\pi_{n} \mid s) \biggr) \nonumber \\
		&\qquad + \gamma \mathbb{E}_{\pi_{n+1}} \Bigl[V_{\pi_{n+1}}(S') - V_{\pi_{n}}(S') \Bigr] + \frac{\nu^{\pi_{n+1}}_{\pi_n}(s)}{d^{\pi_n}(s)} \mathfrak{D}_{\pi_n}(\pi_{n+1} \mid s) \nonumber \\
		\begin{split}
		&= \bigl[\mathcal{M}^{\pi_{n+1}}_\mathfrak{D} V_{\pi_n} \bigr] - \bigl[ \mathcal{M}^{\pi_n}_\mathfrak{D} V_{\pi_n} \bigr] \\
		&\qquad + \gamma \mathbb{E}_{\pi_{n+1}} \Bigl[V_{\pi_{n+1}}(S') - V_{\pi_{n}}(S') \Bigr] + \frac{\nu^{\pi_{n+1}}_{\pi_n}(s)}{d^{\pi_n}(s)} \mathfrak{D}_{\pi_n}(\pi_{n+1} \mid s) \label{eq:mirror_1}
		\end{split} \\
		&\geq \gamma \mathbb{E}_{\pi_{n+1}} \Bigl[V_{\pi_{n+1}}(S') - V_{\pi_{n}}(S') \Bigr] + \frac{\nu^{\pi_{n+1}}_{\pi_n}(s)}{d^{\pi_n}(s)} \mathfrak{D}_{\pi_n}(\pi_{n+1} \mid s), \nonumber
	\end{align}
	where we used Inequality \eqref{eq:mirror_operator_impr} in the final step. We take the infimum over states and replace the expectation with another infimum over states as a lower bound:
	\begin{align*}
		\inf_{s \in \mathcal{S}} \Bigl[ V_{\pi_{n+1}}(s) - V_{\pi_n}(s) \Bigr] &\geq \inf_{s \in \mathcal{S}} \biggl[ \gamma \mathbb{E}_{\pi_{n+1}} \Bigl[V_{\pi_{n+1}}(S') - V_{\pi_{n}}(S') \Bigr] + \frac{\nu^{\pi_{n+1}}_{\pi_n}(s)}{d^{\pi_n}(s)} \mathfrak{D}_{\pi_n}(\pi_{n+1} \mid s) \biggr] \\
		&\geq \inf_{s \in \mathcal{S}} \biggl[ \gamma \inf_{s' \in \mathcal{S}} \Bigl[V_{\pi_{n+1}}(s') - V_{\pi_{n}}(s') \Bigr] + \frac{\nu^{\pi_{n+1}}_{\pi_n}(s)}{d^{\pi_n}(s)} \mathfrak{D}_{\pi_n}(\pi_{n+1} \mid s) \biggr] \\
		&= \gamma \inf_{s' \in \mathcal{S}} \Bigl[V_{\pi_{n+1}}(s') - V_{\pi_{n}}(s') \Bigr] + \inf_{s \in \mathcal{S}} \biggl[ \frac{\nu^{\pi_{n+1}}_{\pi_n}(s)}{d^{\pi_n}(s)} \mathfrak{D}_{\pi_n}(\pi_{n+1} \mid s) \biggr]. 
	\end{align*}
	
	From this expression, we obtain
	\[ \inf_{s \in \mathcal{S}} \Bigl[ V_{\pi_{n+1}}(s) - V_{\pi_n}(s) \Bigr] \geq \frac{1}{1 - \gamma} \inf_{s \in \mathcal{S}} \biggl[ \frac{\nu^{\pi_{n+1}}_{\pi_n}(s)}{d^{\pi_n}(s)} \mathfrak{D}_{\pi_n}(\pi_{n+1} \mid s) \biggr] \geq 0,\]
	since \(\nu^{\pi_{n+1}}_{\pi_n}(s)\) and \(d^{\pi_n}(s)\) are probabilities and the drift \(\mathfrak{D}\) is non-negative. Thus, we have proven the monotonic improvement of value functions \(V_{\pi_{n+1}} (s) \geq V_{\pi_n} (s)\).
	We observe that this already implies the strict monotonic improvement property 
	\[ J(\pi_{n+1}) \geq J(\pi_n) + \mathbb{E}_{S \sim p_0} \biggl[ \frac{\nu^{\pi_{n+1}}_{\pi_n}(S)}{d^{\pi_n}(S)} \mathfrak{D}_{\pi_n} (\pi_{n+1} \mid S) \biggr] \]
	for all \(n \in \mathbb{N}_0\) since applying \eqref{eq:mirror_value_improv} and \eqref{eq:mirror_operator_impr} sequentially to \eqref{eq:mirror_1} yields for all \(s \in \mathcal{S}\)
	\begin{align*}
		V_{\pi_{n+1}} (s) - V_{\pi_n} (s) 
		&= \bigl[\mathcal{M}^{\pi_{n+1}}_\mathfrak{D} V_{\pi_n} \bigr] - \bigl[ \mathcal{M}^{\pi_n}_\mathfrak{D} V_{\pi_n} \bigr] \\
		&\quad + \gamma \mathbb{E}_{\pi_{n+1}} \Bigl[V_{\pi_{n+1}}(S') - V_{\pi_{n}}(S') \Bigr] + \frac{\nu^{\pi_{n+1}}_{\pi_n}(s)}{d^{\pi_n}(s)} \mathfrak{D}_{\pi_n}(\pi_{n+1} \mid s) \\
		&\geq \bigl[\mathcal{M}^{\pi_{n+1}}_\mathfrak{D} V_{\pi_n} \bigr] - \bigl[ \mathcal{M}^{\pi_n}_\mathfrak{D} V_{\pi_n} \bigr] +  \frac{\nu^{\pi_{n+1}}_{\pi_n}(s)}{d^{\pi_n}(s)} \mathfrak{D}_{\pi_n}(\pi_{n+1} \mid s) \\
		&\geq \frac{\nu^{\pi_{n+1}}_{\pi_n}(s)}{d^{\pi_n}(s)} \mathfrak{D}_{\pi_n}(\pi_{n+1} \mid s).
	\end{align*}
	We obtain the desired inequality by taking the expectation over \(S \sim p_0\).
	
	\textbf{Step 2}
	
	From step 1, we know that the value functions increase uniformly over the state space, i.e. 
	\( V_{\pi_{n+1}}(s) - V_{\pi_n} (s) \geq 0\), for all s \(\in \mathcal{S}\), \(n \in \mathbb{N}_0\).
	As the rewards \(r\) are bounded by assumption and we consider the episodic case where episode lengths are also bounded by \(T\) (albeit the same argument applies for infinite time horizons via discounting), the value functions \(V_\pi(s) = \mathbb{E}_{\pi}\bigl[\sum_{k=0}^{T} \gamma^k R_{t+k+1} \mid S_t = s\bigr]\) are also uniformly bounded. Via the Monotone Convergence Theorem (Theorem \ref{th:monotone_convergence}), the sequence of value functions \(\bigl( V_{\pi_n} \bigr)^\infty_{n=0}\) must therefore converge to some limit \(V\).
	
	\textbf{Step 3} 
	
	Now, we show the existence of limit points of the sequence of policies \(\bigl(\pi_n\bigr)^\infty_{n=0}\) and prove by contradiction that these are fixed points of the mirror learning update \eqref{eq:mirror_update}.
	
	The sequence \(\bigl(\pi_n\bigr)^\infty_{n=0}\) is bounded, thus the Bolzano-Weierstrass Theorem (Theorem \ref{th:bolzano}) yields the existence of limits \(\bar{\pi}\) to which some respective subsequence \(\bigl(\pi_{n_i}\bigr)^\infty_{i=0}\) converges. We denote this set of limit points as \(L\Pi\). 
	For each element of such a convergent subsequence \(\bigl(\pi_{n_i}\bigr)^\infty_{i=0}\), mirror learning solves the optimization problem 
	\begin{equation}
		\max_{\pi \in \mathcal{N}(\pi_{n_i})} \mathbb{E}_{S \sim d^{\pi_{n_i}}} \Bigl[ \bigl[ \mathcal{M}^\pi_\mathfrak{D} V_{\pi_{n_i}} \bigr] (S) \Bigr] \label{eq:mirror_subsequence_operator}
	\end{equation}
	This expression is continuous in \(\pi_{n_i}\) due to the continuity of the value function \cite{kuba2021trust}, the drift and neighborhood operator (by definition) and the sampling distribution (by assumption). Let \(\bar{\pi} = \lim_{i \to \infty} \pi_{n_i}\). Berge's Maximum Theorem (Theorem \ref{th:berge}) \cite{ausubel1993generalized} now guarantees the convergence of the above expression, yielding
	\begin{equation}
		\lim_{i \to \infty} \max_{\pi \in \mathcal{N}(\pi_{n_i})} \mathbb{E}_{S \sim d^{\pi_{n_i}}} \Bigl[ \bigl[ \mathcal{M}^\pi_\mathfrak{D} V_{\pi_{n_i}} \bigr] (S) \Bigr] = \max_{\pi \in \mathcal{N}(\bar{\pi})} \mathbb{E}_{S \sim d^{\bar{\pi}}} \Bigl[ \bigl[ \mathcal{M}^\pi_\mathfrak{D} V_{\bar{\pi}} \bigr] (S) \Bigr]. \label{eq:mirror_limit_operator}
	\end{equation}
	For all \(i \in \mathbb{N}_0\), we obtain the next policy \(\pi_{n_i+1}\) as the argmax of Expression \eqref{eq:mirror_subsequence_operator}. Since this expression converges to the limit in \eqref{eq:mirror_limit_operator}, there must exist some subsequence \(\bigl( \pi_{n_{i_k}+1} \bigr)^\infty_{k=0}\) of \(\bigl( \pi_{n_i+1} \bigr)^\infty_{i=0}\) which converges to some policy \(\pi'\), which is the solution to the optimization problem \eqref{eq:mirror_limit_operator}. We now show by contradiction that \(\pi' = \bar{\pi}\), which implies that \(\bar{\pi}\) is a fixed point of the mirror learning update rule. 
	
	Suppose \(\pi' \neq \bar{\pi}\). As \(\pi'\) is induced by the mirror learning update rule, the monotonic improvement results from step 1 yield
	\begin{equation}
		Q_{\pi'} (s,a) = \mathbb{E}_{R,S' \sim P} \Bigl[R + \gamma V_{\pi'}(S')\Bigr] \geq \mathbb{E}_{R,S' \sim P} \Bigl[R + \gamma V_{\bar{\pi}}(S')\Bigr] = Q_{\bar{\pi}} (s,a) \label{eq:mirror_limit_value_impr}
	\end{equation} 
	and 
	\[\bigl[\mathcal{M}^{\pi'}_\mathfrak{D} V_{\bar{\pi}}\bigr] (s) \geq \bigl[\mathcal{M}^{\bar{\pi}}_\mathfrak{D} V_{\bar{\pi}}\bigr] (s).\]
	Suppose 
	\[\mathbb{E}_{S \sim d^{\bar{\pi}}} \Bigl[ \bigl[\mathcal{M}^{\pi'}_\mathfrak{D} V_{\bar{\pi}}\bigr] (S) \Bigr] > \mathbb{E}_{S \sim d^{\bar{\pi}}} \Bigl[ \bigl[\mathcal{M}^{\bar{\pi}}_\mathfrak{D} V_{\bar{\pi}}\bigr] (S) \Bigr],\]
	then we have for some state \(s\)
	\begin{align*}
		\bigl[ \mathcal{M}^{\pi'}_\mathfrak{D} V_{\bar{\pi}} \bigr] (s) &= \mathbb{E}_{\pi'} \Bigl[Q_{\bar{\pi}}(s,A) \Bigr] - \frac{\nu^{\pi'}_{\bar{\pi}}(s)}{d^{\bar{\pi}}(s)}\mathfrak{D}_{\bar{\pi}} (\pi' \mid s) \\
		&> \bigl[\mathcal{M}^{\bar{\pi}}_\mathfrak{D} V_{\bar{\pi}}\bigr] (s) = \mathbb{E}_{\bar{\pi}} \Bigl[Q_{\bar{\pi}}(s,A) \Bigr] - \frac{\nu^{\bar{\pi}}_{\bar{\pi}}(s)}{d^{\bar{\pi}}(s)}\mathfrak{D}_{\bar{\pi}} (\bar{\pi} \mid s) \\
		&= \mathbb{E}_{\bar{\pi}} \Bigl[Q_{\bar{\pi}}(s,A) \Bigr] = V_{\bar{\pi}}(s) = V(s).
	\end{align*}
	In the last equality, we used that the sequence of value functions converges to some unique limit \(V\), which implies \(V_{\bar{\pi}} = V\). We obtain the following via this result, Inequality \eqref{eq:mirror_limit_value_impr}, which must be strict for \(s\), and the non-negativity of the drift \(\mathfrak{D}\):
	\begin{align*}
		V_{\pi'}(s) &= \mathbb{E}_{\pi'} \bigl[Q_{\pi'}(s,A) \bigr] \\
		&> \mathbb{E}_{\pi'} \bigl[Q_{\bar{\pi}}(s,A) \bigr] \\
		&> \mathbb{E}_{\pi'} \bigl[Q_{\bar{\pi}}(s,A) \bigr] - \frac{\nu^{\pi'}_{\bar{\pi}}(s)}{d^{\bar{\pi}}(s)}\mathfrak{D}_{\bar{\pi}} (\pi' \mid s) \\
		&> V(s).
	\end{align*}
	However due to \(V_{\pi'}(s) = \lim_{k \to \infty} V_{\pi_{n_{i_k}+1}}\), this contradicts the uniqueness of the value limit, which gives \(V_{\pi'} = V\). Therefore, we have shown by contradiction that 
	\begin{equation*}
		\bar{\pi} \in \argmax_{\pi \in \mathcal{N}(\bar{\pi})} \mathbb{E}_{S \sim d^{\bar{\pi}}} \Bigl[ \bigl[ \mathcal{M}^\pi_\mathfrak{D} V_{\bar{\pi}} \bigr] (S) \Bigr].
	\end{equation*}
	
	\textbf{Step 4} 
	
	Following step 3, let \(\bar{\pi}\) be a limit point of \(\bigl(\pi_n \bigr)^\infty_{n=0}\). We will show by contradiction that \(\bar{\pi}\) is also a fixed point of GPI (see Theorem \ref{th:gpi}), i.e. that for all \(s \in \mathcal{S}\)
	\begin{equation}
		\bar{\pi} \in \argmax_{\pi \in \Pi} \mathbb{E}_{A \sim \pi} \bigl[A_{\bar{\pi}}(s,A) \bigr] = \argmax_{\pi \in \Pi} \mathbb{E}_{A \sim \pi} \bigl[Q_{\bar{\pi}}(s,A) \bigr]. \label{eq:mirror_gpi_policy}
	\end{equation}
	From step 3, we know that 
	\begin{align}
		\bar{\pi} &\in \argmax_{\pi \in \Pi} \biggl[ \mathbb{E}_{S \sim d^{\bar{\pi}}, A \sim \pi} \Bigl[Q_{\bar{\pi}}(S,A) - \frac{\nu^{\pi}_{\bar{\pi}}(S)}{d^{\bar{\pi}}(S)}\mathfrak{D}_{\bar{\pi}} (\pi \mid S) \Bigr] \biggr] \nonumber \\
		&= \argmax_{\pi \in \Pi} \biggl[ \mathbb{E}_{S \sim d^{\bar{\pi}}, A \sim \pi} \Bigl[A_{\bar{\pi}}(S,A) - \frac{\nu^{\pi}_{\bar{\pi}}(S)}{d^{\bar{\pi}}(S)}\mathfrak{D}_{\bar{\pi}} (\pi \mid S) \Bigr] \biggr] \label{eq:mirror_s4_contr}
	\end{align}
	as subtracting an action-independent baseline does not affect the argmax. Now, we assume the existence of a policy \(\pi'\) and state \(s\) with
	\begin{equation}
		\mathbb{E}_{A \sim \pi'} \bigl[A_{\bar{\pi}}(s,A) \bigr] > \mathbb{E}_{A \sim \bar{\pi}} \bigl[A_{\bar{\pi}}(s,A) \bigr] = 0. \label{eq:mirror_s4_assumption}
	\end{equation}	

	Let \(m = \lvert \mathcal{A} \rvert\) denote the size of the action space. Then, we can write for any policy \(\pi\), \(\pi(\cdot \mid s) = \bigl( x_1, \hdots, x_{m-1}, 1 - \sum^{m-1}_{i=1} x_i \bigr)\). With this notation, we have
	\begin{align*}
		\mathbb{E}_{A \sim \pi} \bigl[A_{\bar{\pi}}(s,A) \bigr] &= \sum^m_{i=1} \pi(a_i \mid s)A_{\bar{\pi}}(s,a_i) \\
		&= \sum^{m-1}_{i=1} x_i A_{\bar{\pi}}(s,a_i) + \Bigl(1 - \sum^{m-1}_{i=1} x_i \Bigr)A_{\bar{\pi}}(s,a_m) \\
		&= \sum^{m-1}_{i=1} x_i \Bigl( A_{\bar{\pi}}(s,a_i) - A_{\bar{\pi}}(s,a_m) \Bigr) + A_{\bar{\pi}}(s,a_m).
	\end{align*}
	This shows that \(\mathbb{E}_{A \sim \pi} \bigl[A_{\bar{\pi}}(s,A) \bigr]\) is an affine function of \(\pi(\cdot \mid s)\), which implies that all its Gâteaux derivatives are constant in \(\Delta(\mathcal{A})\) for fixed directions. Due to Inequality \eqref{eq:mirror_s4_assumption}, this further implies that the Gâteaux derivatives in direction from \(\bar{\pi}\) to \(\pi'\) are strictly positive. Additionally, we have that the Gâteaux derivatives of \(\frac{\nu^{\pi}_{\bar{\pi}}(s)}{d^{\bar{\pi}}(s)}\mathfrak{D}_{\bar{\pi}} (\pi \mid s)\) are zero at \(\pi = \bar{\pi}\). We see this by establishing  lower and upper bounds, which both have derivatives of zero due to the independence of \(\pi\) and the zero-gradient property of the drift:
	\begin{equation*}
		\frac{1}{d^{\bar{\pi}}(s)}\mathfrak{D}_{\bar{\pi}} (\bar{\pi} \mid s) = \frac{\nu^{\bar{\pi}}_{\bar{\pi}}(s)}{d^{\bar{\pi}}(s)}\mathfrak{D}_{\bar{\pi}} (\bar{\pi} \mid s) = 0 \leq \frac{\nu^{\pi}_{\bar{\pi}}(s)}{d^{\bar{\pi}}(s)}\mathfrak{D}_{\bar{\pi}} (\pi \mid s) \leq \frac{1}{d^{\bar{\pi}}(s)}\mathfrak{D}_{\bar{\pi}} (\pi \mid s)
	\end{equation*}
	recalling that \(\mathfrak{D}_{\bar{\pi}} (\bar{\pi} \mid s) = 0\) for any \(s \in \mathcal{S}\) and using \(\nu^{\pi}_{\bar{\pi}}(s) \leq 1\). In combination, we obtain that the Gâteaux derivative of \(\mathbb{E}_{A \sim \pi} \bigl[A_{\bar{\pi}}(s,A)  \bigr] - \frac{\nu^{\pi}_{\bar{\pi}}(s)}{d^{\bar{\pi}}(s)}\mathfrak{D}_{\bar{\pi}} (\pi \mid s)\) is strictly positive as well.
	Therefore, we can find some policy \(\hat{\pi}(\cdot \mid s)\) by taking a sufficiently small step from \(\bar{\pi}(\cdot \mid s)\) in the direction of \(\pi'(\cdot \mid s)\) such that \(\hat{\pi} \in \mathcal{N}(\bar{\pi})\) and
	\begin{equation*}
		\mathbb{E}_{A \sim \hat{\pi}} \bigl[A_{\bar{\pi}}(s,A)  \bigr] - \frac{\nu^{\hat{\pi}}_{\bar{\pi}}(s)}{d^{\bar{\pi}}(s)}\mathfrak{D}_{\bar{\pi}} (\hat{\pi} \mid s) > \mathbb{E}_{A \sim \bar{\pi}} \bigl[A_{\bar{\pi}}(s,A)  \bigr] - \frac{\nu^{\bar{\pi}}_{\bar{\pi}}(s)}{d^{\bar{\pi}}(s)}\mathfrak{D}_{\bar{\pi}} (\bar{\pi} \mid s) = 0.
	\end{equation*}
	With this, we can construct a policy which contradicts Equation \eqref{eq:mirror_s4_contr}. Let \(\tilde{\pi}\) be defined such that 
	\[\tilde{\pi}(\cdot \mid x) = \begin{cases}
		\bar{\pi}(\cdot \mid x) & \text{if } x \neq s, \\
		\hat{\pi}(\cdot \mid x) & \text{if } x = s.
	\end{cases} \]
	This guarantees \(\tilde{\pi} \in \mathcal{N}(\bar{\pi})\) and 
	\begin{align*}
		\mathbb{E}_{S \sim d^{\bar{\pi}}} &\biggl[ \mathbb{E}_{A \sim \tilde{\pi}} \bigl[A_{\bar{\pi}}(S,A)  \bigr] - \frac{\nu^{\tilde{\pi}}_{\bar{\pi}}(S)}{d^{\bar{\pi}}(S)}\mathfrak{D}_{\bar{\pi}} (\tilde{\pi} \mid S) \biggr] \\
		&= d^{\bar{\pi}}(s)\biggl( \mathbb{E}_{A \sim \tilde{\pi}} \bigl[A_{\bar{\pi}}(s,A)  \bigr] - \frac{\nu^{\tilde{\pi}}_{\bar{\pi}}(S)}{d^{\bar{\pi}}(s)}\mathfrak{D}_{\bar{\pi}} (\tilde{\pi} \mid s) \biggr) \\
		&= d^{\bar{\pi}}(s)\biggl( \mathbb{E}_{A \sim \hat{\pi}} \bigl[A_{\bar{\pi}}(s,A)  \bigr] - \frac{\nu^{\hat{\pi}}_{\bar{\pi}}(s)}{d^{\bar{\pi}}(s)}\mathfrak{D}_{\bar{\pi}} (\hat{\pi} \mid s) \biggr) \\
		&> 0,
	\end{align*}
	which contradicts Equation \eqref{eq:mirror_s4_contr}, so the assumption \eqref{eq:mirror_s4_assumption} must be wrong, proving
	\[ \bar{\pi} = \argmax_{\pi \in \Pi} \mathbb{E}_{A \sim \pi} \bigl[ A_{\bar{\pi}}(s,A) \bigr] = \argmax_{\pi \in \Pi} \mathbb{E}_{A \sim \pi} \bigl[ Q_{\bar{\pi}}(s,A) \bigr]. \]
	
	\textbf{Step 5}
	
	The main result \eqref{eq:mirror_gpi_policy} from step 4 shows that any limit point \(\bar{\pi}\) of \(\bigl(\pi_n \bigr)_{n \in \mathbb{N}}\) is also a fixed point of GPI. Thus, as corollaries all properties induced by GPI (see Theorem \ref{th:gpi}) apply to \(\bar{\pi} \in L\Pi\). Particularly, we have the optimality of \(\bar{\pi}\), the value function optimality \(V = V_{\bar{\pi}} = V^*\) and thereby also the maximality of returns as 
	\[\lim_{n \to \infty} J(\pi_n) = \lim_{n \to \infty} \mathbb{E}_{S \sim p_0} \bigl[ V_{\pi_n}(S) \bigr] = \mathbb{E}_{S \sim p_0} \bigl[ V^*(S) \bigr] = \max_{\pi \in \Pi} J(\pi).\]
	Thus, we have shown all properties as claimed by Theorem \ref{th:mirror_learning}.
\end{proof}

We close this section with some remarks. In practice, exact updates according to the mirror learning update rule \eqref{eq:mirror_update} are generally infeasible. Instead, we can sample the expectation to obtain batch estimators over a batch \(\mathcal{D}\) of transitions
\[ \frac{1}{\lvert \mathcal{D} \rvert} \sum_{s,a \in \mathcal{D}} \Bigl(Q_{\pi_\text{old}}(s,a) - \frac{\nu^{{\pi_\text{new}}}_{\pi_\text{old}}}{d^{\pi_\text{old}}} \mathfrak{D}_{\pi_\text{old}}(\pi_\text{new} \mid s) \Bigr), \]
where \(Q_{\pi_\text{old}}\) has to be estimated as well. These batch estimators can also only be approximately optimized each iteration via gradient ascent to update the policy. Given these approximations and the at-best local convergence of gradient ascent, the outlaid convergence properties remain theoretical.


%% file: experiments.tex
\section{Numerical Experiments}\label{sec:experiments}

Now, we empirically compare the discussed policy gradient algorithms. Consistent with the original works \cite{mnih2016asynchronous, schulman2015trust, schulman2017proximal, song2019v}, we compare them on the established MuJoCo task suite \cite{todorov2012mujoco}, accessed through the Gymnasium library \cite{towers_gymnasium_2023}. MuJoCo features robotics simulations, where the tasks are to control and move robots of different shapes by applying torques to each joint. 

Our implementations build on the PPO implementation from the BRAX library \cite{brax2021github} and are written in JAX \cite{jax2018github}. For enhanced comparability, all algorithms that estimate advantages use GAE similarly to PPO. Instead of A3C, we use its synchronous variant A2C due to its simpler implementation. Note that A2C exhibits comparable performance as A3C \cite{weng2018PG} and only differs in that it waits for all actors to collect transitions to update them synchronously. We modify REINFORCE to average gradients over batches of transitions similarly as in the other algorithms since computing one update per environment step is computationally very costly. Note that this is however likely to improve the performance compared to a naive implementation of REINFORCE. We do not tune hyperparameters and keep choices consistent across algorithms where possible. See Appendix \autoref{sec:hyperparameters} for the hyperparameters we use. The experiments were run on a standard consumer CPU. All our implemented algorithms and the code for running the experiments can be found at \url{https://github.com/Matt00n/PolicyGradientsJax}.

\begin{figure}
     \centering
     \begin{subfigure}[b]{0.48\textwidth}
         \centering
         \includegraphics[width=\textwidth]{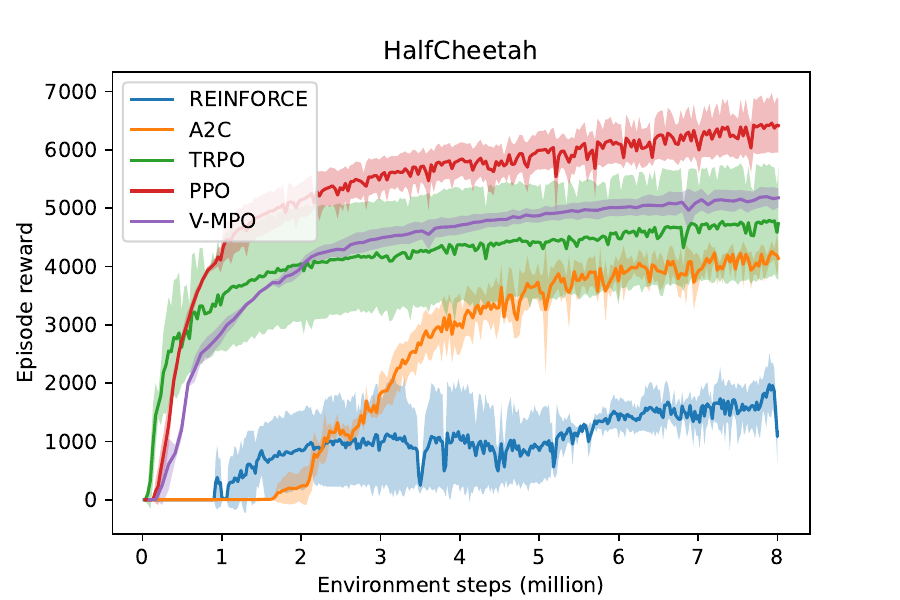}
     \end{subfigure}
     \hfill
     \begin{subfigure}[b]{0.48\textwidth}
         \centering
         \includegraphics[width=\textwidth]{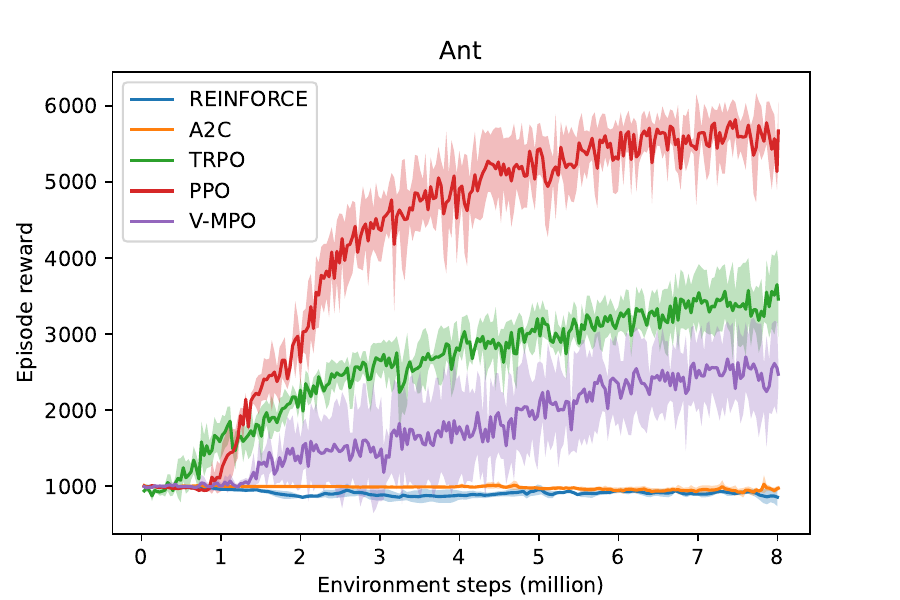}
     \end{subfigure}
     \hfill
     \begin{subfigure}[b]{0.48\textwidth}
         \centering
         \includegraphics[width=\textwidth]{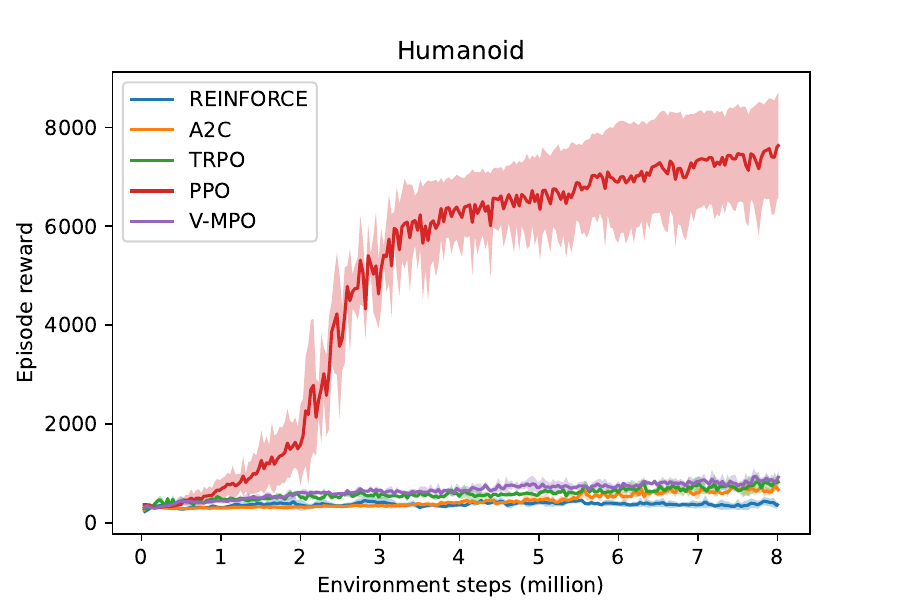}
     \end{subfigure}
     \hfill
     \begin{subfigure}[b]{0.48\textwidth}
         \centering
         \includegraphics[width=\textwidth]{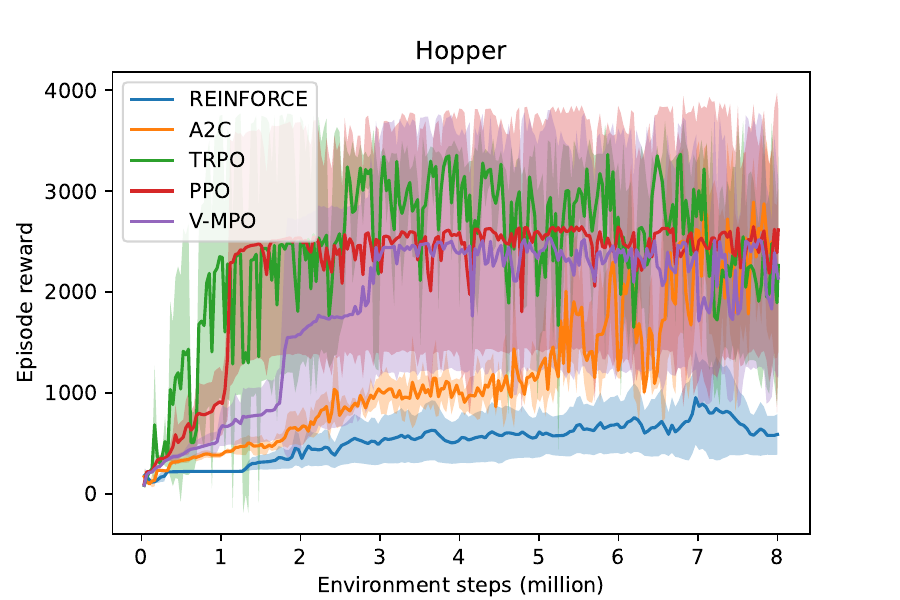}
     \end{subfigure}
        \caption{Comparison of rewards per episode during training on several MuJoCo tasks. For each algorithm, we report means and standard deviations of three runs with different random seeds.}
        \label{fig:perf_comp}
\end{figure}

In our main experiment, we compare the performance of the algorithms in terms of the achieved episodic rewards over the course of training. The performances in different MuJoCo tasks are presented in \autoref{fig:perf_comp}. We observe that PPO outperforms the other algorithms in three of four tasks by achieving higher episodic rewards while learning good policies quickly. The performance difference is most prevalent on the \emph{Humanoid}-task, the most challenging of the four, where PPO learns much stronger policies than the other algorithms. In addition, we find our implementation of PPO to be competitive with common RL libraries as shown in Appendix \ref{sec:framework_comp}.
V-MPO and TRPO are comparable in performance, with each of the two slightly outperforming the other on two out of four environments. We note that V-MPO is intended for training for billions of environment steps, such that its lower performance compared to PPO in our experiments is expected\footnote{Also see the discussions at https://openreview.net/forum?id=SylOlp4FvH on this.} \cite{song2019v}. A2C requires more interactions with the environment to reach similar performance levels as V-MPO and TRPO but fails to learn any useful policy in the \emph{Ant}-task. This slower learning\footnote{Slow in terms of the required environment steps. Note however that A2C runs significantly faster than PPO, TRPO and V-MPO in absolute time due to using less epochs per batch.} is at least partially caused by A2C only using a single update epoch per batch. REINFORCE performance worst on all environments, which is unsurprising giving the high variance of gradients in REINFORCE \cite{Sutton1998}. This also highlights the benefits of the bias-variance trade-off by the other algorithms as discussed in Section \ref{sec:comparison}. We find our performance-based ranking of the algorithms to be consistent with literature (e.g., \cite{schulman2017proximal, song2019v, andrychowicz2020matters}). 

Moreover, we remark that A2C is the only algorithm for which we used an entropy bonus because the learned policies collapsed without it. We showcase this in our expended experiments in Appendix \ref{sec:entropy_a2c}. This underlines the usefulness of the (heuristic) constraints of V-MPO, PPO and TRPO on the KL divergence, which avoid such collapses even without any entropy bonuses. To further investigate this, we show the average KL divergences between consecutive policies throughout training in \autoref{fig:kl_comp}. Here, we approximated the KL divergence using the unbiased estimator \cite{schulman2020approximating}
\begin{equation*}
	\hat{D}_{KL}\bigl(\pi_\text{old}(\cdot \mid s) \: \Vert \: \pi_\text{new}(\cdot \mid s)\bigr) = \mathbb{E}_{A \sim \pi_\text{old}}\biggl[\frac{\pi_\text{new}(A \mid s)}{\pi_\text{old}(A \mid s)} - 1 - \ln \frac{\pi_\text{new}(A \mid s)}{\pi_\text{old}(A \mid s)}\biggr]
\end{equation*}
for all algorithms except TRPO, which analytically calculates the exact KL divergence since it is used within the algorithm. We see that the KL divergences remain relatively constant for all algorithms after some initial movement. TRPO displays the most constant KL divergence, which is explained by its hard constraint. With the chosen hyperparameters, V-MPO uses the same bound on the KL divergence as TRPO, however without strictly enforcing it as outlined in the derivation of V-MPO. Thus, V-MPO's KL divergence exhibits slightly more variance then TRPO and also frequently exceeds this bound. PPO's clipping heuristic achieves a similar effect resulting in a comparable picture. Due to the lack of constraints on the KL divergence, A2C and REINFORCE show slightly more variance. Interestingly, their KL divergences are orders of magnitudes lower than for the other algorithms, especially for REINFORCE (note the logarithmic scale in \autoref{fig:kl_comp}). We reason this with A2C and REINFORCE using only a singly update epoch per batch, whereas the PPO and V-MPO use multiple epochs and TRPO uses a different update scheme via line search. In Appendix \ref{sec:epochs_a2c}, we provide experimental evidence for this hypothesis.  Additionally, we note again that the entropy bonus also stabilizes and limits the KL divergence for A2C as shown in Appendix \ref{sec:entropy_a2c}. 

\begin{figure}
     \centering
     \begin{subfigure}[b]{0.48\textwidth}
         \centering
         \includegraphics[width=\textwidth]{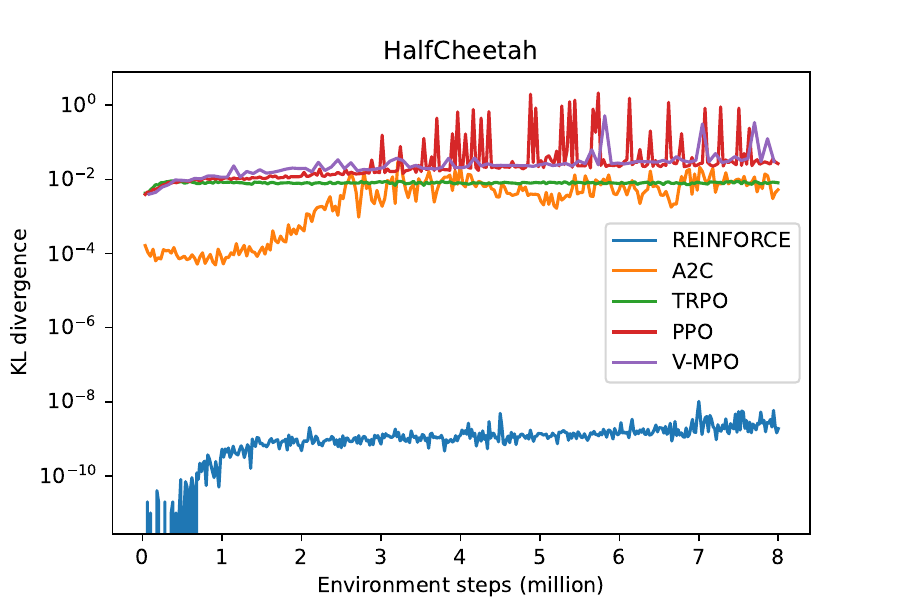}
     \end{subfigure}
     \hfill
     \begin{subfigure}[b]{0.48\textwidth}
         \centering
         \includegraphics[width=\textwidth]{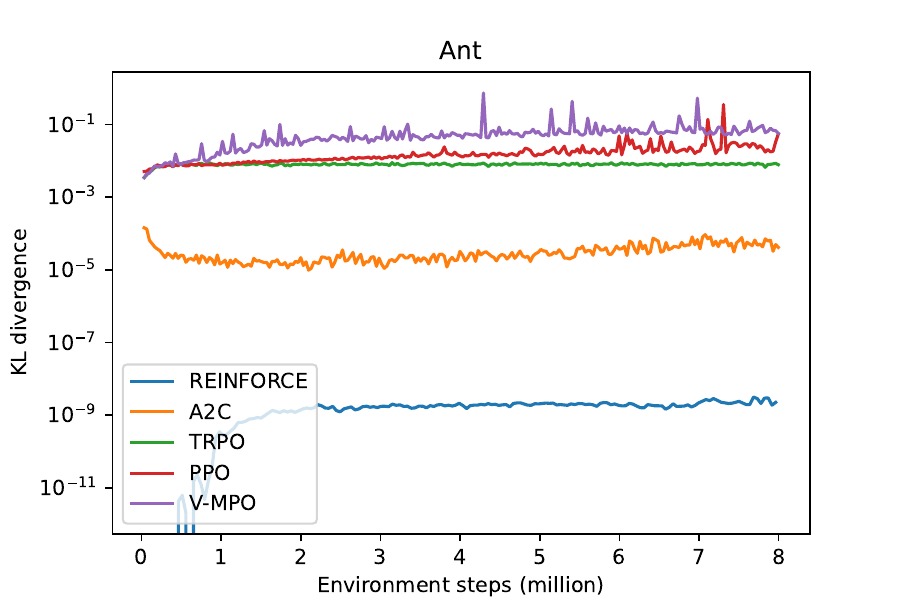}
     \end{subfigure}
     \hfill
     \begin{subfigure}[b]{0.48\textwidth}
         \centering
         \includegraphics[width=\textwidth]{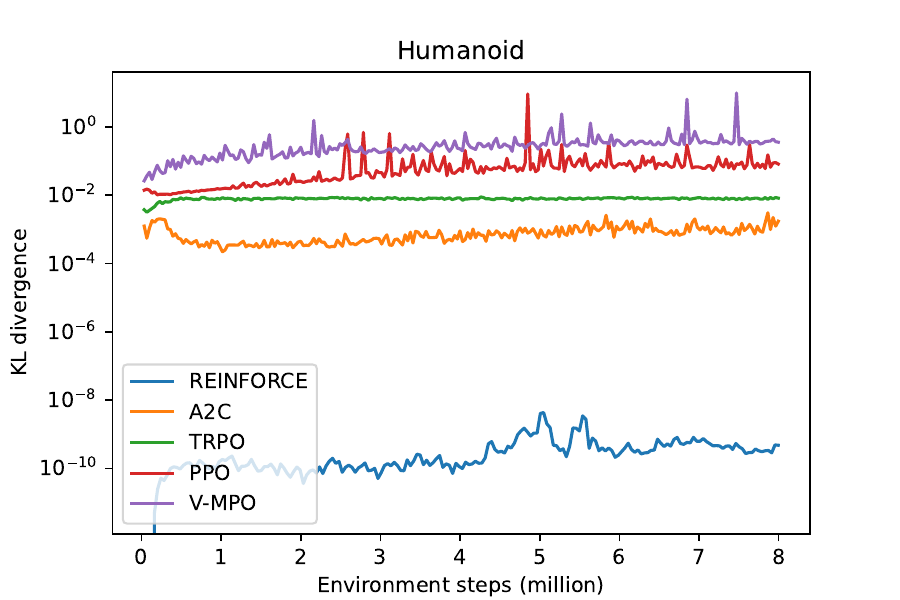}
     \end{subfigure}
     \hfill
     \begin{subfigure}[b]{0.48\textwidth}
         \centering
         \includegraphics[width=\textwidth]{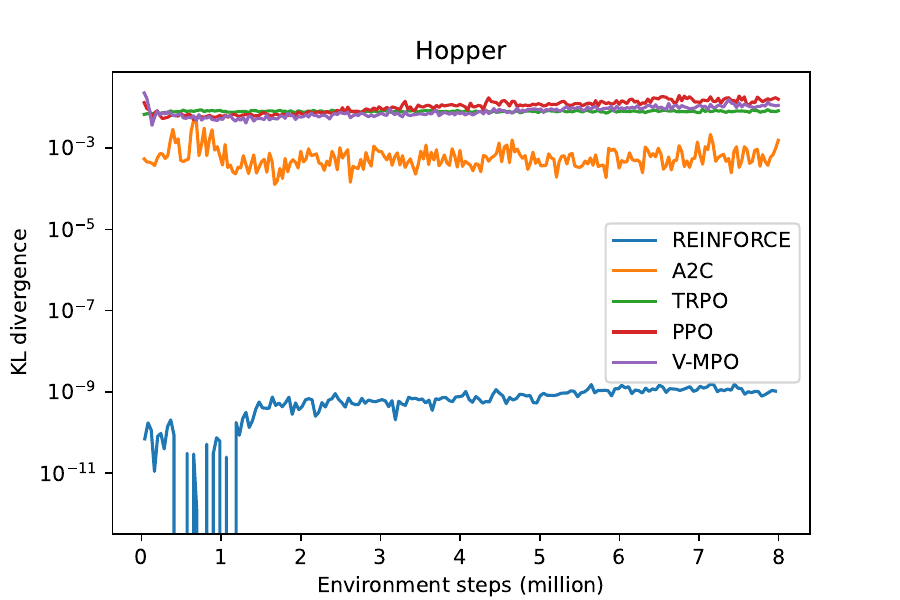}
     \end{subfigure}
        \caption{Comparison of  the average KL divergence across policies during training.}
        \label{fig:kl_comp}
\end{figure}

These findings highlight the benefits of regularization through constraining the KL divergence and incentivizing entropy. Regularization stabilizes learning and prevents a collapse of the policy. At the same time, it allows more frequent updates through multiple epochs per batch, which drastically increases the sample efficiency of the algorithms and speeds up learning.


%% file: conclusion.tex
\section{Conclusion} \label{sec:conclusion}

In this work, we presented a holistic overview of on-policy policy gradient methods in reinforcement learning. We derived the theoretical foundations of policy gradient algorithms, primarily in the form of the Policy Gradient Theorem. We have shown how the most prominent policy gradient algorithms can be derived based on this theorem. We discussed common techniques used by these algorithms to stabilize training including learning an advantage function to limit the variance of estimated policy gradients, constraining the divergence between policies and regularizing the policy through entropy bonuses. Subsequently, we presented evidence from literature on the convergence behavior of policy gradient algorithms, which suggest that they may find at least locally optimal policies. Finally, we conducted numerical experiments on well-established benchmarks to further compare the behavior of the discussed algorithms. Here, we found that PPO outperforms the other algorithms in the majority of the considered tasks and we provided evidence for the necessity of regularization, by constraining KL divergence or by incentivizing entropy, to stabilize training.

We acknowledge several limitations of our work. First, we deliberately limited our scope to on-policy algorithms, which excludes closely related off-policy policy gradient algorithms and the novelties introduced by them. Second, we presented an incomplete overview of on-policy policy gradient algorithms as other, albeit less established, algorithms exist (e.g., \cite{rahman2022robust, cobbe2021phasic}) and the development of further algorithms remains an active research field. Here, we focused on the, in our view, most prominent algorithms as determined by their impact, usage and introduced novelties. Third, the convergence results we referenced rest on assumptions that are quickly violated in practice. In particular, we want to underline that the results based mirror learning rely on the infeasible assumption of finding a global maximizer each iteration. Fourth, while we compared the discussed algorithms empirically and found results to be consistent with existing literature, our analysis is limited to the specific setting we used. Different results may arise on other benchmarks, with different hyperparameters or generally different implementations. 

Finally, we note that still many questions remain to be answered in the field of on-policy policy gradient algorithm. So far, our understanding of which algorithm performs best under which circumstances is still limited. Moreover, it is unclear whether the best possible policy gradient algorithm has yet been discovered, which is why algorithm development remains of interest. Similarly, comprehensive empirical comparisons with other classes of RL algorithms may yield further insights on the practical advantages and disadvantages of policy gradient algorithms and how their performance depends on the problem settings. Finally, we observe that still only a limited number of convergence results exist and not even all discussed algorithms are covered by these, e.g., no convergence results exist for V-MPO to the best of our knowledge. Here, further research is needed to enhance our understanding of the convergence behavior of policy gradient algorithms.


%% file: hyperparameters.tex
\section{Hyperparameters}\label{sec:hyperparameters}

\begin{table}
	\centering
	\begin{tabular}{cccccc}
		\toprule
		& \multicolumn{5}{c}{Value}                   \\
		\cmidrule(r){2-6}
		Hyperparameter     & REINFORCE & A2C & TRPO & PPO & V-MPO \\
		\midrule
		Learning rate     & $3 \cdot 10^{-4}$ & $3 \cdot 10^{-4}$ & $3 \cdot 10^{-4}$ & $3 \cdot 10^{-4}$ & $3 \cdot 10^{-4}$ \\
		Num. minibatches    & $1$ & $8$ & $8$ & $8$ & $8$ \\
		Num. epochs     & $1$ & $1$ & $1$\footnotemark & $10$ & $10$ \\
		Discount (\(\gamma\))     & --- & $0.99$ & $0.99$ & $0.99$ & $0.99$ \\
		GAE parameter (\(\lambda\))     & --- & $0.95$ & $0.95$ & $0.95$ & $0.95$ \\
		Normalize advantages     & --- & True & True & True & False \\
		Entropy bonus coef.     & $0$ & $0.1$ & $0$ & $0$ & $0$ \\
		Max. grad. norm     & $0.5$ & $0.5$ & $0.5$ & $0.5$ & $0.5$ \\
		Unroll length     & --- & $2048$ & $2048$ & $2048$ & $2048$ \\
		KL target (\(\delta\))     & --- & --- & $0.01$ & --- & --- \\
		CG damping     & --- & --- & $0.1$ & --- & --- \\
		CG max. iterations     & --- & --- & $10$ & --- & --- \\
		Line search max. iterations     & --- & --- & $10$ & --- & --- \\
		Line search shrinkage factor     & --- & --- & $0.8$ & --- & --- \\
		PPO clipping (\(\varepsilon\))     & --- & --- & --- & $0.2$ & --- \\
		Min. temp. (\(\eta_\text{min}\))     & --- & --- & --- & --- & $10^{-8}$ \\
		Min. KL pen. (\(\nu_\text{min}\))     & --- & --- & --- & --- & $10^{-8}$ \\
		Init. temp. (\(\eta_\text{init}\))      & --- & --- & --- & --- & $1$ \\
		Init. KL pen. (mean) (\(\nu_{\mu_\text{init}}\))      & --- & --- & --- & --- & $1$ \\
		Init. KL pen. (std) (\(\nu_{\sigma_\text{init}}\))       & --- & --- & --- & --- & $1$ \\
		KL target (mean) (\(\varepsilon_{\nu_\mu}\))      & --- & --- & --- & --- & $0.01$ \\
		KL target (std) (\(\varepsilon_{\nu_\sigma}\))      & --- & --- & --- & --- & $5 \cdot 10^{-5}$ \\
		KL target (temp.) (\(\varepsilon_\eta\))      & --- & --- & --- & --- & $0.01$ \\
		\bottomrule
	\end{tabular}
	\caption{Algorithm hyperparameters.}
	\label{table:hyperparameters}
\end{table}
\footnotetext{TRPO uses one epoch for its policy updates but 10 epochs per batch for updating the value network.}

We report the hyperparameters we use in our main experiments in \autoref{table:hyperparameters}. All algorithms use separate policy and value networks. Policy networks use 4 hidden layers with 32 neurons respectively. Value networks use 5 layers with 256 neurons each. We use swish-activation functions \cite{ramachandran2017searching} throughout both networks. Policy outputs are transformed to fit the bounds of the actions spaces via a squashing function. We use the Adam optimizer \cite{kingma2014adam} with gradient clipping and a slight linear decay of the learning rates. Further, we preprocess observations and rewards by normalizing them using running means and standard deviations and clipping them to the interval \([-10, 10]\). All algorithms except REINFORCE use 8 parallel environments to collect experience. We use independent environments to evaluate the agents throughout training. In the evaluations, agents select actions deterministically as the mode of the constructed distribution.

%% file: extendend_experiments.tex
\section{Extended Experiments}\label{sec:ext_exp}

Here, we present results from further experiments. Unless indicated otherwise, we use the hyperparameters as reported in Appendix \autoref{sec:hyperparameters}.

\subsection{Comparison to RL frameworks}\label{sec:framework_comp}

In \autoref{table:framework_comp}, we compare the performance of our implementation of PPO with popular RL frameworks. Note that we did not tune any hyperparameters for our implementations such that the reported scores should be understood as lower bounds. We compare PPO since it is the most popular and commonly implemented of the discussed algorithms across frameworks. In contrast, especially TRPO and V-MPO are rarely found. 

\begin{table}
	\centering
	\begin{tabular}{cccccc|cc}
		\toprule
		& \multicolumn{7}{c}{Framework}                   \\
		\cmidrule(r){2-8}
		     & CleanRL & Baselines  & SB3 & RLlib & ACME\footnotemark & Ours & Ours \\
		 & \cite{huang2022cleanrl} & \cite{baselines} & \cite{stable-baselines3} & \cite{liang2018rllib} & \cite{hoffman2020acme} & & \\
		\midrule
		MuJoCo version     & v4 & v1 & v3 & v2 & v2 & v4 & v4 \\
		Steps in million    & $1$ & $1$ & $1$ & $44$ & $10$ & $1$ & $8$ \\
		\midrule
		HalfCheetah    & $2906$ & $1669$ & $5819$ & $9664$ & $6800$ & $4332$ & $6414$ \\
		Hopper     & $2052$ & $2316$ & $2410$ & --- & $2550$ & $895$ & $2616$ \\
		Humanoid     & $742$ & --- & --- & --- & $6600$ & $700$ & $7633$ \\
		Ant     & --- & --- & $1327$ & --- & $5200$ & $1258$ & $5671$ \\
		\bottomrule
	\end{tabular}
	\caption{Comparison of the mean performance of our PPO implementation with popular RL frameworks. Scores for the frameworks are shown as reported in the respective paper or documentation. }
	\label{table:framework_comp}
\end{table}
\footnotetext{Numbers read approximately from plots in the paper.}

\subsection{Entropy Bonus in A2C}\label{sec:entropy_a2c}

In \autoref{fig:a2c_coll}, we show that using an entropy bonus improves the performance of A2C by stabilizing learning. In particular, insufficiently low values of the entropy coefficient result in a collapse of the policy after some time. This is visible in a drastic increase in the KL divergences (note the logarithmic scale).

\begin{figure}
     \centering
     \begin{subfigure}[b]{0.49\textwidth}
         \centering
         \includegraphics[width=\textwidth]{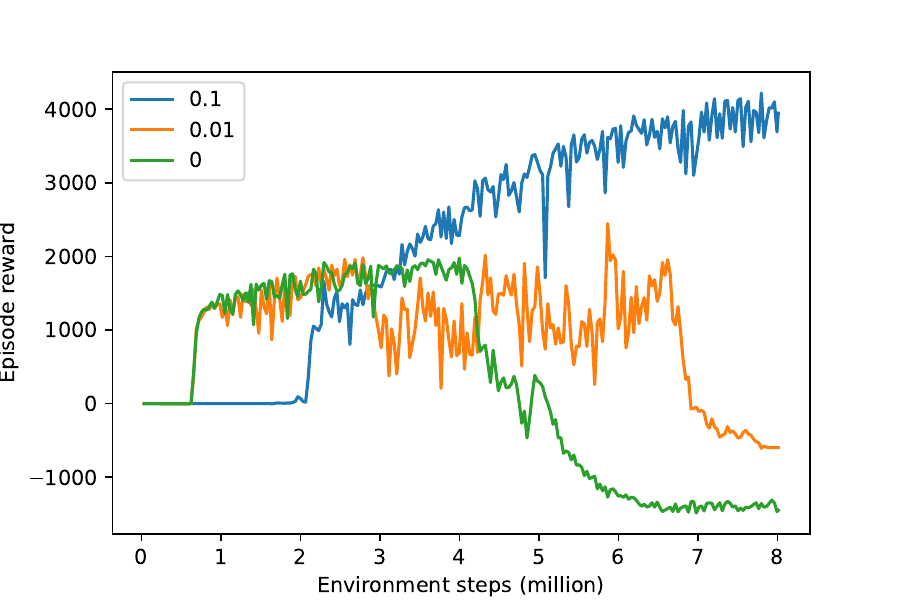}
     \end{subfigure}
     \hfill
     \begin{subfigure}[b]{0.49\textwidth}
         \centering
         \includegraphics[width=\textwidth]{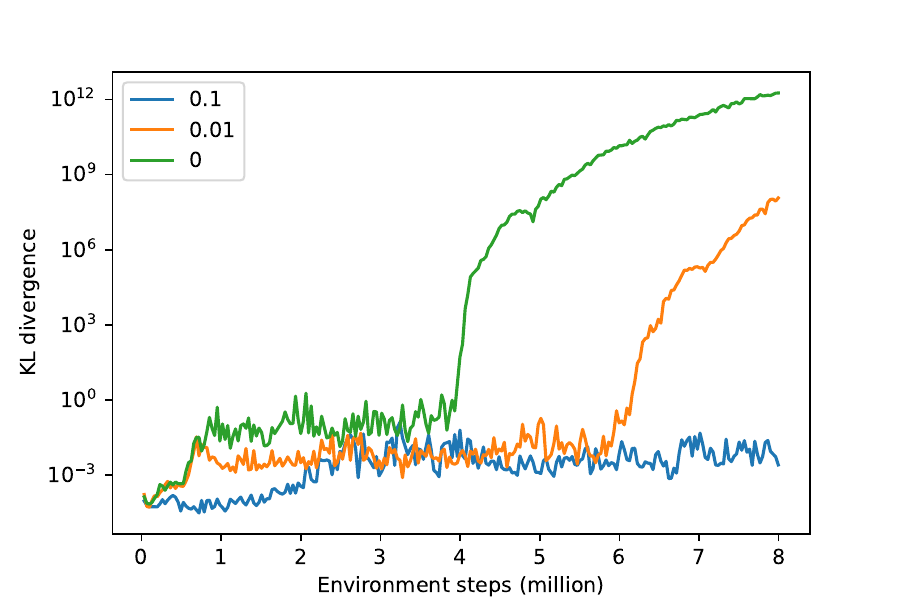}
     \end{subfigure}
        \caption{We compare the episode reward (left) and KL divergence (right) for different values of the entropy coefficient for A2C on HalfCheetah.}
        \label{fig:a2c_coll}
\end{figure}

\subsection{A2C and REINFORCE with Multiple Update Epochs}\label{sec:epochs_a2c}

In \autoref{fig:a2c_epochs}, we showcase that the KL divergence is low for A2C and REINFORCE due to using only a single update epoch per batch. On the contrary, when using multiple epochs, the policies collapse for both algorithms as visible by the diverging KL divergence and abrupt performance loss. Note, that here we show this behavior for five epochs, however in our tests A2C and REINFORCE display similar behaviors already when only using two epochs, albeit the policies then only collapse after an extended period of time. Further, note that over the displayed range of environment steps, the algorithms do not yet learn any useful policies when using a single epoch. However, performance improves for both A2C and REINFORCE when given more time as depicted in \autoref{fig:perf_comp}. 

\begin{figure}
     \centering
     \begin{subfigure}[b]{0.49\textwidth}
         \centering
         \includegraphics[width=\textwidth]{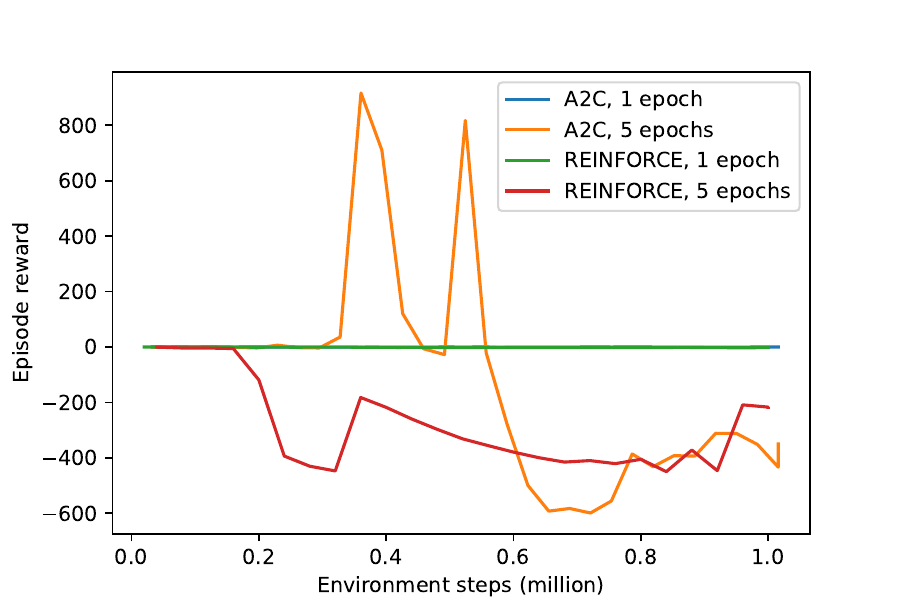}
     \end{subfigure}
     \hfill
     \begin{subfigure}[b]{0.49\textwidth}
         \centering
         \includegraphics[width=\textwidth]{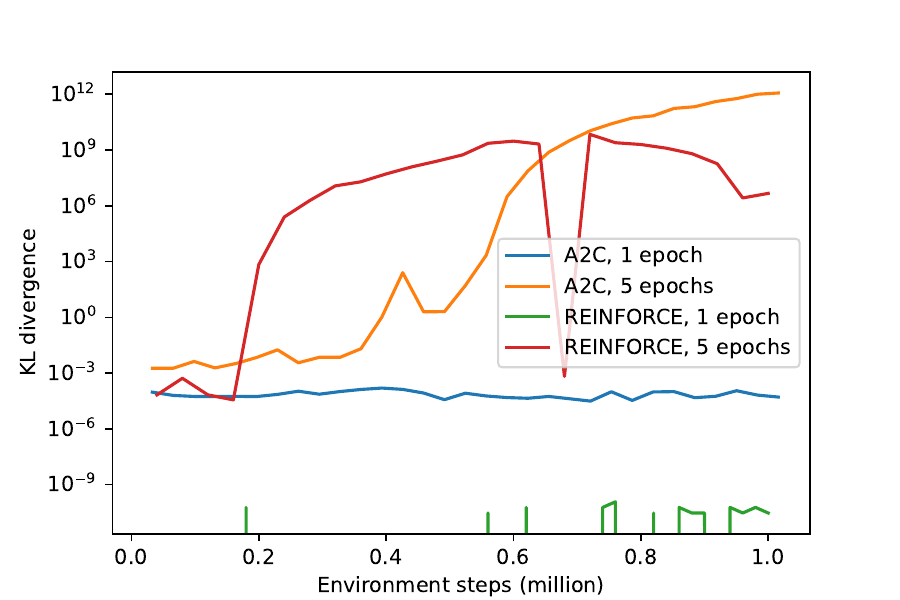}
     \end{subfigure}
     \caption{We compare the episode reward (left) and KL divergence (right) for different numbers of update epochs for A2C and REINFORCE on HalfCheetah.}
     \label{fig:a2c_epochs}
\end{figure}

%% file: vmpo_derivation.tex
\section{V-MPO: Derivation Details}\label{sec:mpo_details}

In the following, we provide a more detailed derivation of the objective function of V-MPO
\begin{equation*}
	J_\text{V-MPO}(\theta, \eta, \nu) = \mathcal{L}_\pi(\theta) + \mathcal{L}_\eta(\eta) + \mathcal{L}_\nu(\theta, \nu),
\end{equation*}
where \(\mathcal{L}_\pi\) is the policy loss
\begin{equation}
	\mathcal{L}_\pi(\theta) = - \sum_{a,s \in \tilde{\mathcal{D}}} \frac{\exp\Bigl( \frac{\hat{A}_{\phi}(s,a)}{\eta} \Bigr)}{\sum_{a',s' \in \tilde{\mathcal{D}}} \exp\Bigl( \frac{\hat{A}_{\phi}(s',a')}{\eta} \Bigr)} \ln \pi_{\theta}(a \mid s), \label{eq:app_mpo_policy_loss} 
\end{equation}
\(\mathcal{L}_\eta\) is the temperature loss
\begin{equation}
	\mathcal{L}_\eta(\eta) = \eta \varepsilon_\eta + \eta \ln \Biggl[ \frac{1}{\lvert \tilde{\mathcal{D}} \rvert} \sum_{a,s \in \tilde{\mathcal{D}}} \exp\biggl( \frac{\hat{A}_{\phi}(s,a)}{\eta} \biggr) \Biggr] \label{eq:app_mpo_temp_loss}
\end{equation}
and \(\mathcal{L}_\nu\) is the trust-region loss
\begin{align}
\begin{split}
	\mathcal{L}_\nu(\theta, \nu) = \frac{1}{\lvert \mathcal{D} \rvert} \sum_{s \in \mathcal{D}} \biggl( \nu \biggl( \varepsilon_\nu - \mathrm{sg\Bigl[\Bigl[ D_{KL}(\pi_{\text{old}}(\cdot \mid s) \: \Vert \: \pi_{\theta}(\cdot \mid s))  \Bigr]\Bigr]} \biggr) + \mathrm{sg}\bigl[\bigl[ \nu \bigr]\bigr] D_{KL}\bigl( \pi_{\text{old}}(\cdot \mid s) \: \Vert \: \pi_{\theta}(\cdot \mid s) \bigr) \biggr). \label{eq:app_mpo_kl_loss}
\end{split}
\end{align}

Let \(p_\theta(s,a) = \pi_\theta(a \mid s) d^{\pi_\theta}(s)\) denote the joint state-action distribution under policy \(\pi_\theta\) conditional on the parameters \(\theta\). Let \(\mathcal{I}\) be a binary random variable whether the updated policy \(\pi_\theta\) is an improvement over the old policy \(\pi_{\text{old}}\), i.e. \(\mathcal{I}=1\) if it is an improvement. We assume the probability of \(\pi_\theta\) being an improvement is proportional to the following expression 
\begin{equation}
	p_\theta(\mathcal{I} = 1 \mid s,a) \propto \exp\Bigl( \frac{A_{\pi_\text{old}}(s,a)}{\eta} \Bigr) \label{eq:app_mpo_0}
\end{equation}
Given the desired outcome \(\mathcal{I}=1\), we seek the posterior distribution conditioned on this event. Specifically, we seek the maximum a posteriori estimate 
\begin{align}
\begin{split}
	\theta^* &= \argmax_\theta \bigl[ p_\theta(\mathcal{I}=1) \rho(\theta) \bigr] \\
	&= \argmax_\theta \bigl[ \ln p_\theta(\mathcal{I}=1) + \ln \rho(\theta) \bigr], \label{eq:app_mpo_2}
\end{split}
\end{align}  
where \(\rho\) is some prior distribution to be specified. Using Theorem \ref{th:log_kl}, we obtain
\begin{equation}
	\ln p_\theta(\mathcal{I}=1) = \mathbb{E}_{S,A \sim \psi} \biggl[\ln \frac{p_\theta(\mathcal{I}=1,S,A)}{\psi(S,A)} \biggr] + D_{KL} \bigl(\psi \: \Vert \: p_\theta(\cdot,\cdot \mid \mathcal{I}=1) \bigr), \label{eq:app_mpo_1}
\end{equation}
where \(\psi\) is a distribution over \(\mathcal{S} \times \mathcal{A}\). Observe that, since the KL-divergence is non-negative, the first term is a lower bound for \(\ln p_\theta(\mathcal{I}=1)\). Akin to EM algorithms, V-MPO now iterates between an expectation (E) and a maximization (M) step. In the E-step we choose the variational distribution \(\psi\) to minimize the KL divergence in Equation \eqref{eq:app_mpo_1} to make the lower bound as tight as possible. In the M-step, we maximize this lower bound and the prior \(\ln \rho(\theta)\) to obtain a new estimate of \(\theta^*\) via Equation \eqref{eq:app_mpo_2}.

First, we consider the E-step. Minimizing \(D_{KL} (\psi \: \Vert \: p_{\theta_\text{old}}(\cdot, \cdot \mid \mathcal{I}=1))\) w.r.t. \(\psi\) leads to 
\begin{align*}
	\psi(s,a) &= p_{\theta_\text{old}}(s,a \mid \mathcal{I}=1) \\
	&= \frac{p_{\theta_\text{old}}(s,a) \: p_{\theta_\text{old}}(\mathcal{I}=1 \mid s,a)}{p_{\theta_\text{old}}(\mathcal{I}=1)} \\
	&= \frac{p_{\theta_\text{old}}(s,a) \: p_{\theta_\text{old}}(\mathcal{I}=1 \mid s,a)}{\int_{s \in \mathcal{S}} \int_{a \in \mathcal{A}} p_{\theta_\text{old}}(s,a) \: p_{\theta_\text{old}}(\mathcal{I}=1 \mid s,a) \: da \: ds }
\end{align*} 
using Bayes' Theorem (Theorem \ref{th:bayes}). Sampling from right-hand side of \eqref{eq:app_mpo_0} thus yields 
\[\hat{\psi}(s,a) = \frac{\exp\Bigl( \frac{A_{\pi_{\text{old}}}(s,a)}{\eta} \Bigr)}{\sum_{a,s \in \mathcal{D}} \exp\Bigl( \frac{A_{\pi_{\text{old}}}(s,a)}{\eta} \Bigr)},\]
which is the variational distribution found in the policy loss \eqref{eq:app_mpo_temp_loss}. \cite{song2019v} find that using only the highest 50 \% of advantages per batch, i.e. replacing \(\mathcal{D}\) with \(\tilde{\mathcal{D}}\), substantially improves the algorithm. The advantage function \(A_{\pi}\) is estimated by \(\hat{A}_\phi\), which is learned identically as in A3C.

We derive the temperature loss to automatically adjust the temperature \(\eta\) by applying \eqref{eq:app_mpo_0} to the KL term in \eqref{eq:app_mpo_1}, which we want to minimize:
\begin{align*}
	D_{KL} \Bigl(\psi \: \Vert \: p(\cdot, \cdot \mid \mathcal{I}=1)\Bigr) &= D_{KL} \biggl(\psi \: \Vert \: \frac{p_{\theta_\text{old}}(S,A) p_{\theta_\text{old}}(\mathcal{I}=1 \mid S,A)}{p_{\theta_\text{old}}(\mathcal{I}=1)}\biggr) \\
	&= D_{KL} \Biggl(\psi \: \Vert \: \frac{p_{\theta_\text{old}}(S,A) \exp\Bigl( \frac{A_{\pi_\text{old}}(S,A)}{\eta} \Bigr)}{p_{\theta_\text{old}}(\mathcal{I}=1)} \Biggr) \\
	&= - \int_{s \in \mathcal{S}} \int_{a \in \mathcal{A}} \psi(s,a) \ln \Biggl( \frac{p_{\theta_\text{old}}(s,a) \exp\Bigl( \frac{A_{\pi_\text{old}}(s,a)}{\eta} \Bigr)}{\psi(s,a) p_{\theta_\text{old}}(\mathcal{I}=1)} \Biggr) \: da \: ds 
\end{align*}
By applying the logarithm to the individual terms, rearranging and multiplying through by \(\eta\) we get
\begin{align*}
	D_{KL} \Bigl(\psi \: \Vert \: p(\cdot, \cdot \mid \mathcal{I}=1)\Bigr) 
	&= - \int_{s \in \mathcal{S}} \int_{a \in \mathcal{A}} \psi(s,a) \biggl( 
	\frac{A_{\pi_\text{old}}(s,a)}{\eta} + \ln p_{\theta_\text{old}}(s,a) \\
	&\qquad - \ln p_{\theta_\text{old}}(\mathcal{I}=1) - \ln \psi(s,a) \biggr) \: da \: ds  \\
	&\propto - \int_{s \in \mathcal{S}} \int_{a \in \mathcal{A}} \psi(s,a) \biggl( 
	A_{\pi_\text{old}}(s,a) + \eta \ln p_{\theta_\text{old}}(s,a) - \eta \ln p_{\theta_\text{old}}(\mathcal{I}=1) \\
	&\qquad - \eta \ln \psi(s,a) \biggr) \: da \: ds  \\
	&= - \int_{s \in \mathcal{S}} \int_{a \in \mathcal{A}} \psi(s,a) 
	A_{\pi_\text{old}}(s,a) \: da \: ds  + \eta \int_{s \in \mathcal{S}} \int_{a \in \mathcal{A}} \psi(s,a) \ln \frac{\psi(s,a)}{p_{\theta_\text{old}}(s,a)} \: da \: ds \\
	&\qquad + \lambda \int_{s \in \mathcal{S}} \int_{a \in \mathcal{A}} \psi(s,a) \: da \: ds 
\end{align*}
with \(\lambda = \eta \ln p_{\theta_\text{old}}(\mathcal{I}=1) \). To optimize \(\eta\) while minimizing the KL term, we transform this into a constrained optimization problem with a bound on the KL divergence 
\begin{align*}
	\argmax_{\psi} &\quad\int_{s \in \mathcal{S}} \int_{a \in \mathcal{A}} \psi(s,a) 
	A_{\pi_\text{old}}(s,a) \: da \: ds \\
	\text{subject to } &\quad \int_{s \in \mathcal{S}} \int_{a \in \mathcal{A}} \psi(s,a) \ln \frac{\psi(s,a)}{p_{\theta_\text{old}}(s,a)} \: da \: ds \leq \varepsilon_\eta, \\
	&\quad \int_{s \in \mathcal{S}} \int_{a \in \mathcal{A}} \psi(s,a) \: da \: ds = 1
\end{align*}
and then back into an unconstrained problem via Lagrangian relaxation, yielding the objective function
\begin{align*}
	\mathcal{J}(\psi, \eta, \lambda) &= \int_{s \in \mathcal{S}} \int_{a \in \mathcal{A}} \psi(s,a) 
	A_{\pi_\text{old}}(s,a) \: da \: ds + \eta \biggl( \varepsilon_\eta \\
	&\quad - \int_{s \in \mathcal{S}} \int_{a \in \mathcal{A}} \psi(s,a) \ln \frac{\psi(s,a)}{p_{\theta_\text{old}}(s,a)} \: da \: ds \biggr) + \lambda \biggl( 1 - \int_{s \in \mathcal{S}} \int_{a \in \mathcal{A}} \psi(s,a) \: da \: ds \biggr).
\end{align*}
Differentiating w.r.t. \(\psi(s,a)\) and setting to zero yields 
\[ \psi(s,a) = p_{\theta_\text{old}}(s,a) \exp \biggl( \frac{A_{\pi_\text{old}}(s,a)}{\eta} \biggr) \exp \biggl( -1 - \frac{\lambda}{\eta} \biggr) \]
Normalizing over \(s\) and \(a\) confirms the already attained solution 
\begin{equation}
	\psi(s,a) = \frac{p_{\theta_\text{old}}(s,a) \exp\bigl( \frac{A_{\pi_\text{old}}(s,a)}{\eta} \bigr)}{\int_{s \in \mathcal{S}} \int_{a \in \mathcal{A}} p_{\theta_\text{old}}(s,a) \exp\bigl( \frac{A_{\pi_\text{old}}(s,a)}{\eta} \bigr) \: da \: ds}, \label{eq:app_mpo_4}
\end{equation} 
but now we can also find the optimal \(\eta\) by substituting this solution into \(\mathcal{J}(\psi, \eta, \lambda)\). Doing so and dropping terms independent of \(\eta\) leads to
\begin{align}
	\begin{split}
	\eta &\biggl( \varepsilon_\eta - \int_{s \in \mathcal{S}} \int_{a \in \mathcal{A}} \psi(s,a) \ln \frac{\psi(s,a)}{p_{\theta_\text{old}}(s,a)} \: da \: ds \biggr) \\
	&= \eta \varepsilon_\eta + \eta \int_{s \in \mathcal{S}} \int_{a \in \mathcal{A}} \psi(s,a) \ln p_{\theta_\text{old}}(s,a) \: da \: ds - \eta \int_{s \in \mathcal{S}} \int_{a \in \mathcal{A}} \psi(s,a) \ln \psi(s,a) \: da \: ds. \label{eq:app_mpo_3}
	\end{split}
\end{align}
Because of Equation \eqref{eq:app_mpo_4}, we have
\begin{align*}
	\eta \psi(s,a) \ln \psi(s,a) 
	&= \eta \psi(s,a) \ln \frac{p_{\theta_\text{old}}(s,a) \exp\Bigl( \frac{A_{\pi_\text{old}}(s,a)}{\eta} \Bigr)}{\int_{s \in \mathcal{S}} \int_{a \in \mathcal{A}} p_{\theta_\text{old}}(s,a) \exp\Bigl( \frac{A_{\pi_\text{old}}(s,a)}{\eta} \Bigr) \: da \: ds} \\
	&= \psi(s,a) \Biggl( \eta \ln p_{\theta_\text{old}}(s,a) + A_{\pi_\text{old}}(s,a)  - \eta \ln \int_{s \in \mathcal{S}} \int_{a \in \mathcal{A}} p_{\theta_\text{old}}(s,a) \exp\biggl( \frac{A_{\pi_\text{old}}(s,a)}{\eta} \biggr) \: da \: ds \Biggr),
\end{align*}
where the first summand cancels out the second term in \eqref{eq:app_mpo_3} and the second summand no longer depends on \(\eta\) and thus can be dropped. Hence, we obtain the temperature loss function
\begin{equation}
	\mathcal{L}_\eta(\eta) = \eta \varepsilon_\eta + \eta \ln \biggl( \int_{s \in \mathcal{S}} \int_{a \in \mathcal{A}} \exp\biggl( \frac{A_{\pi_\text{old}}(s,a)}{\eta} \biggr) \: da \: ds \biggr)
\end{equation} 
through which we can optimize \(\eta\) using gradient descent.

Given the non-parametric sample-based variational distribution \(\psi(s,a)\), the M-step now optimizes the policy parameters \(\theta\). Based on \eqref{eq:app_mpo_2}, we want to maximize the discussed lower bound, i.e. minimize 
\[ - \int_{s \in \mathcal{S}} \int_{a \in \mathcal{A}} \psi(s,a) \ln \frac{p_\theta(\mathcal{I}=1,s,a)}{\psi(s,a)} \: da \: ds - \ln p(\theta) \]
to find new policy parameters \(\theta\). Using Equations \eqref{eq:app_mpo_4} and \eqref{eq:app_mpo_0}, the first term becomes 
\begin{align*}
	- &\int_{s \in \mathcal{S}} \int_{a \in \mathcal{A}} \psi(s,a) \ln \frac{p_\theta(\mathcal{I}=1,s,a)}{\psi(s,a)} \: da \: ds \\
	&= - \int_{s \in \mathcal{S}} \int_{a \in \mathcal{A}} \psi(s,a) \ln \frac{p_\theta(\mathcal{I}=1 \mid s,a) p_\theta(s,a)}{\psi(s,a)} \: da \: ds \\
	&= - \int_{s \in \mathcal{S}} \int_{a \in \mathcal{A}} \psi(s,a) 
	\ln \biggl(\frac{\exp\bigl( \frac{A_{\pi_\text{old}}(s,a)}{\eta} \bigr) p_\theta(s,a)}{p_{\theta_\text{old}}(s,a) \exp\bigl( \frac{A_{\pi_\text{old}}(s,a)}{\eta} \bigr)} \frac{1}{\int_{s \in \mathcal{S}} \int_{a \in \mathcal{A}} p_{\theta_\text{old}}(s,a) \exp\bigl( \frac{A_{\pi_\text{old}}(s,a)}{\eta} \bigr) \: da \: ds}\biggr) \: da \: ds \\
	&= - \int_{s \in \mathcal{S}} \int_{a \in \mathcal{A}} \psi(s,a) 
	\ln \biggl(\frac{p_\theta(s,a)}{p_{\theta_\text{old}}(s,a)} 
	\frac{1}{\int_{s \in \mathcal{S}} \int_{a \in \mathcal{A}} p_{\theta_\text{old}}(s,a) \exp\bigl( \frac{A_{\pi_\text{old}}(s,a)}{\eta} \bigr) \: da \: ds}\biggr) \: da \: ds.
\end{align*}
Using \(p_\theta(s,a) = \pi_\theta(a \mid s) d^{\pi_\theta}(s)\), assuming the state distribution \(d^\pi\) to be independent of \(\theta\) and dropping terms that do not depend on \(\theta\) yields 
\begin{align*}
	\argmin_\theta \biggl( - \int_{s \in \mathcal{S}} \int_{a \in \mathcal{A}} \psi(s,a) \ln \frac{p_\theta(\mathcal{I}=1,s,a)}{\psi(s,a)} \: da \: ds \biggr)
	&= \argmin_\theta \biggl( - \int_{s \in \mathcal{S}} \int_{a \in \mathcal{A}} \psi(s,a) 
	\ln p_\theta(s,a) \: da \: ds \biggr) \\
	&= \argmin_\theta \biggl( - \int_{s \in \mathcal{S}} \int_{a \in \mathcal{A}} \psi(s,a) 
	\ln \pi_\theta(a \mid s) \: da \: ds \biggr),
\end{align*}
which is the weighted maximum likelihood policy loss as in \eqref{eq:app_mpo_policy_loss}, that we compute on sampled transitions, effectively assigning out-of-sample transitions a weight of zero.

A useful prior \(\rho(\theta)\) in Equation \eqref{eq:app_mpo_2} is to keep the new policy close to the previous one as in TRPO and PPO. This translates to \[\rho(\theta) \approx -\nu \mathbb{E}_{S \sim d^{\pi_\text{old}}} \bigl[D_{KL}( \pi_{\text{old}} (\cdot \mid S) \Vert \pi_\theta (\cdot \mid S) ) \bigr]. \] Since optimizing the resulting sample-based maximum likelihood objective directly tends to result in overfitting, this prior is instead transformed into a constraint on the KL-divergence with bound \(\varepsilon_\nu\), i.e.
\begin{align*}
	\argmin_\theta &\quad \biggl( - \int_{s \in \mathcal{S}} \int_{a \in \mathcal{A}} \psi(s,a) \ln \frac{p_\theta(\mathcal{I}=1,s,a)}{\psi(s,a)} \: da \: ds \biggr) \\
	\text{subject to } &\quad \mathbb{E}_{S \sim d^{\pi_\text{old}}} \Bigl[D_{KL}\bigl( \pi_{\text{old}} (\cdot \mid S) \: \Vert \:  \pi_\theta (\cdot \mid S) \bigr) \Bigr] \leq \varepsilon_\nu.
\end{align*}
To employ gradient-based optimization, we use Lagrangian relaxation to transform this constraint optimization problem back into the unconstrained problem
\begin{equation}
	\mathcal{J}(\theta, \nu) = \mathcal{L}_\pi(\theta) + \nu \bigl(\varepsilon_\nu - \mathbb{E}_{S \sim d^{\pi_\text{old}}} \bigl[D_{KL}( \pi_{\text{old}} (\cdot \mid S) \Vert \pi_\theta (\cdot \mid S) ) \bigr]  \bigr).
\end{equation}
This problem is solved by alternating between optimizing for \(\theta\) and \(\nu\) via gradient descent in a coordinate-descent strategy. Using the stop-gradient operator \(\mathrm{sg}[[\cdot]]\), the objective can equivalently to this strategy be rewritten for as
\begin{align*}
	\mathcal{L}_\nu(\theta, \nu) = \nu \biggl( \varepsilon_\nu - \mathbb{E}_{S \sim d^{\pi_\text{old}}} \biggl[ \mathrm{sg\Bigl[\Bigl[ D_{KL}\bigl(\pi_{\theta_\text{old}}(\cdot \mid S) \: \Vert \: \pi_{\theta}(\cdot \mid S)\bigr)  \Bigr]\Bigr] \biggr]} \biggr) + \mathrm{sg}\bigl[\bigl[ \nu \bigr]\bigr] \mathbb{E}_{S \sim d^{\pi_\text{old}}} \Bigl[ D_{KL}\bigl( \pi_{\theta_\text{old}}(\cdot \mid S) \: \Vert \: \pi_{\theta}(\cdot \mid S) \bigr) \Bigr].
\end{align*}
Sampling this gives Equation \eqref{eq:app_mpo_kl_loss}. \(\eta\) and \(\nu\) are Lagrangian multipliers and hence must be positive. We enforce this by projecting the computed values to small positive values \(\eta_\text{min}\) and \(\nu_\text{min}\) respectively if necessary.

%% file: math_results.tex
\section{Auxiliary Theory}\label{sec:aux}


Here, we list a range of well-known definitions and results that we use in our work.
\\

\begin{definition}
	(Compact Space)
	A topological space \(X\) is called compact if for every set \(S\) of open covers of \(X\), there exists a finite subset \(S' \subset S\) that also is an open cover of \(X\). 
\end{definition}

\hfill

\begin{theorem}\label{th:bayes}
	(Bayes' Theorem) Let \((\Omega, \mathcal{A}, \mathbb{P})\) be a probability space and \(\bigcup_{i \in I}B_i\) be a disjoint and finite partition of \(\Omega\) with \(B_i \in \mathcal{A}\) and \(\mathbb{P}(B_i) > 0\) for \(i \in I\). Then, for all \(A \in \mathcal{A}\) and all \(k \in I\)
	\begin{equation*}
		\mathbb{P}(B_k \mid A) = \frac{\mathbb{P}(A \mid B_k) \mathbb{P}(B_k)}{\sum_{i \in I} \mathbb{P}(A \mid B_i) \mathbb{P}(B_i)}.
	\end{equation*}
\end{theorem}

\hfill

\begin{theorem}\label{th:variance}
	Let \(X\) be a random variable. Then,
	\begin{equation*}
		\mathrm{Var}[X] = \mathbb{E}\bigl[X^2\bigr] - \mathbb{E}\bigl[X\bigr]^2.
	\end{equation*}
\end{theorem}

\hfill

\begin{definition}\label{def:entropy}
	(Entropy)
	Let \((\Omega, \mathcal{A}, \mathbb{P})\) be a probability space and \(X \sim \mathbb{P}\) be a random variable. The entropy of \(X\) is given by 
	\begin{equation*}
		H(X) \coloneqq \mathbb{E}_{X \sim \mathbb{P}}\bigl[- \ln \mathbb{P}(X) \bigr].
	\end{equation*}
\end{definition}

\hfill

\begin{definition}\label{def:kl}
	(Kullback-Leibler Divergence)
	For any measurable space \(\mathcal{A}\) and probability densities \(p\) and \(q\) of the respective distributions \(P\) and \(Q\), the Kullback-Leibler divergence or relative entropy from \(Q\) to \(P\) is given by 
	\begin{equation*}
		D_{KL}(p \Vert q) \coloneqq \int_{a \in \mathcal{A}} p(a) \ln \frac{p(a)}{q(a)} da.
	\end{equation*}
\end{definition}

\hfill

\begin{definition}\label{def:tv_div}
	(Total Variation Divergence)
	For any measurable space \(\mathcal{A}\) and probability densities \(p\) and \(q\) of the respective distributions \(P\) and \(Q\), the total variation variance from \(Q\) to \(P\) is given by 
	\begin{equation*}
		D_{TV}(p \Vert q) \coloneqq \frac{1}{2} \int_{a \in \mathcal{A}} p(a) - q(a) da.
	\end{equation*}
\end{definition}

\hfill

\begin{theorem}\label{th:log_kl}
	Let \((\Omega, \mathcal{A})\) be a measurable space and \(p\) and \(\psi\) be probability measures on that space. Let and \(X \in \mathcal{A}\) and \(Z \in  \mathcal{A}\). Then,
	\begin{equation*}
		\ln p(X) = \mathbb{E}_{Z \sim \psi} \biggl[ \ln \frac{p(X,Z)}{\psi(Z)} \biggr] + D_{KL} (\psi \: \Vert \: p(\cdot \mid X) ).
	\end{equation*}
\end{theorem}

\hfill

\begin{theorem}\label{th:least_squares}
	Let \(X\) be a random variable. Then, 
	\begin{equation*}
		\min_a \mathbb{E} \bigl[(X - a)^2 \bigr] = \mathbb{E}[X].
	\end{equation*}
\end{theorem}

\hfill

\begin{theorem}\label{th:measure_change}
	Let \((\mathcal{A}, \Sigma)\) be a measurable space with \(\sigma\)-finite measures \(\mu\) and \(\nu\) such that \(\nu\) is absolutely continuous in \(\mu\). Let \(g\) be a Radon-Nikodym derivative of \(\nu\) w.r.t. \(\mu\), i.e. \(\nu(A) = \int_A g \: d\mu \) for all \(A \in \Sigma \). Let, \(f\) be a \(\nu\)-integrable function. Then,
	\begin{equation*}
		\int_\mathcal{A} f \: d\nu = \int_\mathcal{A} (f \cdot g) \: d\mu.
	\end{equation*}
\end{theorem}

\hfill

\begin{theorem}\label{th:leibniz}
	(Leibniz Integral Rule)
	Let \(X\) be an open subset of \(\mathbb{R}^d\), \(d \in \mathbb{N}\). Let \(\mathcal{A}\) be a measurable set and \(f \colon X \times \mathcal{A} \rightarrow \mathbb{R}\) be a function which satisfies
	\begin{enumerate}
		\item \(f(x, a)\) is a Lebesgue-integrable function of \(a\) for all \(x \in X\).
		\item For almost all \(a \in \mathcal{A}\), all partial derivatives exist for all \(x \in X\).
		\item There exists some integrable function \(g \colon \mathcal{A} \rightarrow \mathbb{R}\) with \(\lvert\nabla_x f(x, a) \rvert \leq g(a)\) for all \(x \in X\) and almost all \(a \in \mathcal{A}\).
	\end{enumerate}
	Then, for all \(x \in X\) we have
	\begin{equation*}
		\nabla_x \int_{a \in \mathcal{A}} f(x, a) da = \int_{a \in \mathcal{A}} \nabla_x  f(x, a) da
	\end{equation*}
\end{theorem}

\hfill

\begin{theorem}\label{th:fubini}
	(Fubini's Theorem)
	Let \(\mathcal{A}_1\) and \(\mathcal{A}_2\) be measurable spaces with measures \(\mu_1\) and \(\mu_2\) and \(f \colon \mathcal{A}_1 \times \mathcal{A}_2 \rightarrow \mathbb{R}\) be measurable and integrable w.r.t. the product measure \(\mu_1 \otimes \mu_2\), i.e. \(\int_{\mathcal{A}_1 \times \mathcal{A}_2} \lvert f\rvert \: d(\mu_1 \otimes \mu_2) < \infty\) or \(f \geq 0\) almost everywhere. Then, \(f(x,y)\) is integrable for almost all \(x\) and \(y\) and
	\begin{equation*}
		\int_{\mathcal{A}_1} \int_{\mathcal{A}_2} f(x,y) \: d\mu_1(x) \: d\mu_2(y) = \int_{\mathcal{A}_2} \int_{\mathcal{A}_1}  f(x,y) \: d\mu_2(y) \: d\mu_1(x) 
	\end{equation*}
\end{theorem}

\hfill

\begin{theorem}\label{th:taylor}
	(Taylor's Theorem - one-dimensional)
	Let \(k \in \mathbb{N}\) and let \(f \colon \mathbb{R} \rightarrow \mathbb{R}\) be  \(k\)-times differentiable at \(a \in \mathbb{R}\). Then, there exists a function \(h_k \colon \mathbb{R} \rightarrow \mathbb{R}\) such that
	\begin{equation*}
		f(x) = \sum^k_{i=0} \frac{f^{(i)}(a)}{i!}(x - a)^i + h_k(x)(x-a)^k.
	\end{equation*}
\end{theorem}

\hfill

\begin{theorem}\label{th:monotone_convergence}
	(Monotone Convergence Theorem)
	Let \(\bigl(x_n\bigr)^\infty_{n=0} \subset \mathbb{R}\) be a bounded and monotonically increasing sequence. Then, the sequence converges, i.e. \(\lim_{n \to \infty} x_n\) exists and is finite.
\end{theorem}

\hfill

\begin{theorem}\label{th:bolzano}
	(Bolzano-Weierstrass Theorem)
	Let \(\bigl(x_n\bigr)^\infty_{n=0} \subset \mathbb{R}^d\), \(d \in \mathbb{N}\) be a bounded sequence. Then, there exists some convergent subsequence \(\bigl(x_{n_i}\bigr)^\infty_{i=0}\).
\end{theorem}

\hfill

\begin{theorem}\label{th:berge}
	(Berge's Maximum Theorem)
	Let \(X\) and \(\Theta\) be topological spaces, \(f \colon X \times \Theta \rightarrow \mathbb{R}\) be continuous on \(X \times \Theta\) and \(C \colon \Theta \rightrightarrows X\) be a compact-valued correspondence with \(C(\theta) \neq \emptyset\) for all \(\theta \in \Theta\). Let 
	\begin{equation*}
		f^*(\theta) = \sup \bigl\{f(x,\theta) \mid x \in C(\theta) \bigr\}
	\end{equation*}
	and 
	\begin{equation*}
		C^*(\theta) = \argmax\bigl\{f(x,\theta) \mid x \in C(\theta) \bigr\} = \bigl\{x \in C(\theta) \mid f(x,\theta) = f^*(\theta) \bigr\}.
	\end{equation*}
	If \(C\) is continuous at \(\theta\), then \(f^*\) is continuous and \(C^*\) is upper hemicontinuous with nonempty and compact values.
\end{theorem}

\hfill

\begin{definition}\label{def:gateaux}
	(Gâteaux Derivative)
	Let \(X\) and \(Y\) be locally convex topological spaces, let \(U\) be an open subset of \(X\) and \(F \colon U \rightarrow Y\). The Gâteaux derivative of \(F\) at \(x \in U\) in the direction \(d \in X\) is defined as 
	\begin{equation*}
		dF(x, d) = \lim_{h \to 0} \frac{F(x + rd) - F(x)}{r}.
	\end{equation*}
\end{definition}